%% file: main.tex
\documentclass{article}

\usepackage[utf8]{inputenc}
\DeclareUnicodeCharacter{3002}{.}

\usepackage{microtype}
\usepackage{graphicx}
\usepackage{subcaption}
\usepackage{booktabs} 
\usepackage{xspace}
\usepackage{tcolorbox}
\usepackage{multirow}
\usepackage{multicol}
\usepackage{graphicx}
\usepackage{makecell} 

\usepackage{hyperref}

\usepackage[preprint]{icml2026}

\usepackage{amsmath}
\usepackage{amssymb}
\usepackage{mathtools}
\usepackage{amsthm}

\usepackage[capitalize,noabbrev]{cleveref}

\theoremstyle{plain}
\newtheorem{theorem}{Theorem}[section]
\newtheorem{proposition}[theorem]{Proposition}

\theoremstyle{definition}

\newtheorem{assumption}[theorem]{Assumption}
\theoremstyle{remark}

\newcommand{\methodname}{{\small \textsf{HyperLens}}\xspace}

\newcommand{\para}[1]{\noindent\textbf{#1}}

\usepackage[textsize=tiny]{todonotes}

\icmltitlerunning{HyperLens: Quantifying Cognitive Effort in LLMs with Fine-grained Confidence Trajectory}

\begin{document}

\twocolumn[
  \icmltitle{HyperLens: Quantifying Cognitive Effort in LLMs with Fine-grained Confidence Trajectory}



  \icmlsetsymbol{cor}{*}

  \begin{icmlauthorlist}
    \icmlauthor{Chengda Lu}{iiis,qizhi}
    \icmlauthor{Xiaoyu Fan}{iiis,qizhi,cor}
    \icmlauthor{Wei Xu}{iiis,qizhi,cor}
  \end{icmlauthorlist}

  \icmlaffiliation{iiis}{IIIS, Tsinghua University}
  \icmlaffiliation{qizhi}{Shanghai Qi Zhi Institute}
  
  \icmlcorrespondingauthor{Xiaoyu Fan}{fanxy98@mail.tsinghua.edu.cn}
  \icmlcorrespondingauthor{Wei Xu}{weixu@tsinghua.edu.cn}

  \icmlkeywords{Machine Learning, ICML}

  \vskip 0.3in
]



\printAffiliationsAndNotice{* Corresponding authors.}  

\begin{abstract}
    While Large Language Models (LLMs) achieve strong performance across diverse tasks, their inference dynamics remain poorly understood because of the limited resolution of existing analysis tools. In this work, we identify an intrinsic magnification mechanism in transformer architectures: deeper layers inherently magnify the small changes of layer-wise confidence, providing a fine-grained confidence trajectory. 
    Building on this insight, we introduce \methodname, a high-resolution probe designed to trace confidence trajectories and quantify the cognitive effort during inference. Across $8$ LLMs and $8$ datasets, \methodname reveals a consistent divergence in confidence trajectories that separates complex from simple tasks. We abstract this pattern into a quantitative cognitive effort metric. Our analysis reveals a fundamental principle: complex tasks consistently require higher cognitive effort. Finally, we provide a mechanistic diagnosis of a common side effect of standard Supervised Fine-Tuning (SFT): it can reduce cognitive effort and consequently degrade performance on in-domain tasks.
\end{abstract}

\input{src/intro}
\input{src/related_work}
\input{src/theory}
\input{src/method}
\input{src/evaluation}
\input{src/conclusion}







\bibliography{custom}
\bibliographystyle{icml2026}

\newpage
\appendix
\onecolumn
\input{src/appendix}

\end{document}

%% file: src/intro.tex
\section{Introduction}
\label{sec:intro}

While Large Language Models (LLMs) exhibit remarkable capabilities, understanding their \textit{inference dynamics}, i.e., how internal states evolve, is crucial for real-world robustness. A prevailing paradigm interpreting the dynamics involves monitoring the \textit{evolution of confidence} via layer-wise projections~\cite{nostalgebraist2020logitlens,schuster2022confident,wang2025logitlens4llms,belrose2023eliciting}. This signal is proven effective for various tasks like hallucination and jailbreak detection~\cite{chuang2023dola, lin2026alert}, and early-exit strategies~\cite{xin2020deebert,chen2023ee}.

\begin{figure*}[h]
    \centering
    \includegraphics[width=0.85\textwidth]{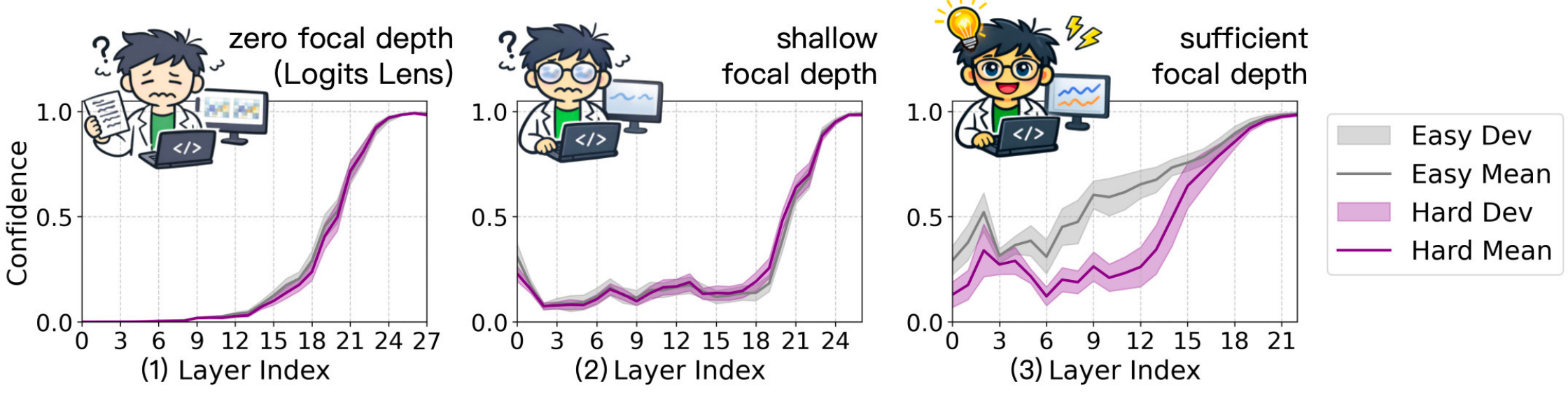}
    \caption{The confidence trajectory of Qwen2.5-7B with different focal depths on easy (CoNaLa~\cite{yin2018learning}) versus hard (APPS~\cite{hendrycks2021measuring_code}) coding tasks. The confidence of three figures is derived with (1) zero focal depth, which is equivalent to Logit Lens; (2) shallow focal depth with one single layer; and (3) sufficient focal depth with five layers.
        } 
    \vspace{-.1in}
    \label{fig:magnify_m}
\end{figure*}

Intuitively, the trajectory of confidence should reflect the model's \textit{cognitive effort}. Specifically, for trivial tasks (e.g., knowledge retrieval) where the answer is readily apparent, the internal confidence should converge swiftly, whereas complex tasks (e.g., hard math problems) necessitate a longer low-confidence ``struggle'' phase to reach the solution.
However, observations using current probing methods fail to validate this intuition, frequently producing indistinguishable trajectories regardless of task difficulty. 

We attribute this failure to the limited \textit{resolution} of existing probes, which limits the utility to only the most obvious anomalies (e.g., sudden spikes / collapses or stable changes).
Existing probes primarily rely on static projections~\cite{alain2016understanding, ghandeharioun2024patchscopes,pal2023future}, most notably the Logit Lens~\cite{alain2016understanding}, which decodes hidden states directly through the last unembedding layer.
These methods rely on a strong assumption that intermediate representations are linearly aligned with the output space. Unfortunately, because of the \emph{anisotropy}~\cite{ethayarajh2019contextual,timkey2021all} and \emph{representation drift}~\cite{geva2021transformer,voita2019analyzing}, these methods function as a \emph{myopic observer}: without proper focal correction, they fail to resolve latent signals effectively. 
To address this limitation, another line of work, like Tuned Lens~\cite{belrose2023eliciting} and others~\cite{cunningham2023sparse, rajamanoharan2024improving, ghandeharioun2024patchscopes} introduces an external \emph{learned adapter} to correct the projection misalignment. However, they inevitably act as a \emph{biased distortion filter}: the external training introduces systematic distortions to the original signal~\cite{hewitt2019designing,voita2020information}. 

Unlike the current myopic or biased tools, we identify an inherent self-magnifying mechanism within transformer layers (Section~\ref{sec:theory}). We prove that instead of directly decoding the hidden state, propagating it through a sequence of future layers before projection magnifies the derived confidence. We term the number of the future layers as \emph{focal depth}.

Based on this foundation, we propose a high-resolution probe, \methodname, which reveals a critical phenomenon aligning with our intuition: \emph{deepening the focal depth magnifies confidence, with the magnitude constrained by the intrinsic difficulty of the task}. 
As Figure~\ref{fig:magnify_m} illustrates, under zero focal depth (i.e., equivalent to Logit Lens), the confidence trajectories for easy and hard tasks are indistinguishable across all layers. Increasing the focal depth by 1 introduces a global magnification, while the trajectories remain entangled. 
Notably, when we increase the focal depth to a sufficient level, the confidence trajectories for easy and hard tasks are clearly separated. The easy task exhibits a more rapid increase, whereas the hard task remains at a low-confidence state for much longer.
This divergence indicates that sufficient magnification makes the model's internal cognitive effort observable.
To quantify this, we introduce the \emph{refinement area ($\Omega$)}, defined as the area \emph{above} the magnified confidence trajectory, serving as a precise measure of cognitive effort required to generate the final output. Intuitively, harder tasks require more cognitive effort to generate the final output, therefore have a higher $\Omega$.
For instance, the $\Omega$ of Qwen2.5-7B increases significantly from $1.78$ (Easy) to $3.05$ (Hard) in code tasks (Figure~\ref{fig:magnify_m}~(3)). 

Comprehensive evaluations across $8$ models and $8$ datasets (see Section~\ref{sec:empirical}) confirm that task difficulty consistently correlates with higher $\Omega$.
Beyond robust measurement, \methodname provides novel insights into the side effects of \emph{Supervised Fine-Tuning (SFT)} (Section~\ref{sec:SFT}). We uncover a ``blind confidence'' phenomenon: simple SFT can cause abnormal decreases in $\Omega$ on in-domain tasks, thereby harming the performance. This suggests that the model learns to shortcut the low-confidence ``struggle'' phase, jumping to obtaining the output without sufficient consideration. 

In summary, the contributions include: 

(1) We propose \methodname, a high-resolution probe that visualizes and quantifies the model's cognitive effort. Utilizing \methodname, we clearly separate the confidence trajectories of easy and hard tasks, revealing that complex tasks necessitate a longer low-confidence phase.

(2) We formally characterize the transformer's inherent self-magnifying mechanism, proving that increasing focal depth strictly magnifies confidence, thereby ensuring a high-resolution view.

(3) By applying \methodname, we uncover that simple SFT causes a ``blind confidence'' problem by cutting the cognitive effort when facing in-domain tasks, providing an interpretability explanation for performance degradation often observed in SFT models.

%% file: src/related_work.tex
\section{Related Works}\label{sec:related} 

\para{Inference dynamics interpretation} focuses on theoretically understanding the evolution of internal states. Intuitively, the inference includes two phases: (1) transform the input into a high-dimensional latent space; and (2) map the latent state to the vocabulary space to generate the next-token~\cite{olsson2022context, voita2019bottom,elhage2021mathematical}
Moving beyond this high-level decomposition, a growing body of research focuses on a more thorough analysis. For instance, \citet{shai2024transformers} prove that the residual stream effectively performs Bayesian updates across layers from the geometric perspective. \citet{tiberi2024dissecting} model the inference process as a kernel renormalization flow, formalizing the evolution of predictor statistics in a thermodynamic perspective.

\para{Probing methods.} To decode these internal dynamics, researchers propose a set of \emph{probing methods} to observe the internal state. According to the dependency on extra parameters, these probing methods can be classified into static and parametric probing. 
Static probing methods include the widely adopted Logit Lens~\cite{nostalgebraist2020logitlens} and its variants like \cite{ye2025transformers}, which uses cross-query semantic patching and a cosine-based representational lens to interpret the implicit reasoning phenomena. In this work, we focus on the static probing methods and mainly compare our method with the widely-adopted Logit Lens.

Parametric probing methods can be classified into two categories: one class, like Linear Probing~\cite{alain2016understanding} uses extra training to explore concept-specific features~\cite{cunningham2023sparse,rajamanoharan2024improving}. This class is typically task-driven, identifying specific features in the latent space that correspond to human-interpretable concepts.
Another class like Tuned Lens~\cite{belrose2023eliciting} and Patchscopes~\cite{ghandeharioun2024patchscopes}, on the other hand, utilizes extra models as an adaptor to correct the inherent flaws of static probings. However, these methods inevitably introduce systematic distortions as stated in~\cite{hewitt2019designing,voita2020information}.

\para{Applications.} Based on the above interpretation of inference dynamics, researchers propose diverse real-world applications that generally fall into two categories: (1) detecting abnormal behaviors and (2) improving the models' performance. 
First, by monitoring the internal states, researchers can effectively detect safety risks such as hallucinations and jailbreak attempts~\cite{lin2026alert, chia2025probing, borah2025alignment}. For instance, \citet{yao2024knowledge} propose a component-level decoding, i.e., knowledge circuits, and use it to analyze and interpret language model behaviors such as hallucinations and in-context learning.
Also, the insights obtained through interpretation help to optimize model efficiency and performance. For example, inference-Time Intervention (ITI)~\cite{li2023inference} identifies ``truthfulness'' directions in the activation space and shifts the inference dynamics along these vectors to enhance the generation of truthful answers. Similarly, \citet{chuang2023dola} leverage the contrast between early and late layer distributions to amplify factual signals and suppress hallucinations during decoding.

%% file: src/theory.tex
\section{Intrinsic Magnification Mechanism}
\label{sec:theory}

\subsection{Formulations}\label{subsec:formulation}

\para{Inference Dynamics.} Consider an LLM with $N$ transformer layers $\{f_1, f_2, \dots, f_N\}$, and a vocabulary space $\mathcal{V}$. The inference dynamics is defined as an autoregressive process that transforms the given input token into a probability distribution over the vocabulary. 
Let $X_{t} \in \mathbb{R}^d$ be the input embedding of the $t$th token. The internal state evolves layer-by-layer through a residual mapping $\mathcal{F}$ (containing Attention and FFN blocks). The internal state $H_{i,t} \in \mathbb{R}^d$ at layer $i$ is defined recursively:
\begin{equation}
    H_{i,t} = f_i(H_{i-1,t}) = H_{i-1,t} + \mathcal{F}_i(H_{i-1,t}), i \in \{1, \dots, N\},
\end{equation}

where $H_{0,t} = X_{t}$ is the initial input embedding. 

To generate the final output, the last internal state $H_{N,t}$ is projected back into the vocabulary space via a linear unembedding matrix $W_U \in \mathbb{R}^{|\mathcal{V}| \times d}$. The final probability distribution $p_t \in \mathcal{P}(|\mathcal{V}|) = \{x \in \mathbb{R}^{|\mathcal{V}|} \mid x \ge 0, \mathbf{1}^\top x = 1\}$ is obtained by:
\begin{equation}
    p_t = \text{Softmax}(W_U H_{N,t}).
\end{equation}

For simplicity, we omit the subscripts $t$ and $i$ when no confusion arises, referring $H_{i,t}$ as $H_i$ or $H$.

\para{Confidence Trajectory} is one of the most widely used signals to interpret the inference dynamics of LLMs. For a generated sequence of length $T$, we define the layer-wise confidence $\mathcal{C}_i$ as the expectation of the Top-K probability of the internal predicted token. Formally:
\begin{equation}\label{eq:conf_traj}
    \mathcal{C}_i = \frac{1}{T}\sum_{t=1}^T \text{Top-K}(p_{i,t}).
\end{equation}

$p_{i,t}$ is the probability of the $t$th token predicted by the $i$th layer, which can be obtained through a decoding function $\mathcal{D}: \mathbb{R}^d \to \mathcal{P}(\mathcal{V})$, i.e., $p_{i,t} = \mathcal{D}(H_{i,t})$. 
Existing methods instantiate $\mathcal{D}$ differently.
The static Logit Lens employs the last unembedding matrix with softmax, i.e., $\mathcal{D}(H) = \text{Softmax}(W_U H)$, whereas in other tunable methods like Tuned Lens, $\mathcal{D}$ is a learned affine transformation. The confidence trajectory is then the expected confidence of the predicted token across all the $N$ layers.

\para{Decoding function with varied focal depth $m$.} Unlike the prior static probing methods, we generalize the decoding function $\mathcal{D}$ from the last unembedding matrix to a composition of the final $m$ layers of the model (and the unembedding projection). Formally: 
\begin{equation}
    \mathcal{D}_m(H) = \text{Softmax}(D_m(H)),
\end{equation}

where $D_m$ is the composition of the final $m$ layers of the model and the unembedding projection. Formally:
\begin{equation}\label{eq:dm}
    D_m(H) = W_U (f_N \circ \cdots \circ f_{N-m+1}(H)).
\end{equation}

For simplicity, we define $D_m$ as the \emph{decoding function (DF)} with focal depth $m$ in the following analysis.

\para{Refinement layers. } Prior research suggests that the computation of the transformer exhibits two stages, separated by a layer $i_0$ (see Section~\ref{sec:related}). The initial layers ($i < i_0$) map the inputs to the latent space, and the subsequent layers ($i \geq i_0$) decode the latent space to the output. In this work, we term the layers in the latter stage as \emph{refinement layers}.

\subsection{Assumptions}

\para{Confidence in the Logit Space.} To facilitate convex analysis, we analyze the confidence in the logit space rather than the probability simplex. Let $S_t$ be the Top-K set. We define the \textit{Logit Margin} function $g_m: \mathbb{R}^d \to \mathbb{R}$ as:
\begin{equation}
    g_m(H) := \log \sum_{v \in S_t} e^{\mathcal{D}_m(H)_v} - \log \sum_{v \notin S_t} e^{\mathcal{D}_m(H)_v}.
\end{equation}
Since the Softmax function is monotonic w.r.t the margin, proving that the sequence-averaged margin $\bar{\mathcal{C}}^{(m)}_i = \frac{1}{T}\sum_{t} g_m(H_{i,t})$ increases is equivalent to proving that the observed confidence $\mathcal{C}_i$ increases. In our analysis, we use $\bar{\mathcal{C}}^{(m)}_i$ as the approximation of the confidence $\mathcal{C}_i$.

Before obtaining the theoretical results, we make the following assumptions:

\begin{assumption}[Smoothness of the Margin Function]\label{assump-lipschitz}
    We assume the margin function $g_m(H)$ has a $\beta$-Lipschitz continuous gradient over the domain of hidden states. Formally, for any two hidden states $H_i, H_k, k \geq i$, there exists a constant $\beta > 0$ such that:
    \begin{equation}
        \|\nabla g_m(H_i) - \nabla g_m(H_{k})\| \le \beta \|H_i - H_{k}\|.
    \end{equation}
\end{assumption}

The rationale behind this assumption is that the hidden states during the forward pass are normalized by LayerNorm (restricting the domain), and the Softmax-Log functions comprising $g_m$ are smooth $\mathcal{C}^\infty$ functions. Consequently, the Hessian of $g_m$ is bounded, preventing the gradient from changing arbitrarily fast.


\begin{proposition}[Quadratic Smoothness Bound]\label{prop-smooth-bound}
    Under Assumption~\ref{assump-lipschitz}, for any current hidden state $H$ and a residual update vector $\Delta$, the margin function satisfies:
    \begin{equation}
        g_m(H + \Delta) \ge g_m(H) + \langle \nabla g_m(H), \Delta \rangle - \frac{\beta}{2} \|\Delta\|^2.
    \end{equation}
\end{proposition}

While Proposition~\ref{prop-smooth-bound} is a classic result (for proof, see  Appendix~\ref{subappen:proof_assumption}), it holds a specific meaning in our inference dynamics. Here, $\Delta$ represents the residual update provided by a residual function on \emph{any} prior hidden state $H_i$, i.e., $\Delta = \mathcal{F}_k(H_i), k \geq i$.

\begin{assumption}[Gradient-aligned Updates]\label{assump-drift} 
    For any layer index $k \geq i_0$ (where $i_0$ denotes the start of the refinement layers (see Section~\ref{subsec:formulation})), the residual update $\Delta$ aligns positively with the gradient of the margin function $g_m(H)$ in expectation. Specifically, there exists a constant $\mu > 0$ such that:

    \begin{equation} 
        \mathbb{E}\left[ \langle \nabla g_m(H), \Delta \rangle \mid H \right] \ge \mu. 
    \end{equation} 
    
\end{assumption} 

This assumption formalizes the intuition that well-trained Transformers function as implicit optimizers. In the late refinement stages ($k \ge i_0$), the residual function $\mathcal{F}_k$ provides updates that push the hidden state $H$ closer to the decision boundary of the target token.

\begin{assumption}[Boundedness]\label{assump-bound}
    The residual updates are bounded by a constant $R$, i.e., $\|\Delta\| \le R$. Furthermore, the fluctuations of the single-step margin gain around its mean are bounded by a constant $b$:

    \begin{equation} 
        |g_m(H_{i, t}) - \mathbb{E}_t\left[ g_m(H_{i, t})\right]| \le b. 
    \end{equation} 
    
\end{assumption}

Assumption~\ref{assump-bound} aligns with the use of normalization and residual connections, which constrain latent increments to prevent representation divergence. 

\subsection{Main Results}\label{subsec:main_res}

Based on the above assumptions, we obtain two fundamental properties of the confidence trajectory: the monotonic growth across layers (Result 1) and the magnification effect of larger focal depth $m$ (Result 2). For the formal proof, please refer to Appendix~\ref{subappen:proof_thm1} and Appendix~\ref{thm:head_lift}.

\para{Result 1: The confidence of the refinement layers increases monotonically.}
The first result quantifies the trajectory of the confidence (where focal depth $m$ is fixed). We show that in the refinement phase, the layer-wise confidence statistically increases layer by layer.

\begin{theorem}[Monotonicity of Confidence]
\label{thm:monotonicity}
    Under Assumptions~\ref{assump-lipschitz}-\ref{assump-bound}, for layers $i \geq i_0$ and sequence length $T$, the probability that the sequence-averaged confidence $\bar{\mathcal{C}}^{(m)}_i = \frac{1}{T}\sum_{t} g_m(H_{i,t})$ decreases in the next layer decays exponentially:
    \begin{equation}
        \mathbb{P}\left( \bar{\mathcal{C}}^{(m)}_{i+1} \le \bar{\mathcal{C}}^{(m)}_{i} \right) \le \exp\left( -\frac{T \gamma^2}{2b^2} \right),
    \end{equation}
    where $\gamma = \mu - \frac{\beta}{2}R^2$.
\end{theorem}

Theorem~\ref{thm:monotonicity} illustrates that while individual tokens may exhibit noisy confidence fluctuations, the \textit{averaged} confidence over a sequence statistically rises in refinement layers. 

\para{Result 2: Self-magnifying effect of the focal depth $m$.}
We now analyze the impact of the focal depth $m$. We compare the confidence observed through the decoding function with focal depth $m+1$ $\mathcal{D}_{m+1}$ versus $\mathcal{D}_m$ with focal depth $m$ at the \textit{same} hidden state $H_i$.

\begin{theorem}[Magnification of Focal Depth]
\label{thm:head_lift}
    Let $D_{m+1}$ be a decoding function with one additional layer compared to $D_m$. Based on Assumption~\ref{assump-drift}, the confidence trajectory observed through $D_{m+1}$ statistically dominates $D_m$ with a high probability:
    \begin{equation}
         \mathbb{P}\left( \bar{\mathcal{C}}^{(m+1)}_{i} > \bar{\mathcal{C}}^{(m)}_{i} \right) \ge 1 - \exp\left( -\frac{T \gamma^2}{2b^2} \right).
    \end{equation}
\end{theorem}

Theorem~\ref{thm:head_lift} proves that increasing $m$ introduces a monotonic upward shift of the confidence trajectory. Intuitively, $D_{m+1}$ incorporates an additional ``refining'' transformation towards the direction of generating the output token, thereby effectively increasing the confidence compared to $D_m$.



%% file: src/method.tex
\section{High-resolution Probing and Measurement}\label{sec:method}

Based on the theoretical foundations in Section~\ref{sec:theory}, we design \methodname, a high-resolution probe to observe the confidence trajectory in a fine-grained manner. Furthermore, we formalize a quantitative metric from the trajectory to measure the cognitive effort of LLMs.

\para{High-resolution probing with a proper focal depth.} Theorem~\ref{thm:monotonicity} and ~\ref{thm:head_lift} indicate that increasing focal depth $m$ magnifies the confidence in the refinement layers. Therefore, unlike Logit Lens that fixes $m = 0$, \methodname utilizes the decoding function $D_m$ (see Equation~\ref{eq:dm}) with focal depth $m \geq 1$ to observe the confidence. Specifically, for each hidden state $H_i$, we first propagate it through $m$ subsequent layers, i.e., $\{f_{N-m+1}, \dots, f_N\}$, and then project it to the logit space through the last unembedding layer to derive the magnified confidence $\mathcal{C}^{(m)}_{i}$. 
Note that a proper $m$ is crucial, because an excessively large $m$ would over-magnify all signals to near-certainty, making the trajectories once again indistinguishable. At a proper $m$, magnification is constrained by the tasks' inherent difficulty, making the trajectories for hard and easy tasks clearly separated. 

\para{Quantitative metric of cognitive effort.} Beyond visually observing the difference between easy and hard tasks, we formalize the cognitive effort into a quantitative metric. We define \emph{refinement area} ($\Omega$) as the area above the magnified confidence curve of the refinement layers with a fixed focal depth $m$, formally:

\begin{equation}\label{eq:ra}
    \Omega = \sum_{k \geq i_0} (1 - \mathcal{C}^{(m)}_k).
\end{equation}

Intuitively, $\Omega$ captures the cumulative uncertainty the model must resolve before converging on a final output, which is positively correlated with the cognitive effort: easy tasks exhibit a rapid confidence magnification, maintaining higher confidence across the refinement layers and yielding a lower $\Omega$; conversely, complex tasks necessitate a longer low-confidence struggle stage, resulting in a higher $\Omega$.

\para{Automatic computation of refinement area.} To facilitate the computation of $\Omega$, we need to determine the beginning point of the refinement layers, i.e., $i_0$, automatically. According to Theorem~\ref{thm:monotonicity}, which establishes that the confidence demonstrates a statistical monotonicity in the refinement layers, we formulate the automatic determination algorithm through a \emph{reverse-scanning} strategy.
Briefly, for a fixed model and focal depth $m$, we scan the confidence trajectory backward from the final layer, and identify $i_0$ as the first layer where the monotonic decreasing trend is violated. The detailed algorithm is in Appendix~\ref{appen:algorithm}.

%% file: src/evaluation.tex
\section{Evaluations}
\label{sec:empirical}


\subsection{Setup}
\label{subsec:exper_setup}
\para{Models.}  
We evaluate \methodname with varied focal depths $m$ on eight models with diverse architectures. When $m=0$, the probe is equivalent to Logit Lens~\cite{nostalgebraist2020logitlens}, the most widely used static probing method, which we use as the baseline.
Specifically, we adopt Llama3-8B-Instruct~\cite{dubey2024llama}[Llama3-8B], Deepseek-Math-7B-Instruct~\cite{shao2024deepseekmath}[Deepseek-7B], and multiple sizes of the Qwen family, including Qwen2.5-Instruct~\cite{hui2024qwen2} (0.5B, 3B, 7B, 32B) [Qwen2.5-0.5B/Qwen2.5-3B/Qwen2.5-7B/Qwen2.5-32B] and Qwen3~\cite{yang2025qwen3} (0.6B, 4B) [Qwen3-0.6B/Qwen3-4B].

\para{Datasets.}
To comprehensively evaluate the cognitive effort, we conduct evaluations across diverse domains.
For each domain, we pair a foundational dataset representing easy tasks with a challenging hard-level counterpart:
(1) Mathematics: MATH (Level 1-2)~\cite{hendrycks2021measuring} versus AIME (2023-2025)~\cite{WIKI-AIME};
(2) Coding: CoNaLa~\cite{yin2018learning} versus APPS~\cite{hendrycks2021measuring_code};
(3) Logic: RuleTaker~\cite{clark2020transformers} versus ZebraLogic~\cite{lin2025zebralogic}; and
(4) General QA: ARC-Easy~\cite{clark2018think} versus the expert-level GPQA~\cite{rein2024gpqa}.
All the results are obtained on $100$ random samples for efficiency and robustness except for AIME, where we utilize all the problems from the 2023-2025 contests.

\begin{figure*}[t]
    \centering
    \includegraphics[width=0.9\linewidth]{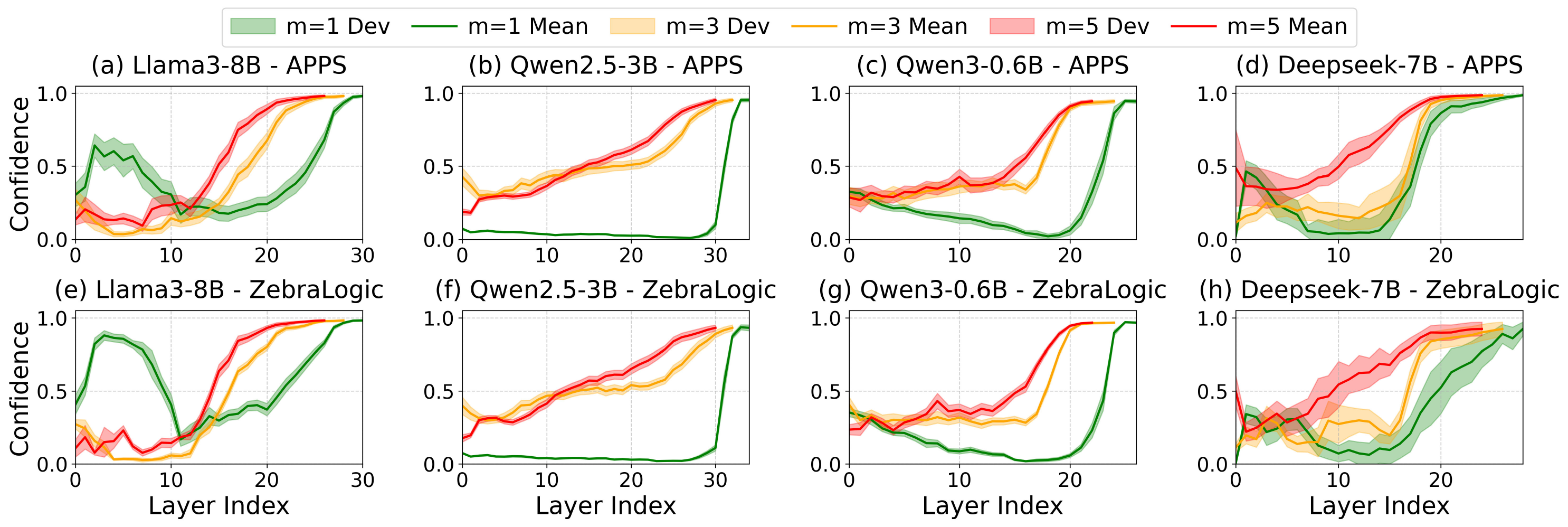}
    \caption{The confidence trajectory of Llama3-8B, Qwen2.5-3B, Qwen3-0.6B, Deepseek-7B on APPS~\cite{hendrycks2021measuring_code} and ZebraLogic~\cite{lin2025zebralogic} datasets with $m = 1,3,5$.
    }
    \vspace{-.15in}
    \label{fig:m_comparison}
\end{figure*}

\subsection{Empirical Validation}

\para{Validating Result1: the confidence trajectory increases monotonically in the refinement layers. } 
Figure~\ref{fig:m_comparison} plots the confidence trajectory (as defined in Equation~\ref{eq:conf_traj}, Section~\ref{subsec:formulation}) of four models with varied $m$ on two datasets. We leave the full results in Appendix~\ref{appen:cog_effort} for succinctness. We observe a robust monotonic growth in confidence throughout the deeper refinement layers across all cases. 
The universality of this trend across diverse architectures and domains empirically validates Theorem~\ref{thm:monotonicity}, demonstrating that the monotonically increasing trajectory is an intrinsic characteristic of the models' inference dynamics.

\para{Validating Result2: larger focal depth $m$ magnifies the observed confidence.}
From Figure~\ref{fig:m_comparison}, we can see that increasing $m$ strictly magnifies the confidence in the deeper refinement layers. Specifically, the trajectories exhibit a strict ordering in the deeper layers. The trajectory for $m=5$ (red) lies above that for $m=3$ (yellow), which in turn surpasses $m=1$ (green) across all the models and datasets. The strict ordering phenomenon validates Theorem~\ref{thm:head_lift}. By incorporating a few transformer layers into the decoding function, we can effectively magnify the confidence signals and therefore provide a fine-grained view.

\para{Further analysis. }Beyond the properties stated in Theorem~\ref{thm:monotonicity} and Theorem~\ref{thm:head_lift}, we also observe several interesting properties in Figure~\ref{fig:m_comparison}: 

\emph{(1) Model-specific sensitivity of the magnification mechanism.} The magnitude of magnification with larger $m$ varies significantly across models. For instance, Qwen2.5-3B exhibits a substantial gap between the trajectories of $m=1$ and $m=3 ~ \text{or}~ 5$, indicating a high sensitivity to small $m$. In contrast, Deepseek-7B shows a much narrower divergence of $m=1$ and $3$. The trajectory becomes significantly lifted until $m=5$, suggesting that the detailed magnification mechanism varies across models.

\emph{(2) Intrinsic beginning points of the refinement layers.} The beginning points of the refinement layers, i.e., the layer where the confidence begins to climb, appear to be an intrinsic characteristic of the models with varied focal depth $m$, which remains remarkably consistent across diverse datasets for any fixed model and $m$.
For instance, when $m = 1$, we can see that the beginning point is around $10$, $30$, $20$ and $10$ for the four  respective models across datasets. This consistency indicates that the location of refinement layers is an intrinsic property of the model, independent of the specific datasets being processed.

\begin{figure*}[t]
    \centering
    \includegraphics[width=0.9\linewidth]{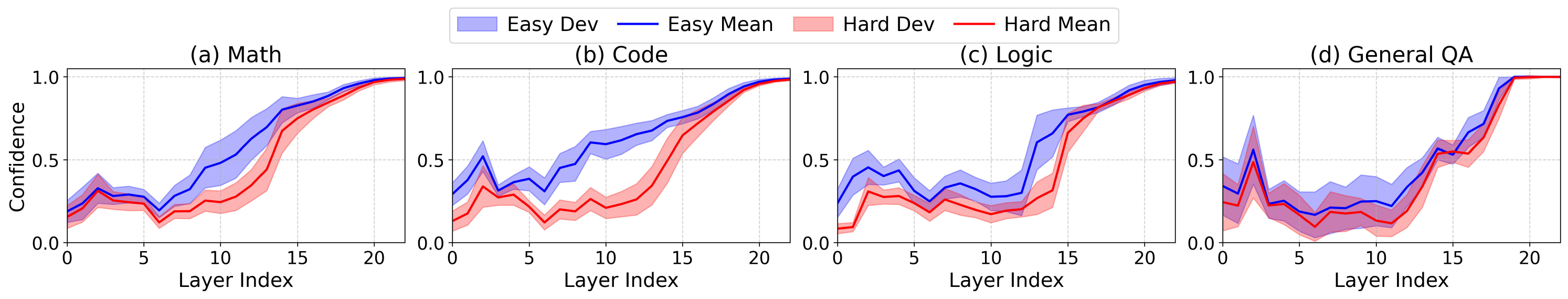}
    \vspace{-.1in}
    \caption{The confidence trajectory of Qwen2.5-7B across four domains comparing easy versus hard datasets with focal depth $m = 5$. 
    }
    \label{fig:cog_effort}
\end{figure*}

\begin{table*}[t]
\centering
\caption{Refinement area ($\Omega$) across models and data domains.}
\label{tab:refinement_results}
\vspace{-.1in}
\setlength{\tabcolsep}{2.5pt}
\resizebox{0.85\textwidth}{!}{
\begin{tabular}{l c c | cc | cc | cc | cc}
\toprule
\multirow{2}{*}{\textbf{Model}} & \multirow{2}{*}{\textbf{$m$}} & \multirow{2}{*}{\textbf{$i_0$}} &
\multicolumn{2}{c|}{\textbf{Math}} & \multicolumn{2}{c|}{\textbf{Code}} & 
\multicolumn{2}{c|}{\textbf{Logic}} & \multicolumn{2}{c}{\textbf{General QA}} \\
\cmidrule(lr){4-5} \cmidrule(lr){6-7} \cmidrule(lr){8-9} \cmidrule(lr){10-11}
 & & & \textbf{Easy} & \textbf{Hard} & \textbf{Easy} & \textbf{Hard} & \textbf{Easy} & \textbf{Hard} & \textbf{Easy} & \textbf{Hard} \\
\midrule
\midrule


Llama3-8B & 0 & 30 & $1.02_{\pm 0.02}$ & $1.04_{\pm 0.03}$ & $1.01_{\pm 0.01}$ & $1.01_{\pm 0.01}$ & $\textcolor{blue}{1.02}_{\pm 0.01}$ & $\textcolor{blue}{1.01}_{\pm 0.01}$ & $1.00_{\pm 0.01}$ & $1.02_{\pm 0.02}$ \\ 

Llama3-8B & 5 & 12 & $3.80_{\pm 0.46}$ & $4.18_{\pm 0.67}$ & $3.43_{\pm 0.38}$ & $3.95_{\pm 0.42}$ & $3.35_{\pm 0.39}$ & $3.93_{\pm 0.28}$ & $2.99_{\pm 0.43}$ & $3.54_{\pm 0.53}$ \\ 

\midrule

Deepseek-7B & 0 & 1 & $2.59_{\pm 0.19}$ & $2.67_{\pm 0.22}$ & $2.71_{\pm 0.20}$ & $2.90_{\pm 0.33}$ & $2.08_{\pm 0.26}$ & $2.46_{\pm 0.26}$ & $2.39_{\pm 0.33}$ & $2.61_{\pm 0.33}$ \\ 

Deepseek-7B & 1 & 14 & $3.19_{\pm 0.45}$ & $3.61_{\pm 0.78}$ & $4.34_{\pm 1.09}$ & $6.16_{\pm 0.84}$ & $6.22_{\pm 1.05}$ & $6.87_{\pm 1.41}$ & $7.23_{\pm 1.26}$ & $8.05_{\pm 2.36}$ \\ 

\midrule

Qwen2.5-0.5B & 0 & 22 & $1.01_{\pm 0.02}$ & $1.02_{\pm 0.02}$ & $1.04_{\pm 0.02}$ & $1.05_{\pm 0.02}$ & $1.06_{\pm 0.04}$ & $1.07_{\pm 0.03}$ & $1.11_{\pm 0.04}$ & $1.07_{\pm 0.05}$ \\ 

Qwen2.5-0.5B & 5 & 3 & $3.74_{\pm 0.48}$ & $4.04_{\pm 0.50}$ & $5.76_{\pm 0.39}$ & $5.78_{\pm 0.40}$ & $\textcolor{blue}{6.06}_{\pm 0.50}$ & $\textcolor{blue}{6.03}_{\pm 0.35}$ & $\textcolor{blue}{6.11}_{\pm 0.41}$ & $\textcolor{blue}{5.87}_{\pm 0.97}$ \\ 

\midrule

Qwen2.5-3B & 0 & 23 & $5.07_{\pm 0.35}$ & $5.61_{\pm 0.38}$ & $4.63_{\pm 0.20}$ & $5.09_{\pm 0.22}$ & $5.06_{\pm 0.22}$ & $5.14_{\pm 0.19}$ & $4.66_{\pm 0.34}$ & $5.73_{\pm 0.41}$ \\ 

Qwen2.5-3B & 3 & 22 & $1.97_{\pm 0.27}$ & $2.24_{\pm 0.30}$ & $2.41_{\pm 0.32}$ & $2.52_{\pm 0.23}$ & $2.19_{\pm 0.58}$ & $2.87_{\pm 0.28}$ & $2.62_{\pm 0.64}$ & $3.01_{\pm 0.57}$ \\ 

\midrule

Qwen2.5-7B & 0 & 13 & $6.57_{\pm 0.45}$ & $7.10_{\pm 0.52}$ & $6.74_{\pm 0.37}$ & $6.97_{\pm 0.34}$ & $7.33_{\pm 0.23}$ & $7.50_{\pm 0.24}$ & $7.52_{\pm 0.34}$ & $7.81_{\pm 0.39}$ \\ 

Qwen2.5-7B & 5 & 12 & $1.46_{\pm 0.40}$ & $2.38_{\pm 0.44}$ & $1.78_{\pm 0.32}$ & $3.05_{\pm 0.55}$ & $2.38_{\pm 0.52}$ & $3.38_{\pm 0.46}$ & $2.13_{\pm 0.32}$ & $2.63_{\pm 0.44}$ \\ 

\midrule

Qwen2.5-32B & 0 & 39 & $\textcolor{blue}{1.13}_{\pm 0.26}$ & $\textcolor{blue}{1.06}_{\pm 0.32}$ & $\textcolor{blue}{0.97}_{\pm 0.24}$ & $\textcolor{blue}{1.09}_{\pm 0.26}$ & $0.73_{\pm 0.21}$ & $1.09_{\pm 0.19}$ & $0.66_{\pm 0.38}$ & $1.37_{\pm 0.48}$ \\ 

Qwen2.5-32B & 5 & 42 & $3.76_{\pm 0.67}$ & $4.78_{\pm 0.43}$ & $3.36_{\pm 0.73}$ & $5.32_{\pm 0.52}$ & $2.65_{\pm 0.67}$ & $4.53_{\pm 0.61}$ & $2.55_{\pm 0.55}$ & $4.40_{\pm 0.83}$ \\ 

\midrule

Qwen3-0.6B & 0 & 16 & $3.84_{\pm 0.13}$ & $3.98_{\pm 0.13}$ & $3.70_{\pm 0.09}$ & $3.82_{\pm 0.11}$ & $3.76_{\pm 0.09}$ & $3.98_{\pm 0.11}$ & $3.59_{\pm 0.11}$ & $3.99_{\pm 0.12}$ \\ 

Qwen3-0.6B & 5 & 14 & $2.07_{\pm 0.21}$ & $2.33_{\pm 0.16}$ & $2.47_{\pm 0.21}$ & $2.70_{\pm 0.19}$ & $\textcolor{blue}{2.36}_{\pm 0.13}$ & $\textcolor{blue}{2.34}_{\pm 0.15}$ & $2.50_{\pm 0.23}$ & $2.70_{\pm 0.17}$ \\ 

\midrule

Qwen3-4B & 0 & 22 & $4.70_{\pm 0.29}$ & $5.00_{\pm 0.25}$ & $4.38_{\pm 0.27}$ & $4.84_{\pm 0.22}$ & $4.43_{\pm 0.20}$ & $5.09_{\pm 0.19}$ & $3.68_{\pm 0.21}$ & $4.68_{\pm 0.31}$ \\

Qwen3-4B & 5 & 21 & $1.87_{\pm 0.21}$ & $1.87_{\pm 0.23}$ & $2.29_{\pm 0.25}$ & $2.52_{\pm 0.21}$ & $\textcolor{blue}{2.25}_{\pm 0.31}$ & $\textcolor{blue}{2.03}_{\pm 0.17}$ & $1.86_{\pm 0.32}$ & $2.35_{\pm 0.24}$ \\

\bottomrule
\bottomrule
\end{tabular}%
}
\end{table*}

\subsection{Cognitive Effort Measurement}

In this section, we provide an empirical analysis of how task difficulty influences the model's cognitive effort.

\para{Visualization of cognitive effort difference.}
We visualize the confidence trajectory of the models dealing with easy and hard tasks with proper focal depths.
Figure~\ref{fig:cog_effort} shows the confidence trajectory of Qwen2.5-7B on easy and hard tasks with a focal depth $m=5$ across four domains. The results of other models are shown in Appendix~\ref{appen:cog_effort}.
We can observe that the confidence trajectory dealing with easy tasks (blue curves) remains higher than that of hard tasks (red curves) throughout the deeper refinement layers. 
Also, the gap between easy and hard trajectories is more significant in Math and Code domains. This is because these tasks are inherently more difficult and require higher logical efforts. In contrast, the gap is much less in the General QA domain, where the tasks predominantly focus on textual comprehension, showing a less severe easy / hard distinction.

\para{Quantitative measurement of cognitive effort difference.} To quantify the distinction, we report the refinement area $\Omega$ of all the tested models dealing with easy and hard tasks in Table~\ref{tab:refinement_results}. For each model, we report the $\Omega$ calculated with a fixed focal depth $m=0$ (equivalent to Logit Lens) and the most appropriate focal depth selected from $m = \{1, 3, 5\}$ that maximizes the difference between the easy and hard tasks.
We have the following observations from Table~\ref{tab:refinement_results}:

(1) \emph{The refinement area $\Omega$ of hard tasks is always larger than that of easy tasks.} Specifically, $\Omega$ of hard tasks is larger than that of easy tasks in almost all the $32 = 8 (\text{models}) \times 2 (\text{focal depth } m) \times 4 (\text{task domains})$ cases except for a few corner cases that we marked with blue colors.
However, the differences in these corner cases are almost negligible and happen mostly in the Logic and General QA domains, which is consistent with our visualization results stated in Figure~\ref{fig:cog_effort}, i.e., the inherent difficulty difference between easy and hard tasks is not significant in these two domains.
The results align with our intuition stated in Section~\ref{sec:intro}, i.e., hard tasks require more cognitive effort to solve.

(2) \emph{The difference between the refinement area $\Omega$ of easy and hard tasks is more significant with a focal depth $m \geq 1$.} We can see that the difference of $\Omega$ between easy/hard tasks is larger when $m \geq 1$ (second row per model) compared to $m=0$ baseline (first row per model). Specifically, the maximum difference expands from $1.1$ at $m=0$ to $2.0$ at $m \geq 1$. 
This contrast suggests that a non-zero focal depth is necessary to observe the effort required for varied tasks.

\begin{figure*}[t]
    \centering
    \includegraphics[width=0.48\linewidth]{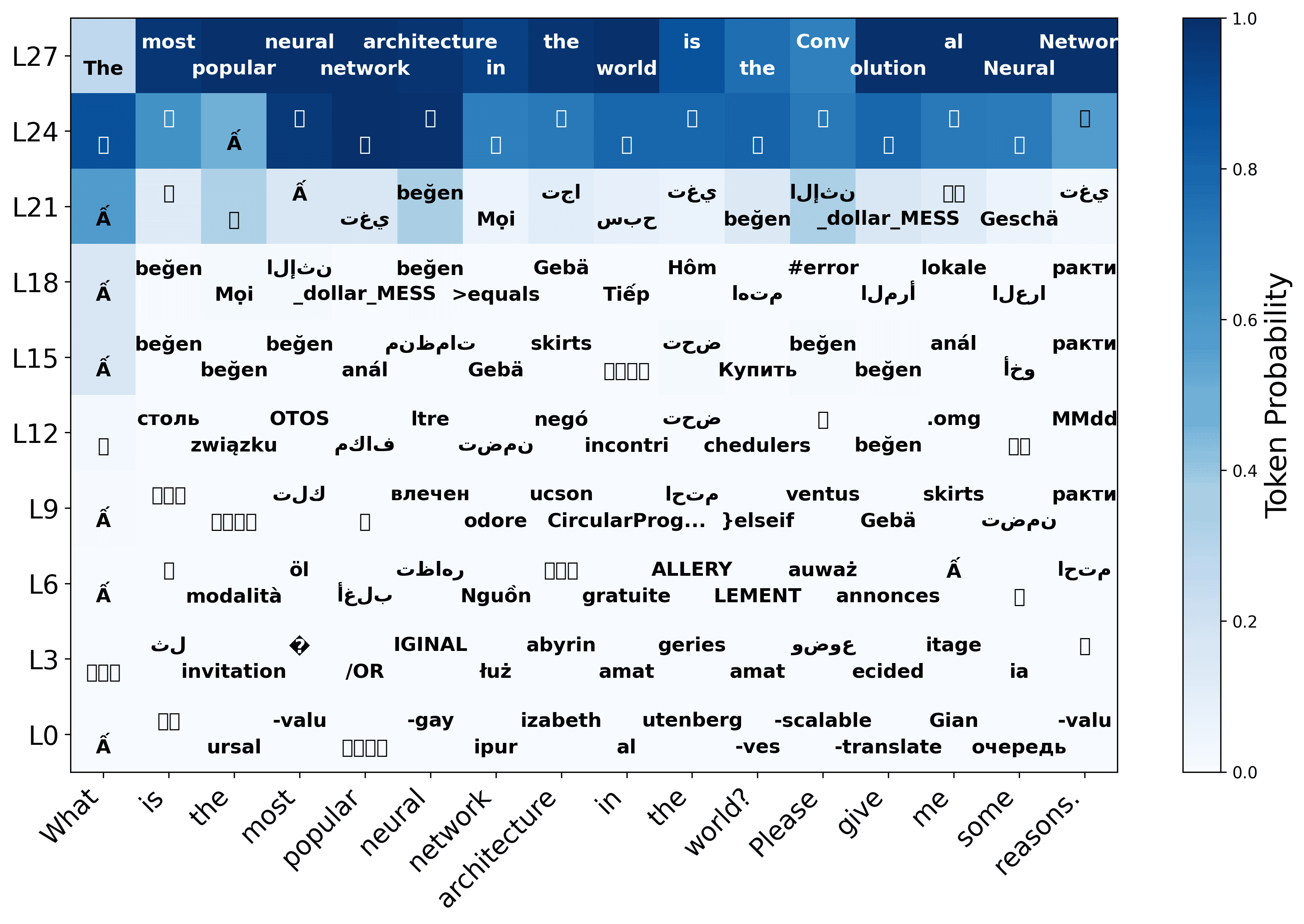}
    \includegraphics[width=0.48\linewidth]{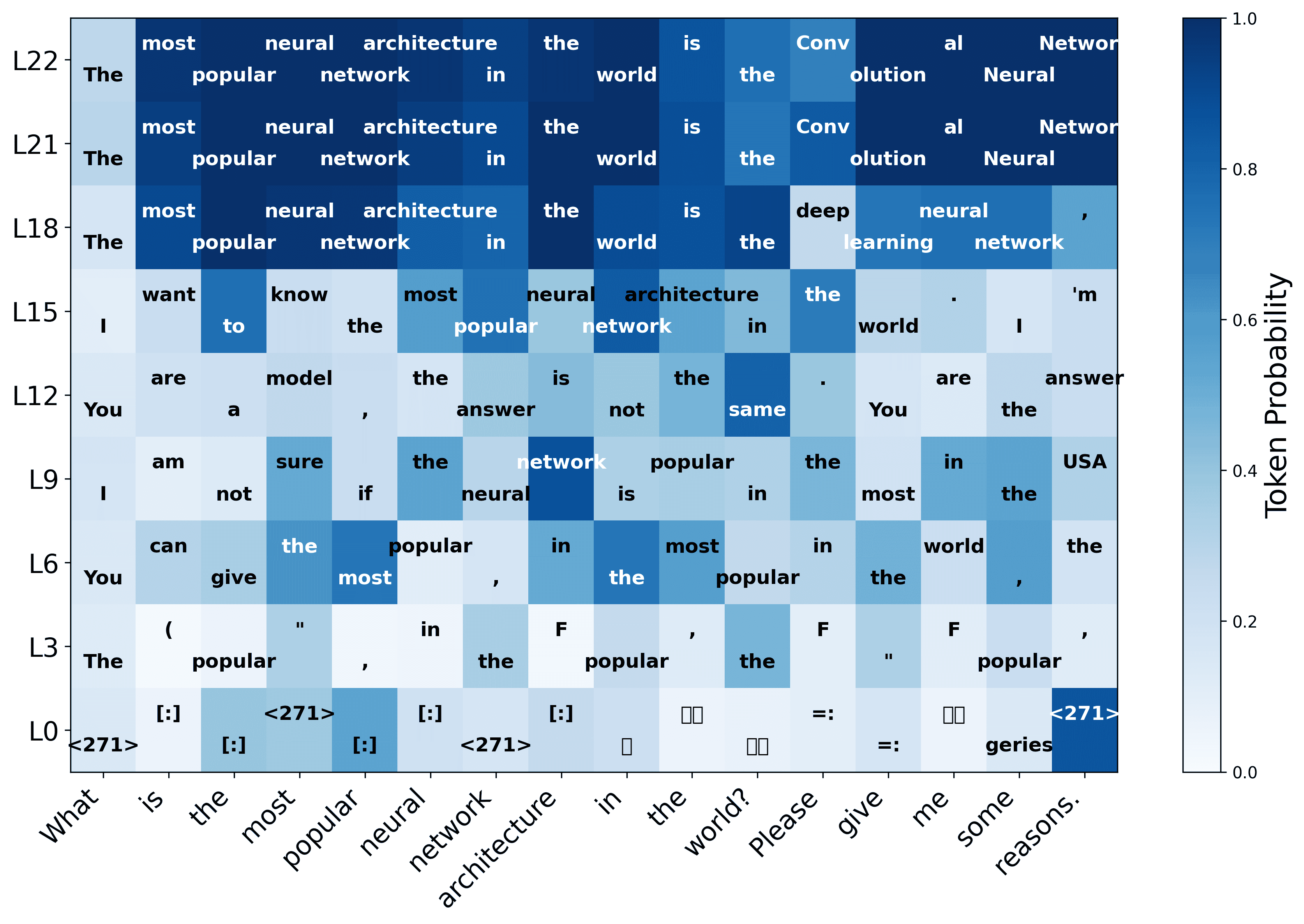}
    \vspace{-.1in}
    \caption{Semantic decoding of Qwen2.5-7B. The left and right are Logit Lens ($m = 0$) and {\footnotesize{\textsf{HyperLens}}} (with $m = 5$), respectively. 
    }
    \label{fig:semantic_decoding}
\end{figure*}

\begin{table*}[t]
\centering
\caption{Refinement Area ($\Omega$) with and without SFT}
\vspace{-.1in}
\label{tab:refinement_results_sft}
\setlength{\tabcolsep}{3pt}
\resizebox{0.78\textwidth}{!}{
\begin{tabular}{l | cc | cc | cc | cc | cc | cc | cc | cc |}
\toprule
\multirow{2}{*}{\textbf{Model}} &  
\multicolumn{2}{c|}{\textbf{MATH}} & \multicolumn{2}{c|}{\textbf{AIME}} & 
\multicolumn{2}{c|}{\textbf{CoNaLa}} & \multicolumn{2}{c|}{\textbf{APPS}} &  
\multicolumn{2}{c|}{\textbf{Ruletaker}} & \multicolumn{2}{c|}{\textbf{ZebraLogic}} & 
\multicolumn{2}{c|}{\textbf{ARC Easy}} & \multicolumn{2}{c}{\textbf{GPQA}} \\
\cmidrule(lr){2-3} \cmidrule(lr){4-5} \cmidrule(lr){6-7} \cmidrule(lr){8-9} \cmidrule(lr){10-11} \cmidrule(lr){12-13} \cmidrule(lr){14-15} \cmidrule(lr){16-17}
& $\Omega$ & $\hat{\Omega}$ & $\Omega$ & $\hat{\Omega}$ & $\Omega$ & $\hat{\Omega}$ & $\Omega$ & $\hat{\Omega}$ & $\Omega$ & $\hat{\Omega}$ & $\Omega$ & $\hat{\Omega}$ & $\Omega$ & $\hat{\Omega}$ & $\Omega$ & $\hat{\Omega}$  \\
\midrule
\midrule


Llama3-8B & 3.0 & 807.8 & 3.4 & 3409.5 & 2.7 & 649.3 & 3.2 & 888.5 & 3.1 & 578.7 & 2.5 & 1835.0 & 2.2 & 21.1 & 2.8 & 357.2 \\ 

Llama3-8B-SFT & 1.4 & 156.1 & 1.6 & 378.5 & 3.3 & 141.2 & 2.4 & 175.3 & 2.3 & 9.3 & 3.0 & 8.5 & 1.3 & 11.2 & 1.3 & 164.1 \\ 

\midrule

Qwen2.5-7B & 1.5 & 629.7 & 2.4 & 2363.9 & 1.8 & 558.5 & 3.1 & 1861.0 & 2.4 & 379.4 & 3.4 & 2439.0 & 2.1 & 336.5 & 2.63 & 1270.1 \\ 

Qwen2.5-7B-SFT & 1.7 & 227.2 & 2.7 & 1307.2 & 2.5 & 85.2 & 4.1 & 545.1 & 2.2 & 7.2 & 2.5 & 6.0 & 1.63 & 13.8 & 1.85 & 43.8 \\

\bottomrule
\bottomrule
\end{tabular}%
}
\vspace{-.1in}
\end{table*}

\subsection{Semantic Analysis}


Beyond analyzing the confidence trajectory, we also utilize the decoding function 
to visualize the semantic evolution of hidden states by mapping the decoded logits to the vocabulary space. Figure~\ref{fig:semantic_decoding} shows the semantics of Logit Lens (left, $m=0$) and \methodname (right, $m=5$) for Qwen2.5-7B, for other models and focal depths, see Appendix~\ref{appen:complete_semantic_decoding}. 

Consistent with prior observations~\cite{belrose2023eliciting,wang2025logitlens4llms}, Logit Lens fails to decode meaningful semantics on models like Qwen except for the very last layer. In Figure~\ref{fig:semantic_decoding}(left), we can only see meaningful semantics in the final layer. As a comparison, \methodname effectively decodes intelligible and highly suggestive semantics from the intermediate layers. Most surprisingly, we observe a pattern of ``evolution''.  For instance, in middle layers, we decode high-level and self-reflection tokens like \texttt{i am not sure} and \texttt{i want to know}. These semantics suggest that before converging to the definitive answer, the model undergoes a thinking phase. Our results indicate that unlike Tuned Lens~\cite{belrose2023eliciting} that requires an extra learned adapter to correct the semantic alignment, the model's own future layers inherently act as a natural and zero-shot adapter to provide effective semantic signals.

\section{Diagnosing Blind Confidence of SFT}
\label{sec:SFT}

\begin{table}[t]
\centering
\caption{Accuracy with and without SFT.}
\label{tab:accuracy_sft}
\vspace{-.1in}
\setlength{\tabcolsep}{3pt}
\resizebox{0.9\linewidth}{!}{
\begin{tabular}{l | cc | cc | cc | cc}
\toprule
\multirow{2}{*}{\textbf{Model}} & \multicolumn{2}{c|}{\textbf{Math}} & \multicolumn{2}{c|}{\textbf{Code}} & 
\multicolumn{2}{c|}{\textbf{Logic}} & \multicolumn{2}{c}{\textbf{General}} \\
\cmidrule(lr){2-3} \cmidrule(lr){4-5} \cmidrule(lr){6-7} \cmidrule(lr){8-9}
 & \textbf{Easy} & \textbf{Hard} & \textbf{Easy} & \textbf{Hard} & \textbf{Easy} & \textbf{Hard} & \textbf{Easy} & \textbf{Hard} \\
\midrule
\midrule


Llama3-8B & 52\% & 4\% & 91\% & 70\% & 57\% & 42\% & 93\% & 37\% \\ 

Llama3-8B-SFT & 24\% & 3\% & 73\% & 59\% & 54\% & 18\% & 82\% & 25\% \\ 

\midrule

Qwen2.5-7B & 69\% & 9\% & 87\% & 71\% & 71\% & 39\% & 94\% & 28\% \\ 

Qwen2.5-7B-SFT & 55\% & 4\% & 81\% & 63\% & 44\% & 20\% & 96\% & 27\% \\ 

\bottomrule
\bottomrule
\end{tabular}%
}
\vspace{-.2in}
\end{table}

In this section, we investigate the impact of SFT on the model's performance using \methodname.

\para{Setup.} We employ two base models, Llama3-8B and Qwen2.5-7B, and four datasets, GSM8K~\cite{cobbe2021training}, MBPP~\cite{austin2021program}, babi QA~\cite{weston2015towards}, and OpenBookQA~\cite{mihaylov2018can} for the four domains, Math, Code, Logic, and General QA, respectively, to conduct SFT using full parameter fine-tuning of Llama-Factory~\cite{zheng2024llamafactory}. The difficulty of the four datasets is between the easy and hard tasks stated in Section~\ref{sec:empirical}. We then analyze the impact using the original easy and hard task datasets on models without and with SFT, which is denoted with suffix ``-SFT''.

\para{Performance degradation.} Table~\ref{tab:accuracy_sft} shows the accuracy results on the original easy/hard task datasets with and without SFT. 
We observe a significant drop in accuracy across almost all the $16 = 2 (\text{models}) \times 2 (\text{SFT epochs}) \times 4 (\text{domains})$ cases, except for one corner case, i.e., Qwen-2.5-7B on the easy task of General QA. The results suggest that simple SFT is prone to degrading the model performance. 

\para{Mechanistic interpretation via \methodname.}
To interpret the internal mechanism, we use \methodname to visualize the confidence trajectory. Figure~\ref{fig:ordinary_sft} shows the trajectory of Llama3-8B with and without SFT on AIME. For full results, see Appendix~\ref{appen:SFT}. We observe that the confidence trajectory becomes higher after SFT, suggesting that simple SFT downgrades the cognitive effort by letting the model skip the necessary low-confidence struggle phase. We term this phenomenon as ``blind confidence''.


Quantitatively, Table~\ref{tab:refinement_results_sft} shows the refinement area ($\Omega$). We can see that while $\Omega$ decreases in most cases, it paradoxically increases in Math and Code domains for Qwen-2.5-7B. From detailed analysis of the model's outputs (see Appendix~\ref{appen:eg_sft}), we find that Qwen2.5-7B on these domains tends to output significantly shorter yet high-quality core tokens, like formulas and code snippets, without any helpful tokens like explanations and step-by-step considerations directly. 
Because $\Omega$ is defined as the average effort per token (see Equation~\ref{eq:conf_traj}), we also calculate the overall refinement area $\hat{\Omega}$ that aggregates $\Omega$ of the entire output sequence\footnote{{\footnotesize$\hat{\Omega} = \sum_{t} \Omega_{\text{each token}}, \forall t \in [1, T]$ and $\Omega_{\text{each token}}$ is calculated by measuring the confidence of each layer using a single output token.}} to capture the entire effort. We can see that $\hat{\Omega}$ sharply decreases across all cases after SFT. This indicates SFT reduces the model's total effort by lowering the per-token cognitive effort or omitting the helpful tokens, ultimately leading to the observed performance degradation.

\begin{figure}[t]
    \centering
    \includegraphics[width=0.75\linewidth]{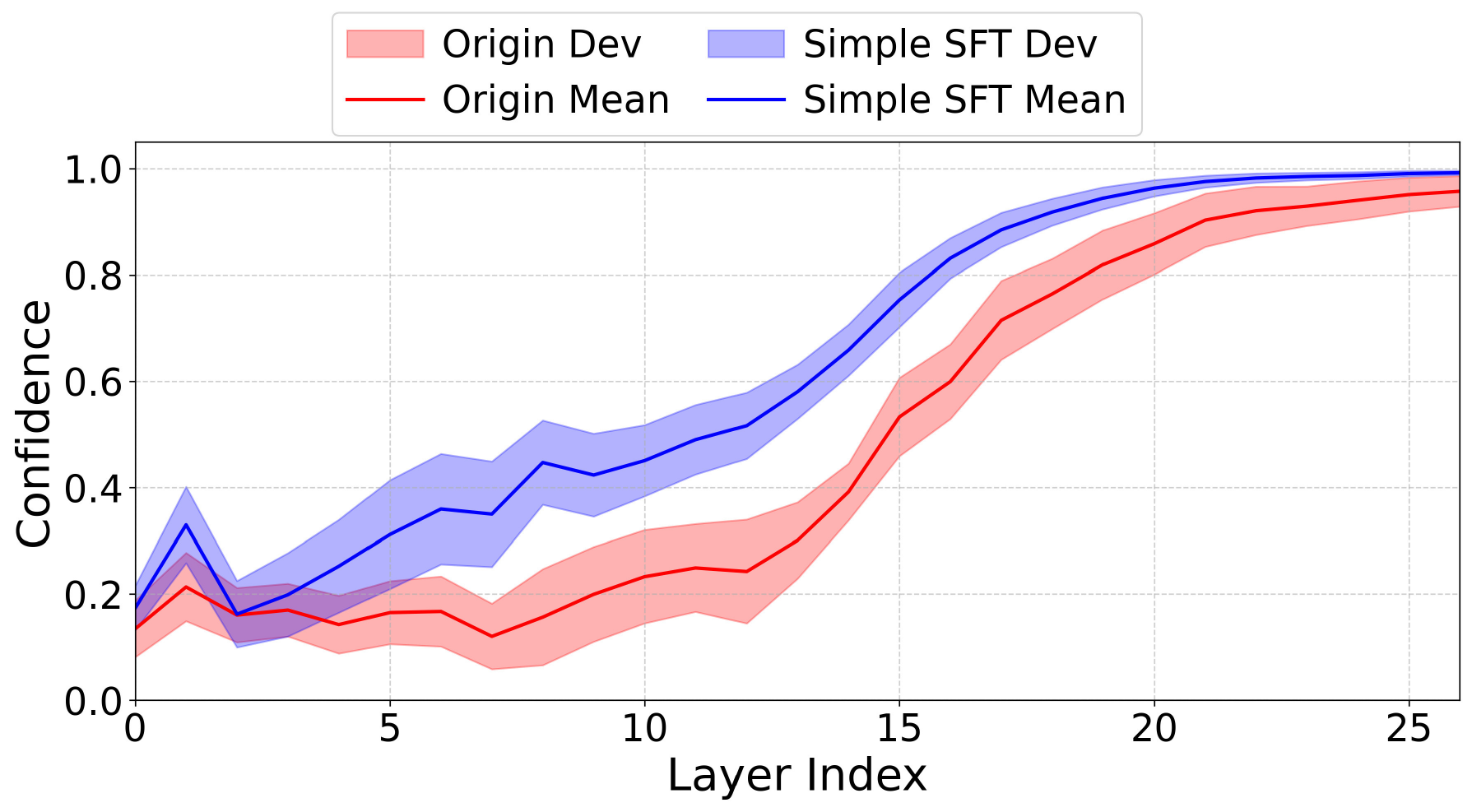}
    \vspace{-.1in}
    \caption{The confidence trajectory of Llama3-8B with and without SFT on AIME.
    }
    \label{fig:ordinary_sft}
    \vspace{-.2in}
\end{figure}





%% file: src/conclusion.tex
\section{Conclusion}
\label{sec:conclusion}

By leveraging the transformer’s inherent self-magnifying mechanism, we develop \methodname to precisely quantify a model's cognitive effort, overcoming the limitations of current probes.
The effectiveness of \methodname actually provides a novel insight for the interpretability community, i.e., the most effective probes through which to view a model's internal mysteries may be the model itself. Our findings suggest that leveraging the model's own features, such as the future layers, can resolve highly meaningful signals that are otherwise lost in the high-dimensional latent space. 
For future work, we will explore the utility of \methodname in other applications, including performance improvement, harmful behavior detection, and the interpretability of abnormal phenomena. More importantly, we plan to uncover additional intrinsic properties within the transformer architecture that can be leveraged for self-interpretation. By pursuing these paths, we aim to provide fresh insights and inspire future breakthroughs in AI interpretability.

\clearpage

\section*{Impact Statement}

This paper presents work whose goal is to enhance the transparency and interpretability of Large Language Models (LLMs). By introducing \methodname to quantify the internal cognitive effort of models, our work contributes to the development of more reliable and robust AI systems.

%% file: src/appendix.tex
\section{Proofs of results}
\label{appen:proof}

\subsection{Proof of Assumption~\ref{assump-lipschitz}}
\label{subappen:proof_assumption}

\begin{proof}
    From the fundamental theorem of calculus, we can express the difference as an integral along the path $H(t) = H + t\Delta$ for $t \in [0,1]$:
    \begin{equation}
        g_m(H + \Delta) - g_m(H) = \int_{0}^{1} \langle \nabla g_m(H + t\Delta), \Delta \rangle \, dt.
    \end{equation}
    Subtracting the first-order approximation $\langle \nabla g_m(H), \Delta \rangle$ from both sides, we analyze the remainder:
    \begin{align}
        &\left| g_m(H + \Delta) - g_m(H) - \langle \nabla g_m(H), \Delta \rangle \right| \nonumber \\
        &= \left| \int_{0}^{1} \langle \nabla g_m(H + t\Delta) - \nabla g_m(H), \Delta \rangle \, dt \right| \nonumber \\
        &\le \int_{0}^{1} \| \nabla g_m(H + t\Delta) - \nabla g_m(H) \| \cdot \|\Delta\| \, dt \quad (\text{Cauchy-Schwarz}) \nonumber \\
        &\le \int_{0}^{1} \beta \| t\Delta \| \cdot \|\Delta\| \, dt \quad (\text{by Assumption~\ref{assump-lipschitz}}) \nonumber \\
        &= \beta \|\Delta\|^2 \int_{0}^{1} t \, dt = \frac{\beta}{2} \|\Delta\|^2.
    \end{align}
    Removing the absolute value gives the lower bound: $g_m(H+\Delta) - g_m(H) \ge \langle \nabla g_m(H), \Delta \rangle - \frac{\beta}{2}\|\Delta\|^2$.
\end{proof}

\subsection{Proof of Theorem \ref{thm:monotonicity} (Monotonicity of Confidence)}
\label{subappen:proof_thm1}

To make the notation simpler and easier to follow, we assign a specific subscript to the residual update in this and the next subsection. It represents the update contributed by the $k$-th residual block on the hidden state at layer $i$:
$\Delta_{i,k}(H_{i, t}) = \mathcal{F}_k(H_{i, t}), ~k > i$.

In this section, we provide the derivation for the probabilistic bound on the confidence growth across layers. We establish that with sufficient alignment drift $\mu$ and sequence length $T$, the average margin $\bar{\mathcal{C}}^{(m)}_i$ increases monotonically with high probability.

\subsubsection{Step 1: Local Gain Analysis}
Consider the transition from layer $i$ to $i+1$. For a specific token at position $t$, the hidden state updates as $H_{i+1, t} = H_{i, t} + \Delta_{i+1, t}$. We define the single-step margin gain random variable $X_{i,t}$ as:
\begin{equation}
    X_{i,t} := g_m(H_{i+1, t}) - g_m(H_{i, t}).
\end{equation}
Invoking the \textbf{Quadratic Smoothness Bound} (Proposition~\ref{prop-smooth-bound}), which relates the function change to the gradient and the norm of the update, we have:
\begin{equation}
    X_{i,t} \ge \langle \nabla g_m(H_{i, t}), \Delta_{i+1, t} \rangle - \frac{\beta}{2} \|\Delta_{i+1, t}\|^2.
\end{equation}
Taking the conditional expectation with respect to the current state (filtration $\mathcal{F}_{t}$), and applying Assumption~\ref{assump-drift} (Positive Alignment Drift) and Assumption~\ref{assump-bound} (Boundedness $\|\Delta\| \le R$):
\begin{equation}
\begin{aligned}
    \mathbb{E}[X_{i,t} \mid \mathcal{F}_{t}] &\ge \mathbb{E}\left[\langle \nabla g_m(H_{i, t}), \Delta_{i+1, t} \rangle\right] - \frac{\beta}{2} \mathbb{E}\left[\|\Delta_{i+1, t}\|^2\right] \\
    &\ge \mu - \frac{\beta}{2} R^2.
\end{aligned}
\end{equation}
We define the \textbf{effective drift} $\gamma := \mu - \frac{\beta}{2} R^2$. Assuming the alignment component dominates the curvature penalty, we have $\gamma > 0$. Thus, the expected gain for any single token is strictly positive:
\begin{equation}
    \mathbb{E}[X_{i,t}] \ge \gamma.
\end{equation}

\subsubsection{Step 2: Sequence Aggregation and Martingale Construction}
We aim to bound the behavior of the average margin over the prompt of length $T$. The change in the average margin is:
\begin{equation}
    \Delta \bar{\mathcal{C}} = \bar{\mathcal{C}}^{(m)}_{i+1} - \bar{\mathcal{C}}^{(m)}_{i} = \frac{1}{T} \sum_{t=1}^T X_{i,t}.
\end{equation}
Let $S_T = \sum_{t=1}^T X_{i,t}$ be the total gain. We construct a Martingale Difference Sequence (MDS) to analyze the fluctuations of $S_T$ around its expectation. Let $Y_t$ be the centered variable:
\begin{equation}
    Y_t = X_{i,t} - \mathbb{E}[X_{i,t} \mid \mathcal{F}_{t-1}].
\end{equation}
The sequence $\{Z_k\}_{k=1}^T$ defined by $Z_k = \sum_{t=1}^k Y_t$ is a martingale. According to Assumption~\ref{assump-bound}, the fluctuations of the margin are bounded by $b$. Specifically, we assume the centered random variable is bounded such that $|Y_t| \le b$ almost surely.

\subsubsection{Step 3: Concentration via Azuma-Hoeffding}
We are interested in the failure probability, defined as the event where the average margin does not increase (i.e., $\Delta \bar{\mathcal{C}} \le 0$). This is equivalent to sum $S_T \le 0$.

Decomposing the sum:
\begin{equation}
    S_T = \sum_{t=1}^T \mathbb{E}[X_{i,t}] + \sum_{t=1}^T Y_t \ge T\gamma + Z_T.
\end{equation}
Thus, the event $S_T \le 0$ implies $T\gamma + Z_T \le 0$, or $Z_T \le -T\gamma$. We apply the \textbf{Azuma-Hoeffding Inequality} for bounded martingale differences. For any $\epsilon > 0$:
\begin{equation}
    \mathbb{P}(Z_T \le -\epsilon) \le \exp\left( -\frac{\epsilon^2}{2 \sum_{t=1}^T c_t^2} \right),
\end{equation}
where $c_t$ is the bound on $Y_t$. Since $|Y_t| \le b$, the range of $Y_t$ is at most $2b$, relying on the standard form for bounded random variables centered at 0, strictly $|Y_t| \le b$ implies the step squared in the denominator is bounded.
Substituting $\epsilon = T\gamma$:
\begin{equation}
\begin{aligned}
    \mathbb{P}\left(\Delta \bar{\mathcal{C}} \le 0\right) = \mathbb{P}(S_T \le 0) &\le \mathbb{P}(Z_T \le -T\gamma) \\
    &\le \exp\left( -\frac{(T\gamma)^2}{2 \sum_{t=1}^T b^2} \right) \\
    &= \exp\left( -\frac{T^2 \gamma^2}{2 T b^2} \right) \\
    &= \exp\left( -\frac{T \gamma^2}{2b^2} \right).
\end{aligned}
\end{equation}
This concludes the proof. The average confidence increases monotonically with a probability that approaches 1 exponentially fast as the sequence length $T$ or the effective alignment strength $\gamma$ increases.
\qed

\subsection{Proof of Theorem~\ref{thm:head_lift} (Magnification of Focal Depth)}
\label{subappen:proof_thm_head_lift}

To prove Theorem~\ref{thm:head_lift}, we treat the increase in decoding head depth effectively as a ``virtual'' forward pass through an additional Transformer layer which is similar to \ref{subappen:proof_thm1}.

\subsubsection{Step 1: Local Function Analysis}
Let $D_m(\cdot)$ denote the function represented by the decoding head consisting of the final $m$ layers of the model plus the unembedding projection. Let $\mathcal{F}_{k}(\cdot)$ denote the transformation of the $k$-th Transformer block (where $k = N-m$ is the specific layer added when moving from depth $m$ to $m+1$). 

By definition, applying the deeper head $D_{m+1}$ to a hidden state $H_{i,t}$ is mathematically equivalent to passing $H_{i,t}$ through the block $\mathcal{F}_{N-m}$ first, and then applying the shallower head $D_m$ to the result. Formally:
\begin{equation}
    D_{m+1}(H_{i,t}) \equiv D_m\left( \mathcal{F}_{N-m}(H_{i,t}) \right).
\end{equation}
Consequently, the logit margin calculated by the deeper head, denoted as $\mathcal{C}^{(m+1)}_{i,t}$, can be expressed in terms of the margin function $g_m$ evaluated at a perturbed state:
\begin{equation}
    \mathcal{C}^{(m+1)}_{i,t} = g_m\left( \mathcal{F}_{N-m}(H_{i,t}) \right) = g_m(H_{i,t} + \Delta_{head}),
\end{equation}
where $\Delta_{head} = \mathcal{F}_{N-m}(H_{i,t}) - H_{i,t}$ represents the residual update contributed by the specific layer $N-m$.

\subsubsection{Step 2: Sequence Aggregation and Martingale Construction}
We are interested in the probability that the average margin increases, i.e., $\frac{1}{T}\sum_t (\mathcal{C}^{(m+1)}_{i,t} - \mathcal{C}^{(m)}_{i,t}) > 0$. Let $X_t = \mathcal{C}^{(m+1)}_{i,t} - \mathcal{C}^{(m)}_{i,t}$ be the instantaneous gain at token $t$.

Using the Quadratic Smoothness Bound (Proposition~\ref{prop-smooth-bound}), we expand the gain:
\begin{equation}
    X_t = g_m(H_{i,t} + \Delta_{head}) - g_m(H_{i,t}) \ge \langle \nabla g_m(H_{i,t}), \Delta_{head} \rangle - \frac{\beta}{2}\|\Delta_{head}\|^2.
\end{equation}

By Assumption~\ref{assump-bound}, the residual norm is bounded by $R$. By Assumption~\ref{assump-drift}, provided the layer index $N-m$ falls within the refinement phase ($N-m \ge i_0$), the alignment condition holds: $\mathbb{E}[\langle \nabla g_m, \Delta_{head} \rangle] \ge \mu$. Incorporating these into the expectation of $X_t$:
\begin{align}
    \mathbb{E}[X_t \mid \mathcal{F}_{t-1}] &\ge \mathbb{E}\left[ \langle \nabla g_m(H_{i,t}), \Delta_{head} \rangle \right] - \frac{\beta}{2}R^2 \\
    &\ge \mu - \frac{\beta}{2}R^2 := \gamma.
\end{align}
Here, $\gamma$ represents the strictly positive effective drift (under the condition that alignment strength $\mu$ dominates the curvature penalty).

\subsubsection{Step 3: Concentration via Azuma-Hoeffding}
We now bound the probability that the empirical average $\bar{X} = \frac{1}{T}\sum_{t=1}^T X_t$ is non-positive. Let $S_T = \sum_{t=1}^T X_t$. We construct a martingale difference sequence to control the deviation from the expected growth.

Since the single-step fluctuations are bounded by $b$ (Assumption~\ref{assump-bound}), we apply the Azuma-Hoeffding inequality. We analyze the event where the sum is negative:
\begin{equation}
    \mathbb{P}(S_T \le 0) = \mathbb{P}(S_T - \mathbb{E}[S_T] \le -\mathbb{E}[S_T]).
\end{equation}
Since $\mathbb{E}[S_T] \ge T\gamma$, this implies:
\begin{equation}
    \mathbb{P}(S_T \le 0) \le \mathbb{P}(S_T - \mathbb{E}[S_T] \le -T\gamma).
\end{equation}
Applying the inequality for bounded random variables:
\begin{equation}
    \mathbb{P}\left( \bar{\mathcal{C}}^{(m+1)}_i - \bar{\mathcal{C}}^{(m)}_i \le 0 \right) \le \exp\left( -\frac{(T\gamma)^2}{2 \sum_{t=1}^T b^2} \right) = \exp\left( -\frac{T^2 \gamma^2}{2 T b^2} \right) = \exp\left( -\frac{T \gamma^2}{2 b^2} \right).
\end{equation}
Conversely, the probability that the deeper head yields a strictly higher confidence is:
\begin{equation}
    \mathbb{P}\left( \bar{\mathcal{C}}^{(m+1)}_i > \bar{\mathcal{C}}^{(m)}_i \right) \ge 1 - \exp\left( -\frac{T \gamma^2}{2 b^2} \right).
\end{equation}
This completes the proof. \hfill $\square$

\section{Algorithms to compute refinement index automatically}
\label{appen:algorithm}

\subsection{Algorithm to compute re}
\label{subappen:algo-re}

\begin{algorithm}[H]
  \caption{Calculation of Refinement End ($re$)}
  \label{alg:calc_re}
  \begin{algorithmic}
    \STATE {\bfseries Input:} Data sequence $\mathcal{C} = \{\mathcal{C}_0, \mathcal{C}_1, \dots, \mathcal{C}_k\}$, scalar $threshold$
    \STATE {\bfseries Output:} Index $re$
    
    \STATE Initialize $re = 0$  \COMMENT{Find the minimum index $i$ such that $\forall j \geq i, \mathcal{C}_j \geq \mathcal{C}_k - threshold$}
    \FOR{$i = k-1$ {\bfseries down to} $0$}
        \IF{$\mathcal{C}_i < \mathcal{C}_k - threshold$}
            \STATE $re = i + 1$
            \STATE {\bfseries break}
        \ENDIF
    \ENDFOR
    
    \STATE \textbf{return} $re$
  \end{algorithmic}
\end{algorithm}
\vspace{-.1in}

In the refinement layers, we observe that almost all the curve increases firstly, then maintains in a relatively high value. We call the middle variable Refinement End (re). To robustly find the index, we set a threshold value. Specifically in the experiment, we set $threshold = 0.07$。

\subsection{Algorithm to compute $i_0$}
\label{subappen:algo-i_0}

\begin{algorithm}[H]
  \caption{Calculation of $i_0$}
  \label{alg:calc_i_0}
  \begin{algorithmic}
    \STATE {\bfseries Input:} Data sequence $\mathcal{C}$, index $re$, scalar $threshold$
    \STATE {\bfseries Output:} Index $i_0$
    
    \STATE Initialize $rmin = re$ \COMMENT{Step 1: Find $rmin$ satisfying the convexity condition}
    \FOR{$i = re - 1$ {\bfseries down to} $0$}
        \IF{$i > 0$}
            \STATE $val = (\mathcal{C}_i + \mathcal{C}_{i-1}) / 2$
        \ELSE
            \STATE $val = \mathcal{C}_0$ \COMMENT{Using $\mathcal{C}_0 \leq \mathcal{C}_1$ boundary check logic}
        \ENDIF
        
        \IF{$val \leq \mathcal{C}_{i+1}$}
            \STATE $rmin = i$
        \ELSE
            \STATE {\bfseries break}
        \ENDIF
    \ENDFOR

    \STATE Initialize $i_0 = rmin$ \COMMENT{Step 2: Find $i_0$ based on $rmin$ and threshold}
    \FOR{$j = rmin$ {\bfseries to} $re - 1$}
        \IF{$j \geq rmin + threshold$} 
            \STATE $i_0 = j - 1$
            \STATE {\bfseries break}
        \ENDIF
    \ENDFOR

    \STATE \textbf{return} $i_0$
  \end{algorithmic}
\end{algorithm}
\vspace{-.1in}

Since the refinement layers are the layers that decode the latent space to the output vocabulary. By our observation, we find that in these layers, almost all the curves increase firstly, then maintains in a relatively high value for a while. So to find the $i_0$, we need to define the meaning of increase. To make it robust, we let the definition of increase to be $\frac{\mathcal{C}_{i-1} + \mathcal{C}_i}{2} \leq \mathcal{C}_{i+1}$. Using the definition of increase and the middle variable re, we can find another middle variable rmin which means the start index of robust increase. Similarly to calculate the middle variable re, we set a threshold to find $i_0$. Here we also set $threshold = 0.07$ in the experiments.

\subsection{Algorithm to compute Refinement Area ($\Omega$)}
\label{subappen:algo-ra}

\begin{algorithm}[H]
  \caption{Calculation of Refinement Area ($\Omega$)}
  \label{alg:calc_ra}
  \begin{algorithmic}
    \STATE {\bfseries Input:} Data sequence $\mathcal{C} = \{\mathcal{C}_0, \mathcal{C}_1, \dots, \mathcal{C}_k\}$, index $i_0$
    \STATE {\bfseries Output:} Value $\Omega$
    
    \STATE Initialize $\Omega = 0$ \COMMENT{Calculate Refinement Area $\Omega$ according to $i_0$}
    \FOR{$i = i_0$ {\bfseries to} $k$}
        \STATE $\Omega = \Omega + (1 - \mathcal{C}_i)$
    \ENDFOR
    
    \STATE \textbf{return} $\Omega$
  \end{algorithmic}
\end{algorithm}
\vspace{-.1in}

When we find the start point of refinement layers $i_0$, according to the definition in the main text, $\Omega = \sum_{i \ge i_0} (1 - \mathcal{C}_i)$ (area above the confidence curve).

\section{Full Results of Varied Focal Depth m}
\label{appen:varid_m}

In this section, we present the comprehensive experimental results regarding the impact of focal depth $m$ on confidence trajectories. Figures~\ref{fig:varied_m_math-level12} through~\ref{fig:varied_m_gpqa} display the confidence trajectories for four models (Llama3-8B, Qwen2.5-3B, Qwen3-0.6B, and Deepseek-7B) across eight diverse datasets: MATH (Level 1-2), AIME, CoNaLa, APPS, RuleTaker, ZebraLogic, ARC-Easy, and GPQA.

These extensive results universally corroborate the conclusions drawn in the main text:
\textbf{Universal Magnification:} The strict ordering of trajectories ($m=5 > m=3 > m=1$) holds across almost all diverse tasks, confirming that larger focal depths consistently magnify confidence signals.
\textbf{Robust Model Characteristics:} Both the model-specific sensitivity to $m$ and the intrinsic starting locations of the refinement layers almost remain consistent regardless of the input dataset.

\begin{figure*}[!htbp]
    \centering
    \includegraphics[width=\textwidth]{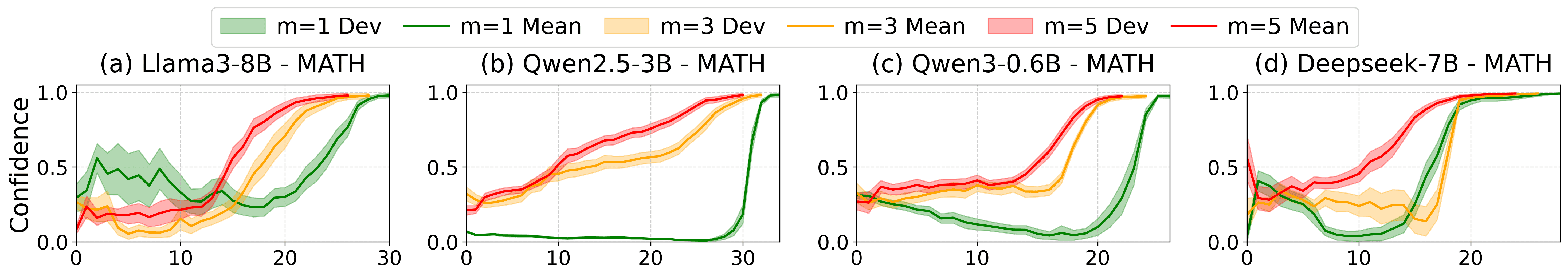}
    \caption{The confidence trajectory of Llama3-8B, Qwen2.5-3B, Qwen3-0.6B, Deepseek-7B on dataset Math (Level 1-2).
    }
    \label{fig:varied_m_math-level12}
    \vspace{-.1in}
\end{figure*}

\begin{figure*}[!htbp]
    \centering
    \includegraphics[width=\textwidth]{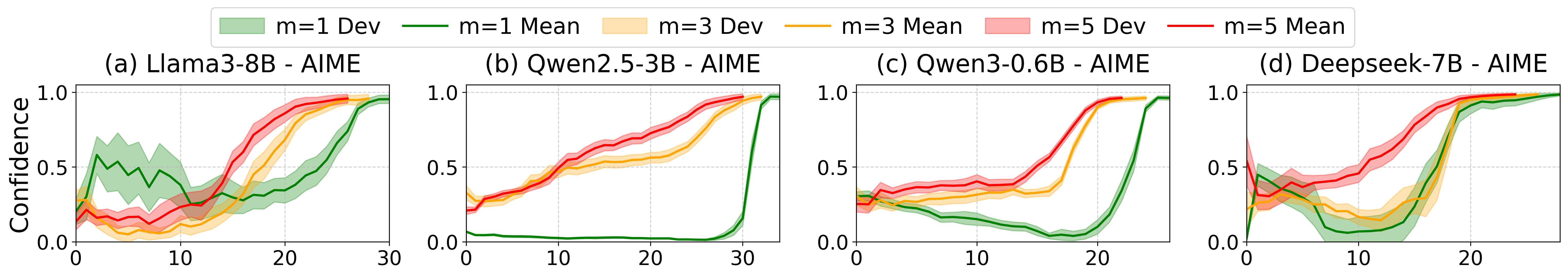}
    \caption{The confidence trajectory of Llama3-8B, Qwen2.5-3B, Qwen3-0.6B, Deepseek-7B on dataset AIME.
    }
    \label{fig:varied_m_aime}
    \vspace{-.1in}
\end{figure*}

\begin{figure*}[!htbp]
    \centering
    \includegraphics[width=\textwidth]{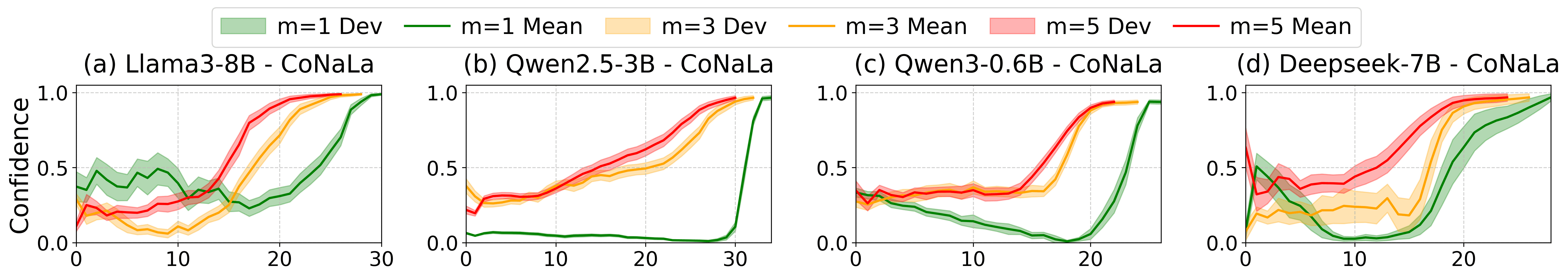}
    \caption{The confidence trajectory of Llama3-8B, Qwen2.5-3B, Qwen3-0.6B, Deepseek-7B on dataset CoNaLa.
    }
    \label{fig:varied_m_conala}
    \vspace{-.1in}
\end{figure*}

\clearpage

\begin{figure*}[!htbp]
    \centering
    \includegraphics[width=\textwidth]{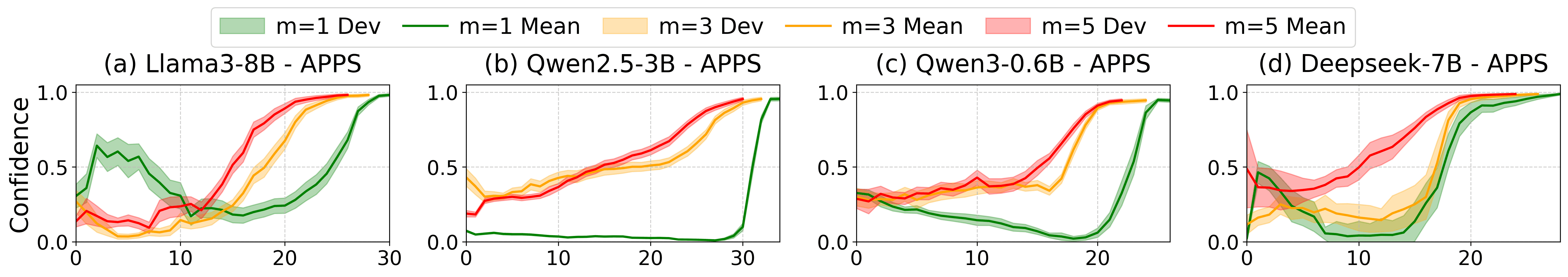}
    \caption{The confidence trajectory of Llama3-8B, Qwen2.5-3B, Qwen3-0.6B, Deepseek-7B on dataset APPS.
    }
    \label{fig:varied_m_apps}
\end{figure*}

\begin{figure*}[!htbp]
    \centering
    \includegraphics[width=\textwidth]{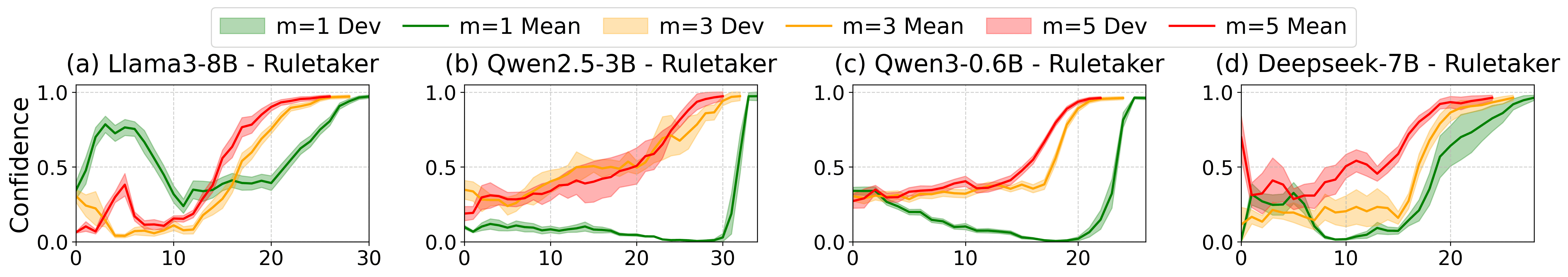}
    \caption{The confidence trajectory of Llama3-8B, Qwen2.5-3B, Qwen3-0.6B, Deepseek-7B on dataset Ruletaker.
    }
    \label{fig:varied_m_ruletaker}
\end{figure*}

\begin{figure*}[!htbp]
    \centering
    \includegraphics[width=\textwidth]{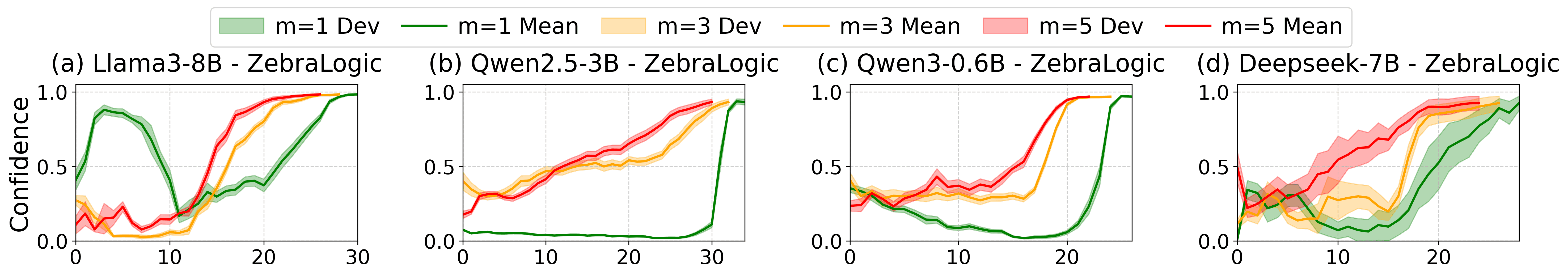}
    \caption{The confidence trajectory of Llama3-8B, Qwen2.5-3B, Qwen3-0.6B, Deepseek-7B on dataset ZebraLogic.
    }
    \label{fig:varied_m_logicbench}
\end{figure*}

\begin{figure*}[!htbp]
    \centering
    \includegraphics[width=\textwidth]{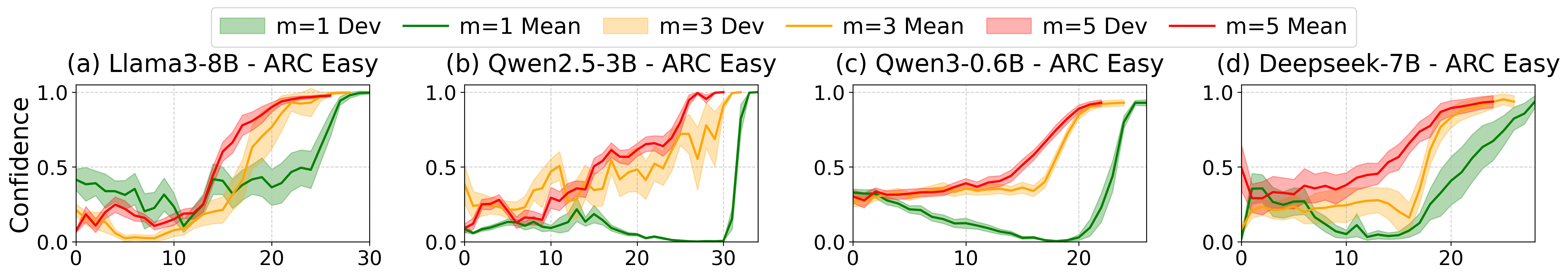}
    \caption{The confidence trajectory of Llama3-8B, Qwen2.5-3B, Qwen3-0.6B, Deepseek-7B on dataset ARC Easy.
    }
    \label{fig:varied_m_arc-easy}
\end{figure*}

\begin{figure*}[!htbp]
    \centering
    \includegraphics[width=\textwidth]{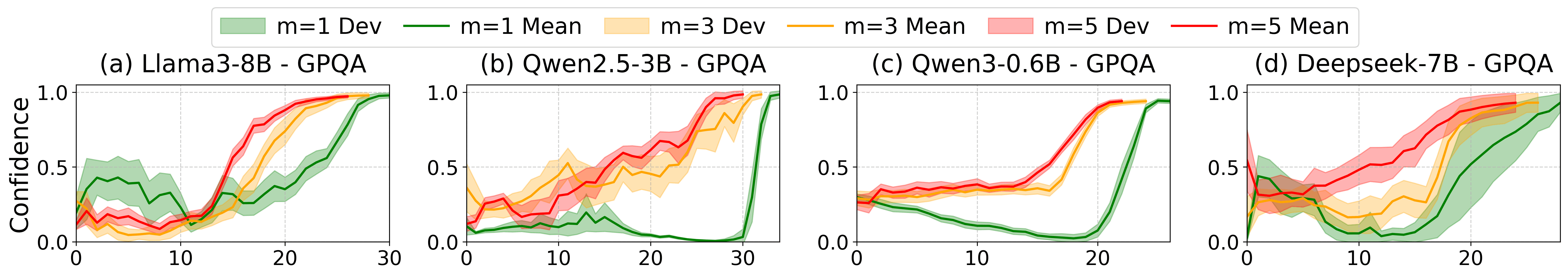}
    \caption{The confidence trajectory of Llama3-8B, Qwen2.5-3B, Qwen3-0.6B, Deepseek-7B on dataset GPQA.
    }
    \label{fig:varied_m_gpqa}
\end{figure*}

\clearpage

\section{Full Results of Cognitive Effort Difference}
\label{appen:cog_effort}

In the main text, we demonstrated the cognitive effort measurement using Qwen2.5-7B as a representative example. In this section, we provide the complete visualization results across a diverse set of LLMs with varying parameter sizes (0.5B to 32B) and architectures (Llama-3, Deepseek, Qwen2.5, and Qwen-3). We compare the confidence trajectories under standard decoding (focal depth $m=0$) and our look-ahead approach (e.g., $m=5$).

\subsection{Impact of Focal Depth on Trajectory Visibility}
Figures~\ref{fig:cognitive_llama3-8B_m0} through~\ref{fig:cognitive_qwen3-4b-m5} illustrate the confidence trajectories for different models. A distinct pattern emerges when comparing the focal depths:

\para{Standard Decoding ($m=0$):} As seen in Figure~\ref{fig:cognitive_llama3-8B_m0} (Llama3-8B) and Figure~\ref{fig:semantic_decoding_Qwen2.5-32B_m01} (Qwen2.5-32B), the confidence trajectories for the standard next-token prediction often exhibit a trivial step-function behavior or rapid saturation. In these cases, the confidence remains near zero or shoots to one immediately, masking the internal reasoning process. This confirms that standard output probabilities are insufficient for fine-grained cognitive effort measurement.

\para{Look-ahead Refinement ($m \ge 1$):} When a focal depth is applied (e.g., $m=5$ in Figure~\ref{fig:cognitive_llama3-8B_m5} and Figure~\ref{fig:cognitive_qwen2.5-32b-m5}), the trajectories reveal a dynamic evolution of confidence across layers. This validates our hypothesis that a look-ahead window is necessary to observe the gradual resolution of uncertainty, distinct from immediate token recognition.

\subsection{Consistency of Easy vs. Hard Distinction}
Across all evaluated models, the distinction between ``Easy'' (blue curves) and ``Hard'' (red curves) tasks remains consistent with the observations in the main text:

\para{Layer-wise Separation:} In models equipped with sufficient focal depth, the confidence for easy tasks consistently rises earlier and remains higher than that of hard tasks throughout the intermediate and later layers. For example, in Qwen2.5-7B with $m=5$ (Figure~\ref{fig:cognitive_qwen2.5-7b-m5}) and Qwen2.5-32B with $m=5$ (Figure~\ref{fig:cognitive_qwen2.5-32b-m5}), the blue region is clearly situated above the red region.
    
\para{Domain Sensitivity:} The gap between easy and hard trajectories remains most significant in the \textbf{Math} and \textbf{Code} domains across all model families. This supports the intuition that these reasoning-intensive tasks induce greater ``cognitive load,'' requiring more layers to bridge the confidence gap for harder instances. Conversely, the \textbf{General QA} domain frequently shows overlapping trajectories (e.g., Figure~\ref{fig:cognitive_ds_m5} (d)), reaffirming that knowledge retrieval tasks exhibit less layer-wise reasoning differentiation than logical tasks.

\subsection{Model Scale and Stability}
We also observe the effect of model scale on cognitive effort visualization. Larger models, such as Qwen2.5-32B (Figure~\ref{fig:cognitive_qwen2.5-32b-m5}), exhibit smoother trajectories with tighter confidence intervals. In contrast, smaller models like Qwen2.5-0.5B (Figure~\ref{fig:cognitive_qwen2.5-0.5b-m5}) and Qwen3-0.6B (Figure~\ref{fig:cognitive_qwen3-0.6b-m5}) display higher variance (indicated by wider shaded regions), reflecting their relative instability and higher uncertainty when dealing with complex reasoning tasks.

Collectively, these results demonstrate the universality of our cognitive effort measurement method across different model structures and scales.

\begin{figure*}[h]
    \centering
    \includegraphics[width=\textwidth]{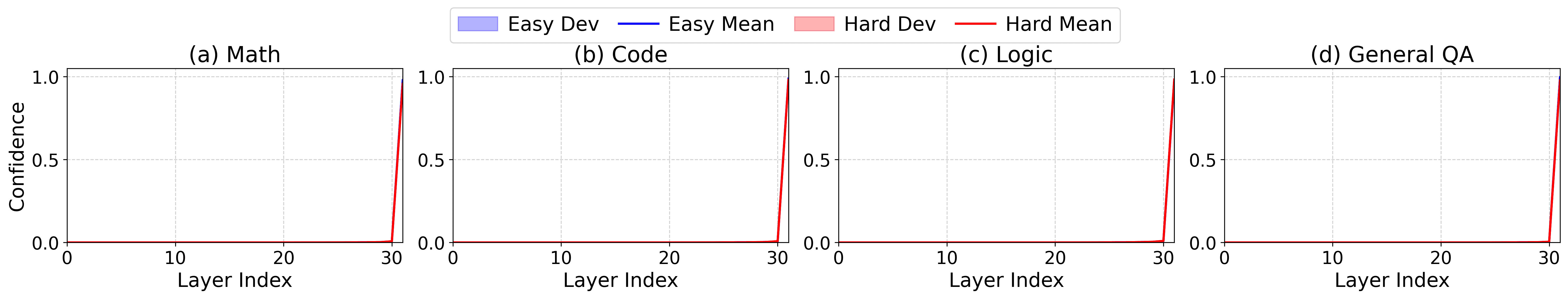}
    \caption{Cognitive Effort Difference of Llama3-8B with focal depth m = 0 for 4 domains.
    }
    \label{fig:cognitive_llama3-8B_m0}
\end{figure*}

\begin{figure*}[h]
    \centering
    \includegraphics[width=\textwidth]{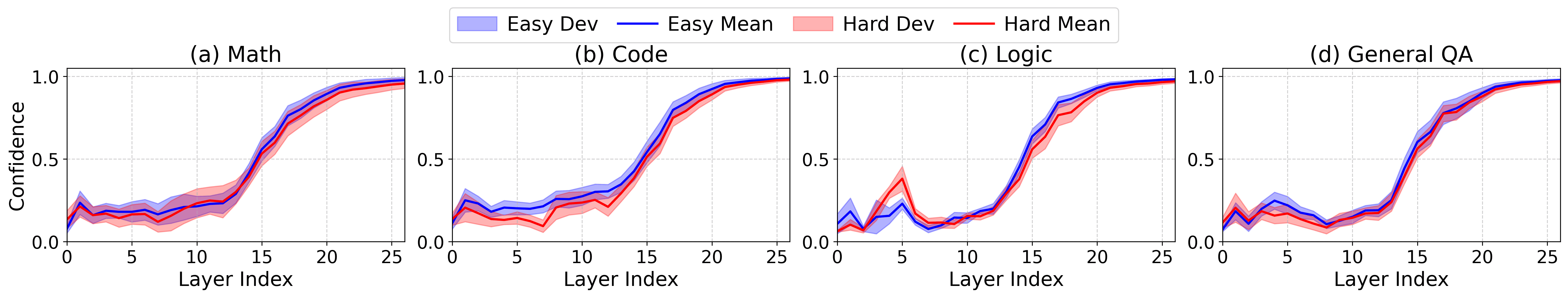}
    \caption{Cognitive Effort Difference of Llama3-8B with focal depth m = 5 for 4 domains.
    }
    \label{fig:cognitive_llama3-8B_m5}
\end{figure*}

\begin{figure*}[h]
    \centering
    \includegraphics[width=\textwidth]{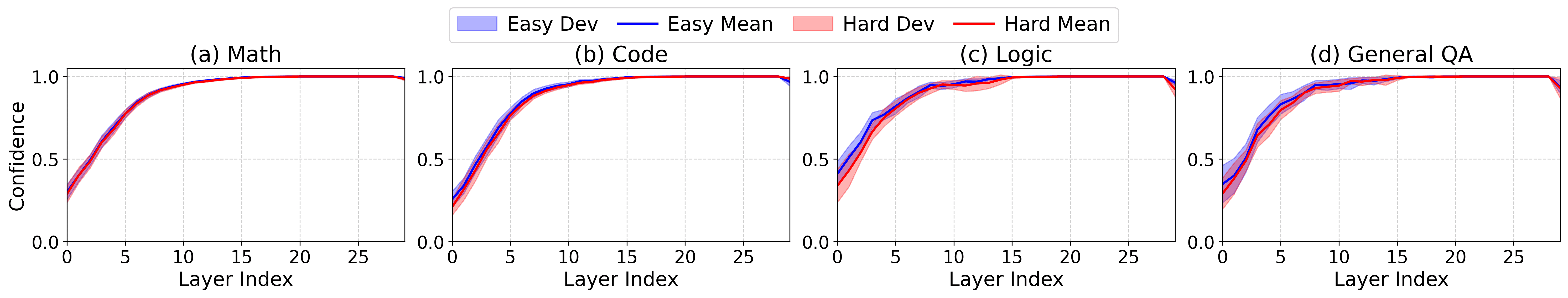}
    \caption{Cognitive Effort Difference of Deepseek-7B with focal depth m = 0 for 4 domains.
    }
    \label{fig:cognitive_ds_m0}
\end{figure*}

\begin{figure*}[h]
    \centering
    \includegraphics[width=\textwidth]{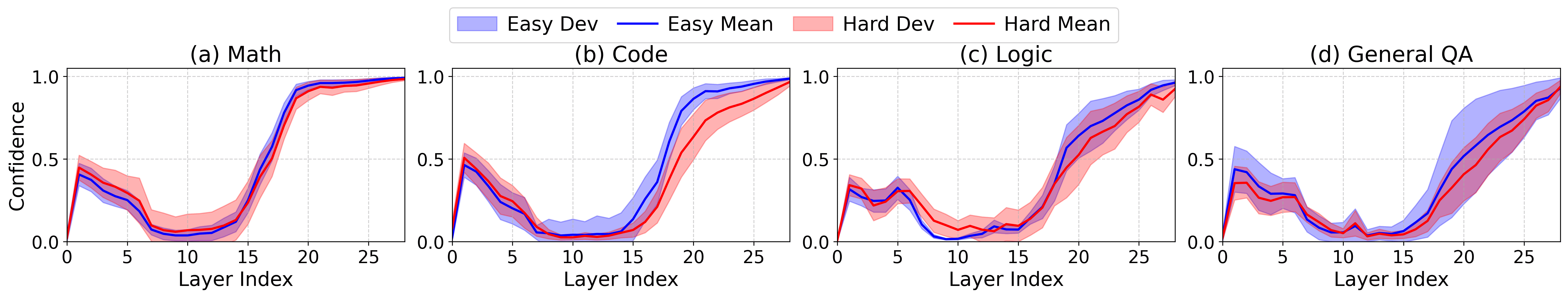}
    \caption{Cognitive Effort Difference of Deepseek-7B with focal depth m = 1 for 4 domains.
    }
    \label{fig:cognitive_ds_m5}
\end{figure*}

\begin{figure*}[h]
    \centering
    \includegraphics[width=\textwidth]{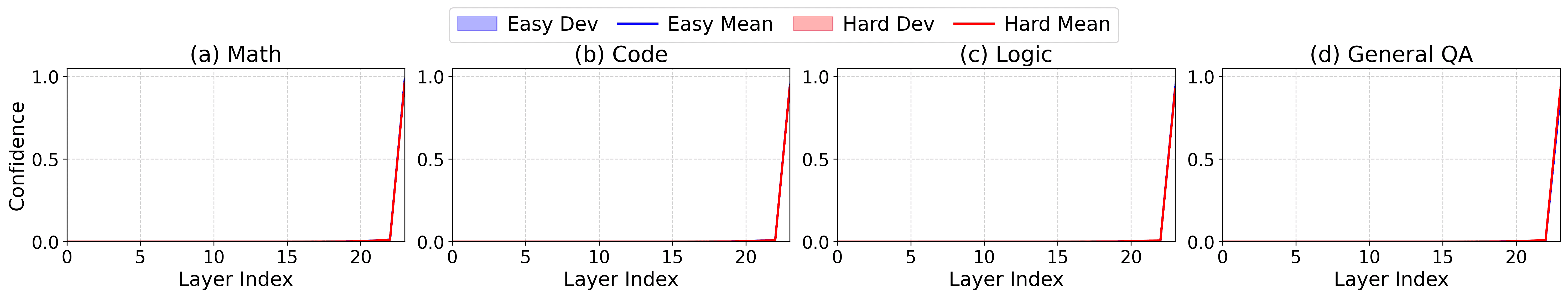}
    \caption{Cognitive Effort Difference of Qwen2.5-0.5B-Instruct with focal depth m = 0 for 4 domains.
    }
    \label{fig:cognitive_qwen2.5-0.5b-m0}
\end{figure*}

\begin{figure*}[h]
    \centering
    \includegraphics[width=\textwidth]{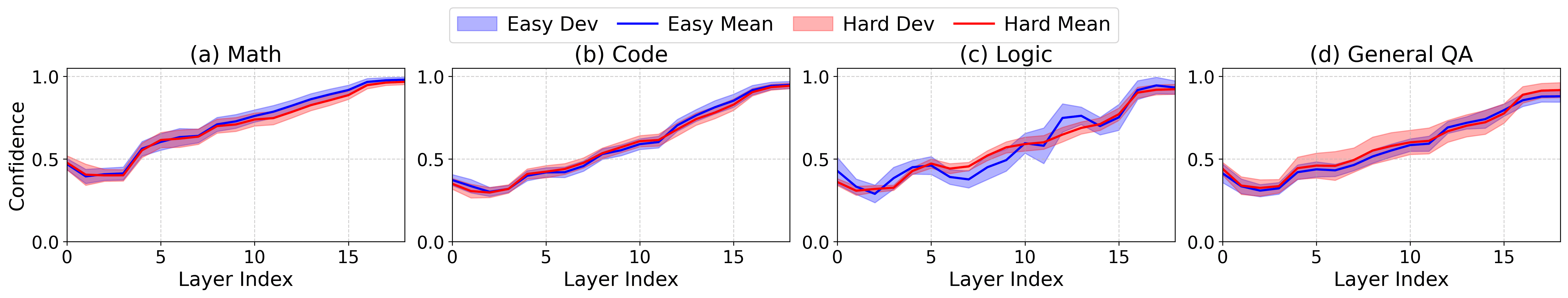}
    \caption{Cognitive Effort Difference of Qwen2.5-0.5B-Instruct with focal depth m = 5 for 4 domains.
    }
    \label{fig:cognitive_qwen2.5-0.5b-m5}
\end{figure*}

\begin{figure*}[h]
    \centering
    \includegraphics[width=\textwidth]{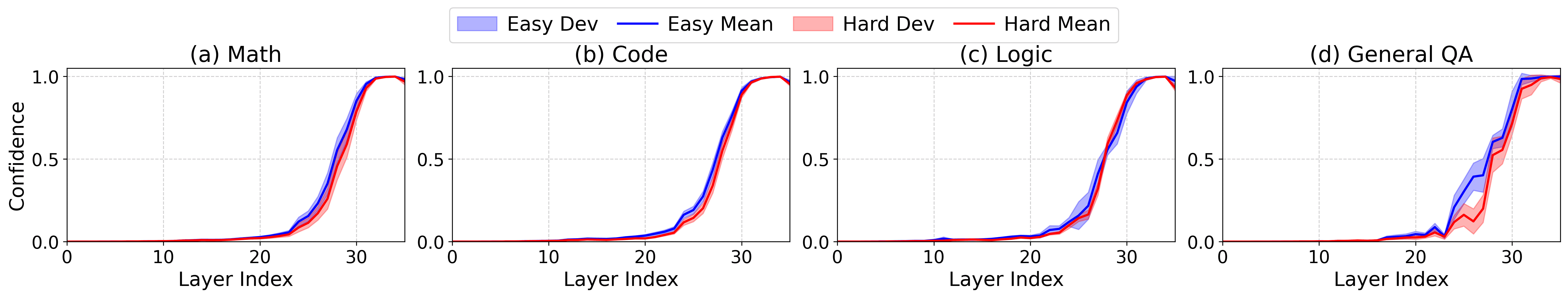}
    \caption{Cognitive Effort Difference of Qwen2.5-3B-Instruct with focal depth m = 0 for 4 domains.
    }
    \label{fig:cognitive_qwen2.5-3b-m0}
\end{figure*}

\begin{figure*}[h]
    \centering
    \includegraphics[width=\textwidth]{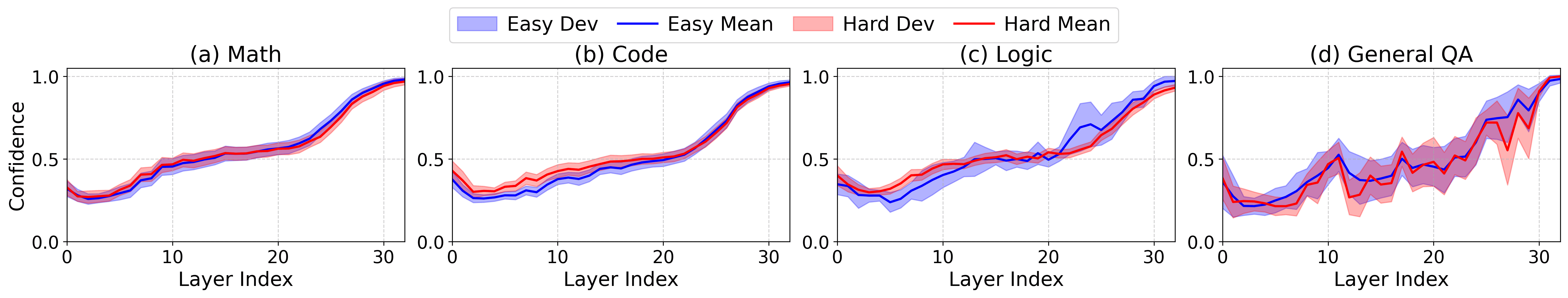}
    \caption{Cognitive Effort Difference of Qwen2.5-3B-Instruct with focal depth m = 3 for 4 domains.
    }
    \label{fig:cognitive_qwen2.5-3b-m3}
\end{figure*}

\begin{figure*}[h]
    \centering
    \includegraphics[width=\textwidth]{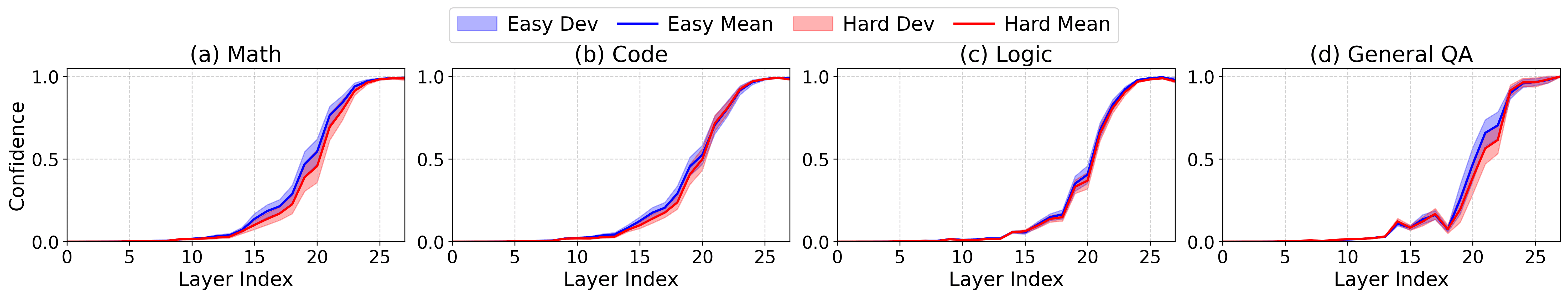}
    \caption{Cognitive Effort Difference of Qwen2.5-7B-Instruct with focal depth m = 0 for 4 domains.
    }
    \label{fig:cognitive_qwen2.5-7b-m0}
\end{figure*}

\begin{figure*}[h]
    \centering
    \includegraphics[width=\textwidth]{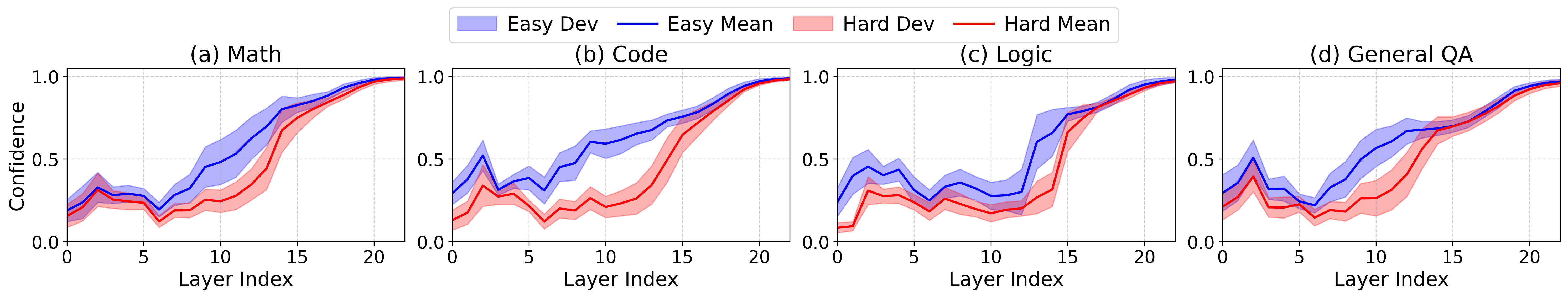}
    \caption{Cognitive Effort Difference of Qwen2.5-7B-Instruct with focal depth m = 5 for 4 domains.
    }
    \label{fig:cognitive_qwen2.5-7b-m5}
\end{figure*}

\begin{figure*}[h]
    \centering
    \includegraphics[width=\textwidth]{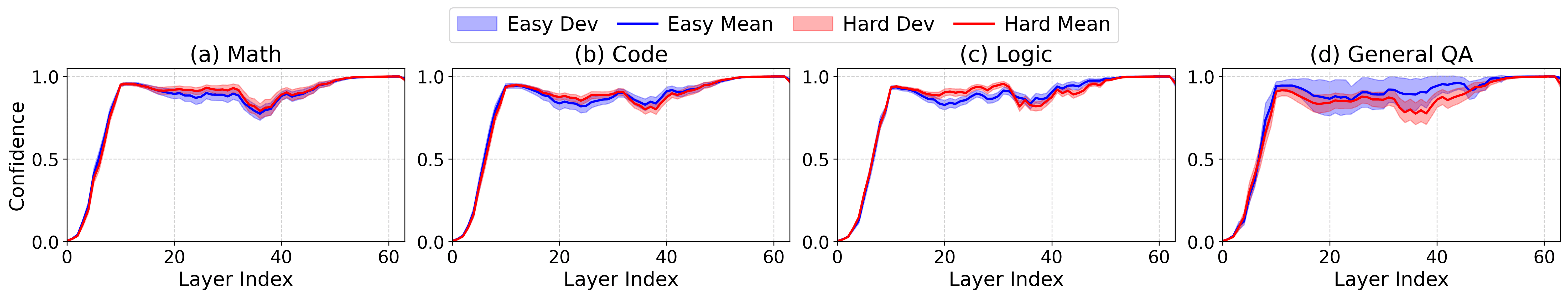}
    \caption{Cognitive Effort Difference of Qwen2.5-32B-Instruct with focal depth m = 0 for 4 domains.
    }
    \label{fig:cognitive_qwen2.5-32b-m0}
\end{figure*}

\begin{figure*}[h]
    \centering
    \includegraphics[width=\textwidth]{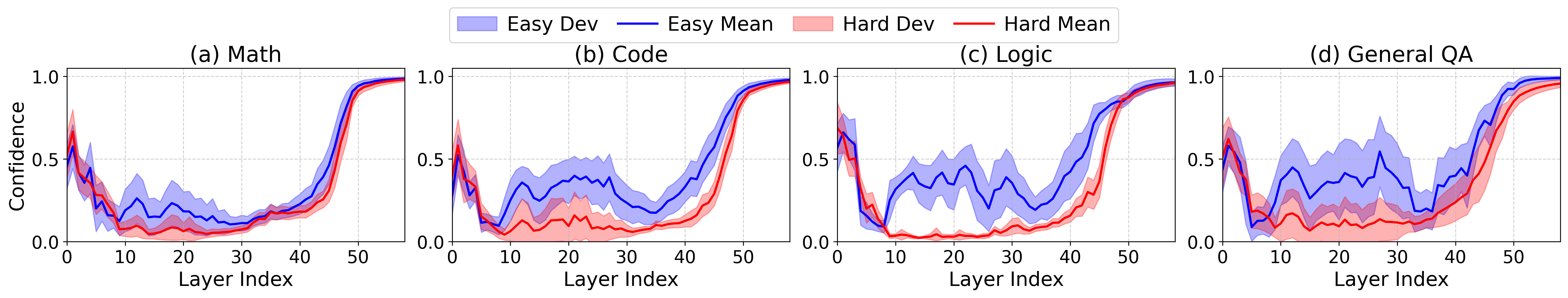}
    \caption{Cognitive Effort Difference of Qwen2.5-32B-Instruct with focal depth m = 5 for 4 domains.
    }
    \label{fig:cognitive_qwen2.5-32b-m5}
\end{figure*}

\begin{figure*}[h]
    \centering
    \includegraphics[width=\textwidth]{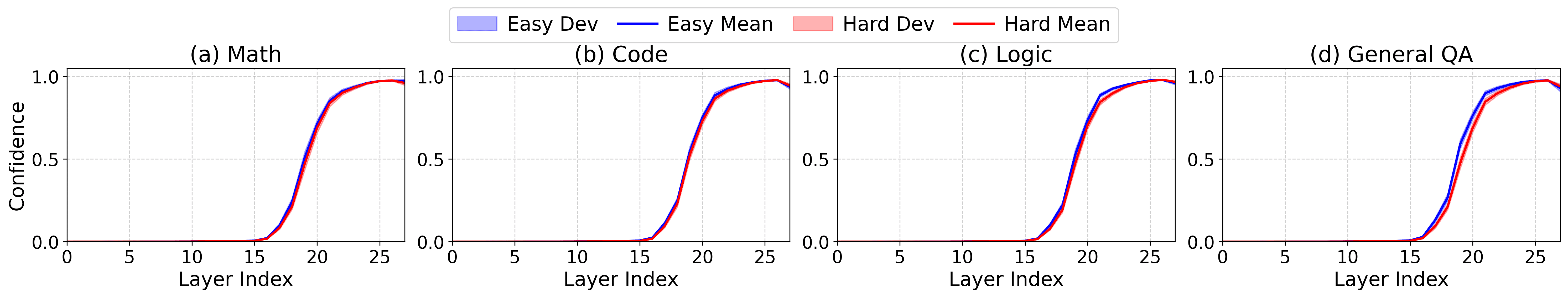}
    \caption{Cognitive Effort Difference of Qwen3-0.6B-Instruct with focal depth m = 0 for 4 domains.
    }
    \label{fig:cognitive_qwen3-0.6b-m0}
\end{figure*}

\begin{figure*}[h]
    \centering
    \includegraphics[width=\textwidth]{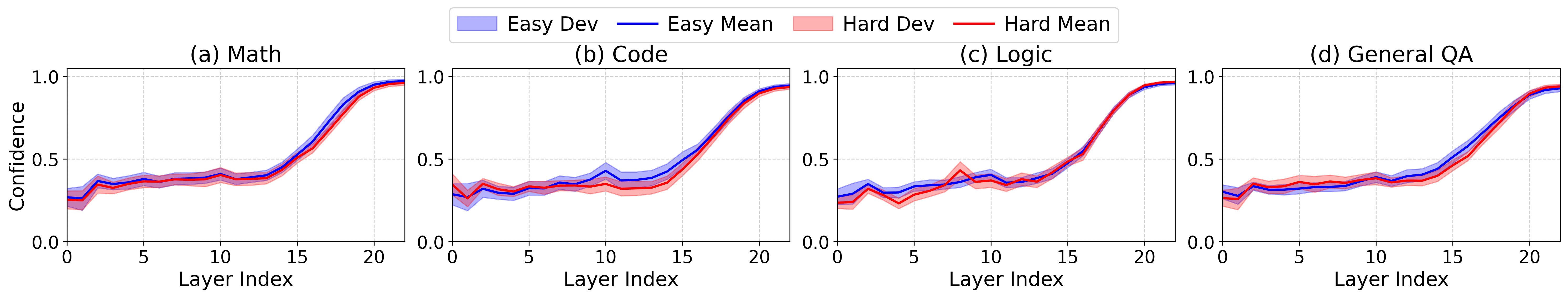}
    \caption{Cognitive Effort Difference of Qwen3-0.6B-Instruct with focal depth m = 5 for 4 domains.
    }
    \label{fig:cognitive_qwen3-0.6b-m5}
\end{figure*}

\begin{figure*}[h]
    \centering
    \includegraphics[width=\textwidth]{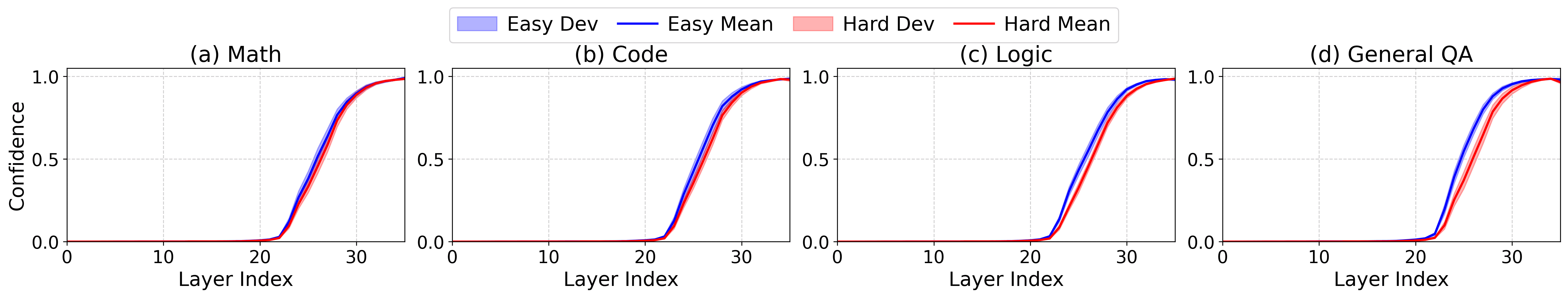}
    \caption{Cognitive Effort Difference of Qwen3-4B-Instruct with focal depth m = 0 for 4 domains.
    }
    \label{fig:cognitive_qwen3-4b-m0}
\end{figure*}

\begin{figure*}[h]
    \centering
    \includegraphics[width=\textwidth]{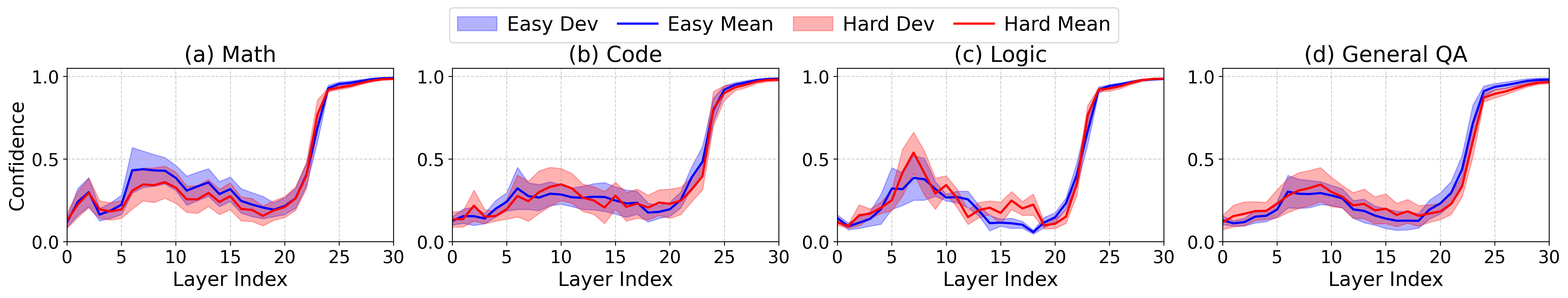}
    \caption{Cognitive Effort Difference of Qwen3-4B-Instruct with focal depth m = 5 for 4 domains.
    }
    \label{fig:cognitive_qwen3-4b-m5}
\end{figure*}

\clearpage

\section{Full Results of Semantic Decoding}
\label{appen:complete_semantic_decoding}

The following figures are the complete truncated results of semantic decoding among the eight models with input: What is the most popular neural network architecture in the world? Please give me some reasons.

This evolutionary trajectory is distinctly observable across different model architectures. For example, in Llama-3-8B (Firgure~\ref{fig:semantic_decoding_llama_m01} and Figure~\ref{fig:semantic_decoding_llama_m35}), while the baseline Logit Lens remains noisy in intermediate layers (L15-L20), our method clearly reveals the retrieval of domain-specific concepts such as ``CNN'', ``RNN'', and ``LSTM'' long before the final layer settles on the output. Similarly, in Qwen2.5-7B (Figure~\ref{fig:semantic_decoding_Qwen2.5-7B_m35}), when processing the query for reasons, the intermediate layers (L12-L18) explicitly decode intent-driven tokens like ``I'', ``want'', and ``know'', which serve as a hidden ``scratchpad'' for the model to formulate its response. This visualization provides qualitative evidence that the model narrows down the search space from abstract intents to concrete entities in a coarse to fine manner.

\begin{figure*}[h]
    \centering
    \includegraphics[width=0.45\linewidth]{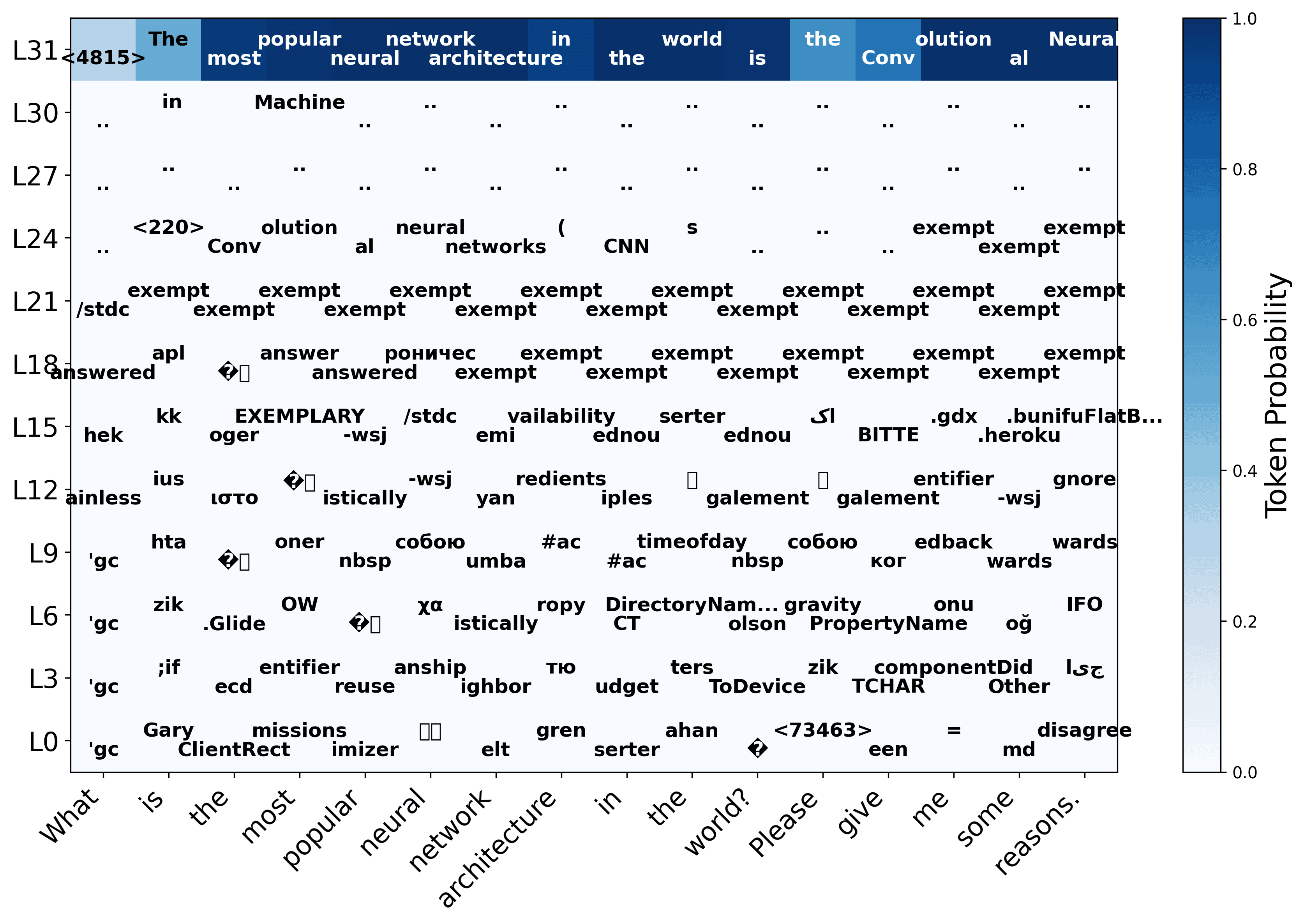}
    \includegraphics[width=0.45\linewidth]{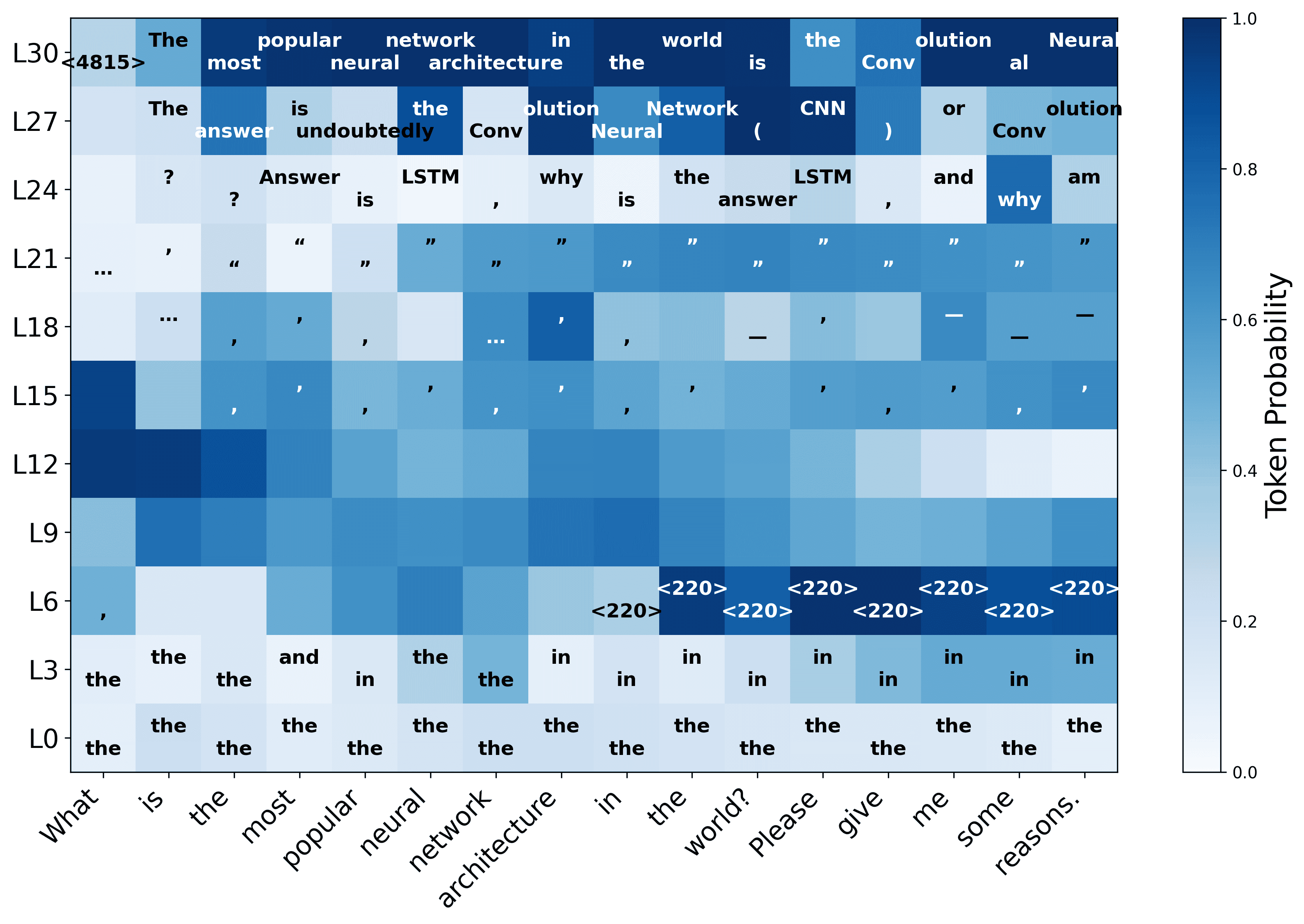}
    \caption{Semantic decoding of Llama3-8B ($m=0$ vs. $m=1$).}
    \label{fig:semantic_decoding_llama_m01}
\end{figure*}

\begin{figure*}[h]
    \centering
    \includegraphics[width=0.45\linewidth]{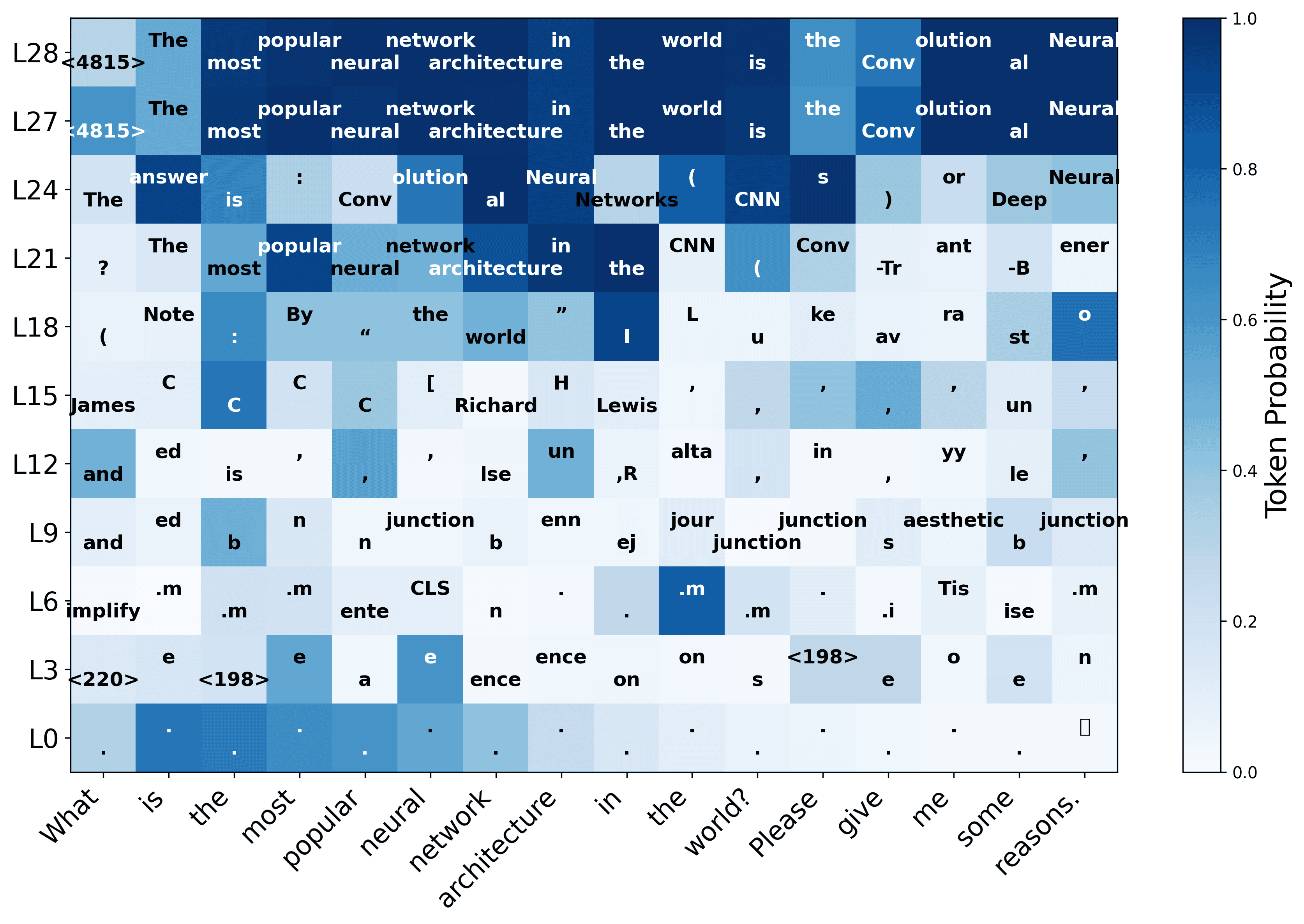}
    \includegraphics[width=0.45\linewidth]{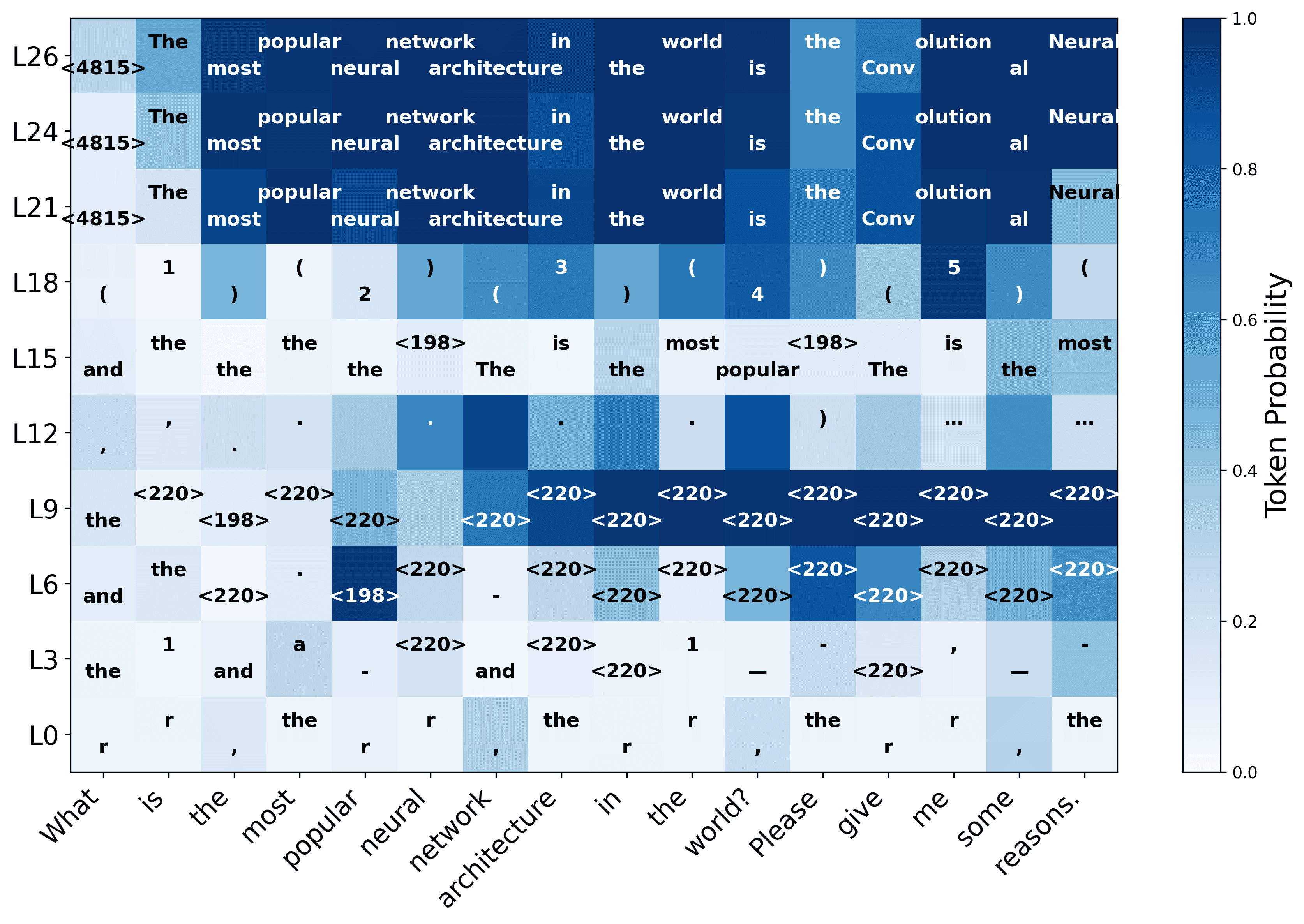}
    \caption{Semantic decoding of Llama3-8B ($m=3$ vs. $m=5$).}
    \label{fig:semantic_decoding_llama_m35}
\end{figure*}

\begin{figure*}[t]
    \centering
    \includegraphics[width=0.45\linewidth]{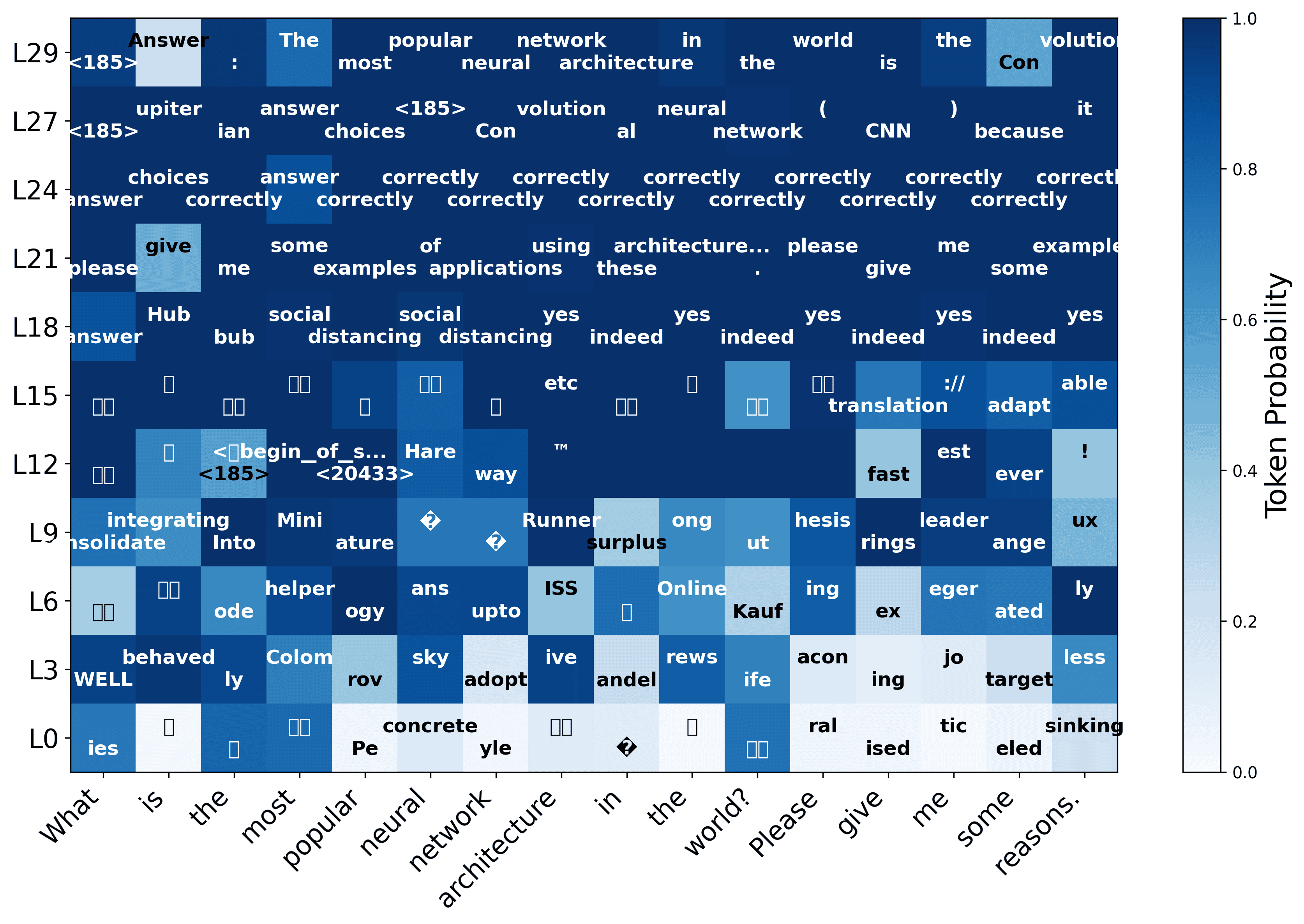}
    \includegraphics[width=0.45\linewidth]{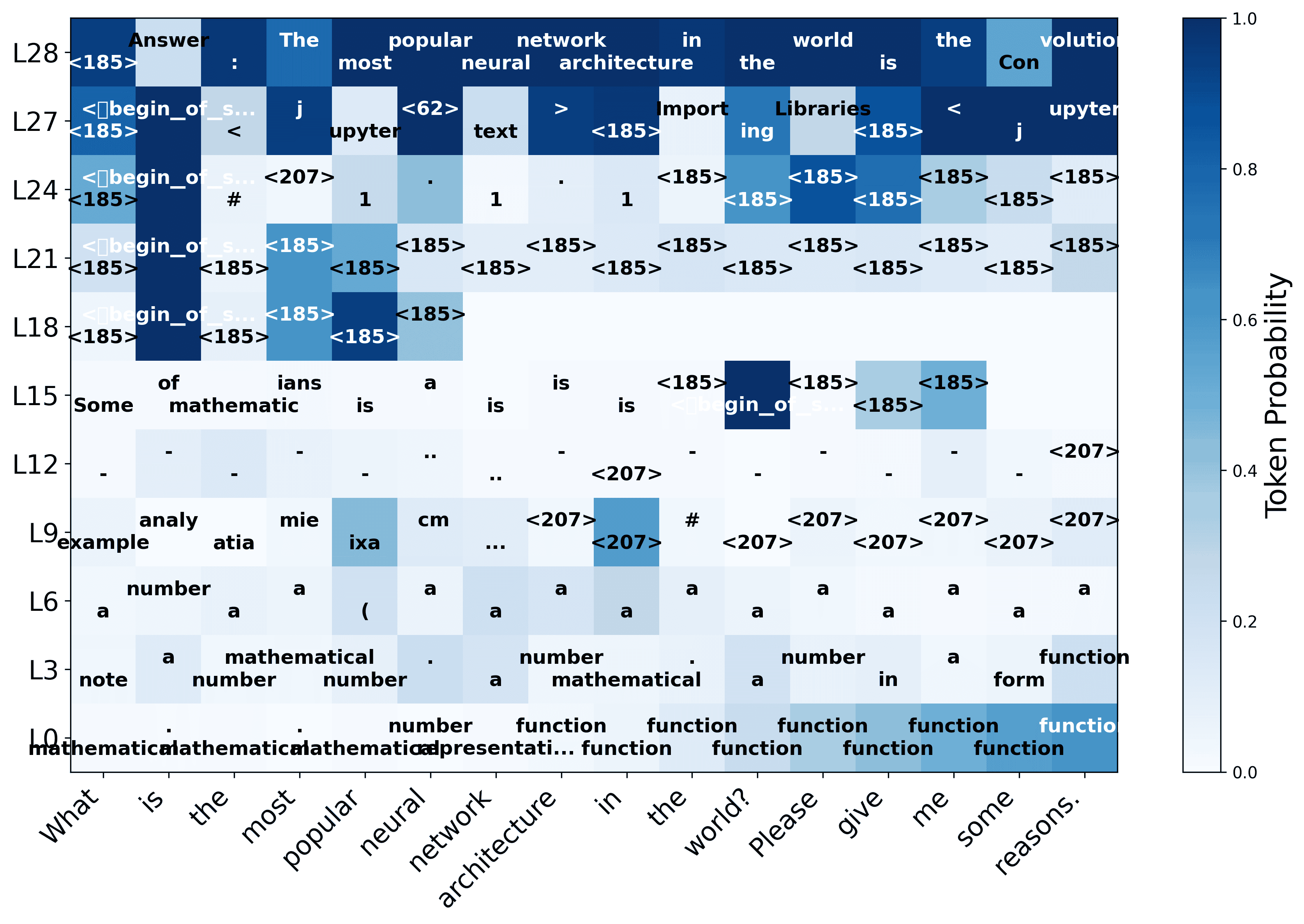}
    \caption{Semantic decoding of Deepseek-7B ($m=0$ vs. $m=1$).  
    }
    \label{fig:semantic_decoding_ds_m01}
\end{figure*}

\begin{figure*}[t]
    \centering
    \includegraphics[width=0.45\linewidth]{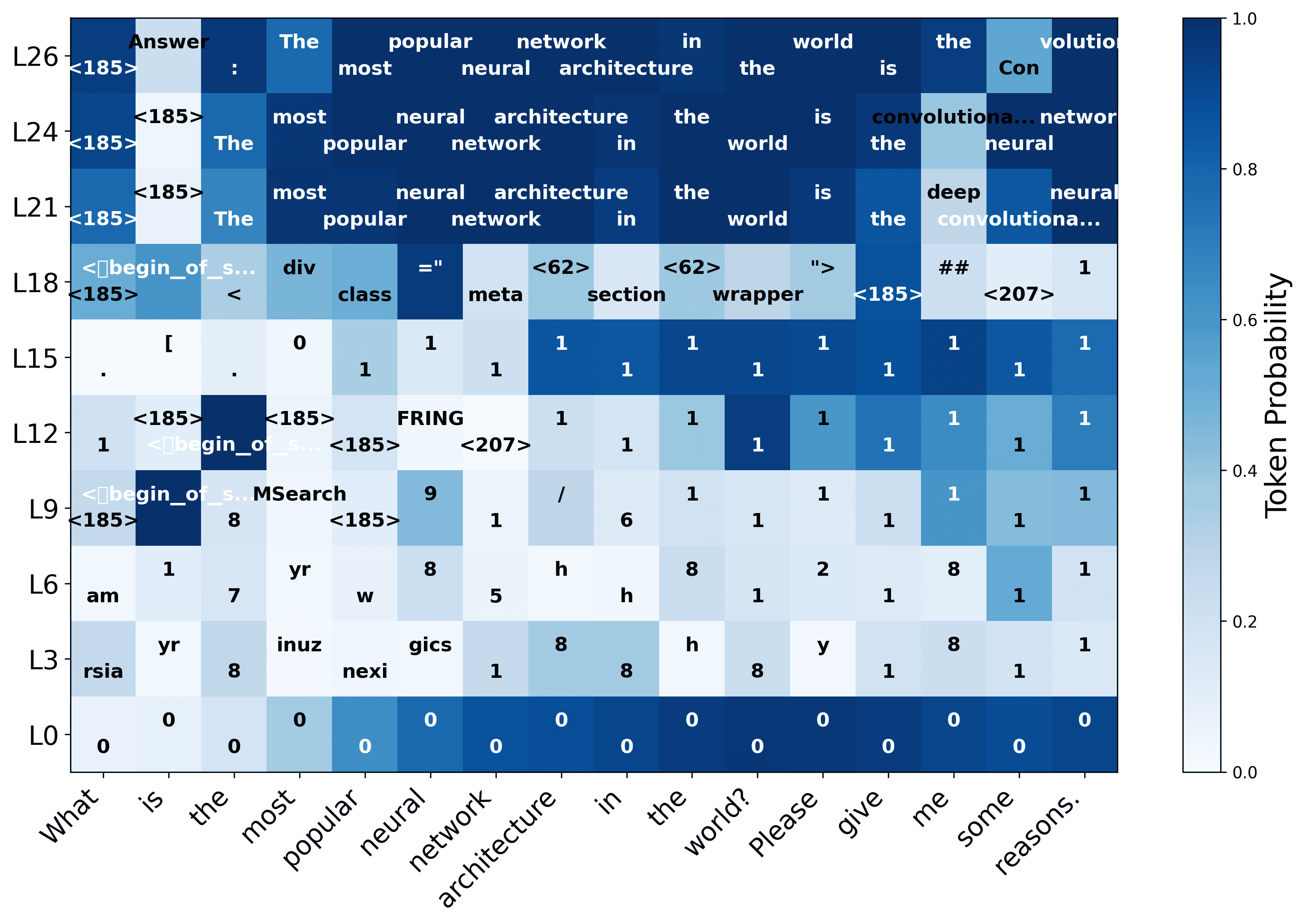}
    \includegraphics[width=0.45\linewidth]{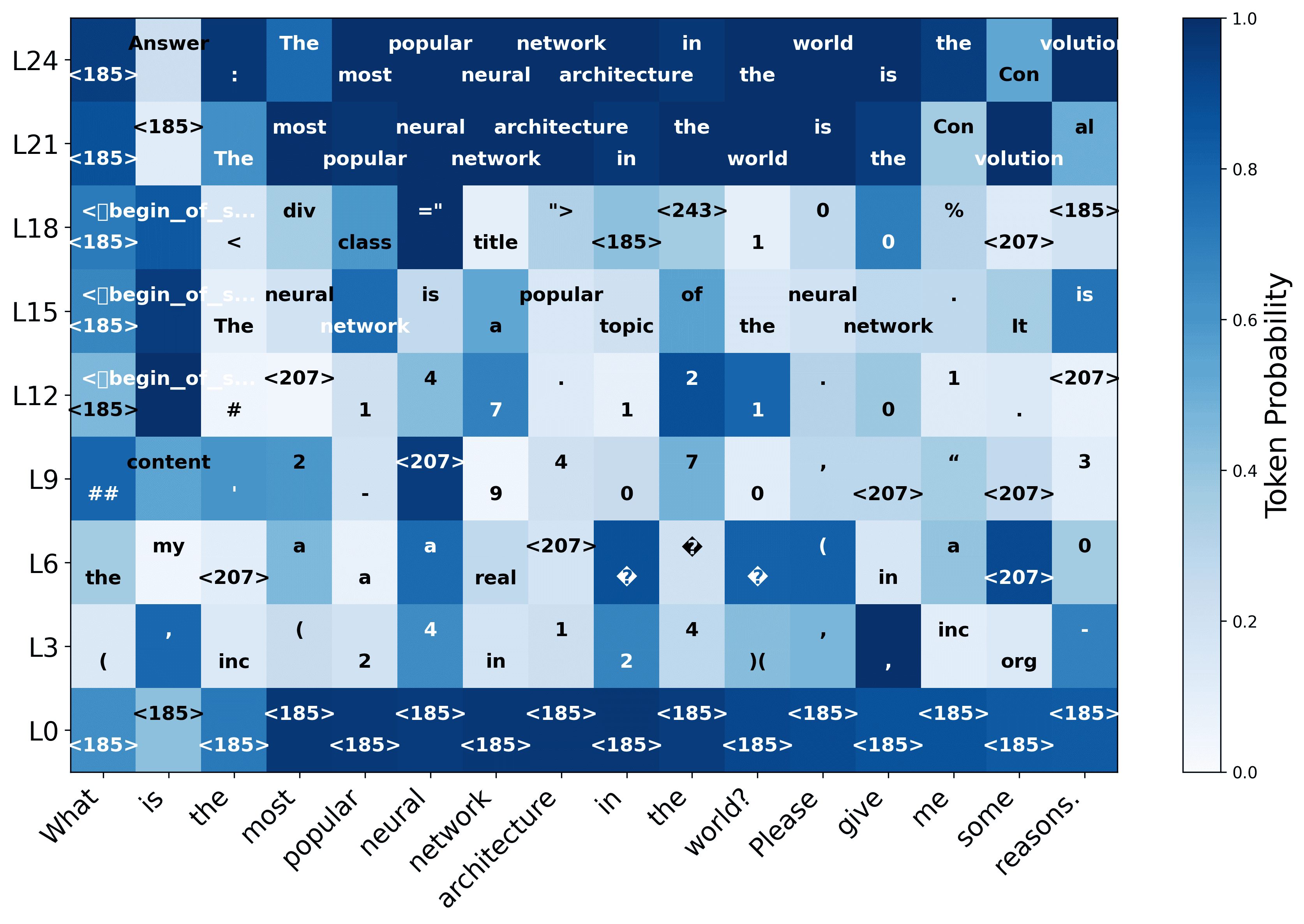}
    \caption{Semantic decoding of Deepseek-7B ($m=3$ vs. $m=5$). 
    }
    \label{fig:semantic_decoding_ds_m35}
\end{figure*}

\begin{figure*}[t]
    \centering
    \includegraphics[width=0.45\linewidth]{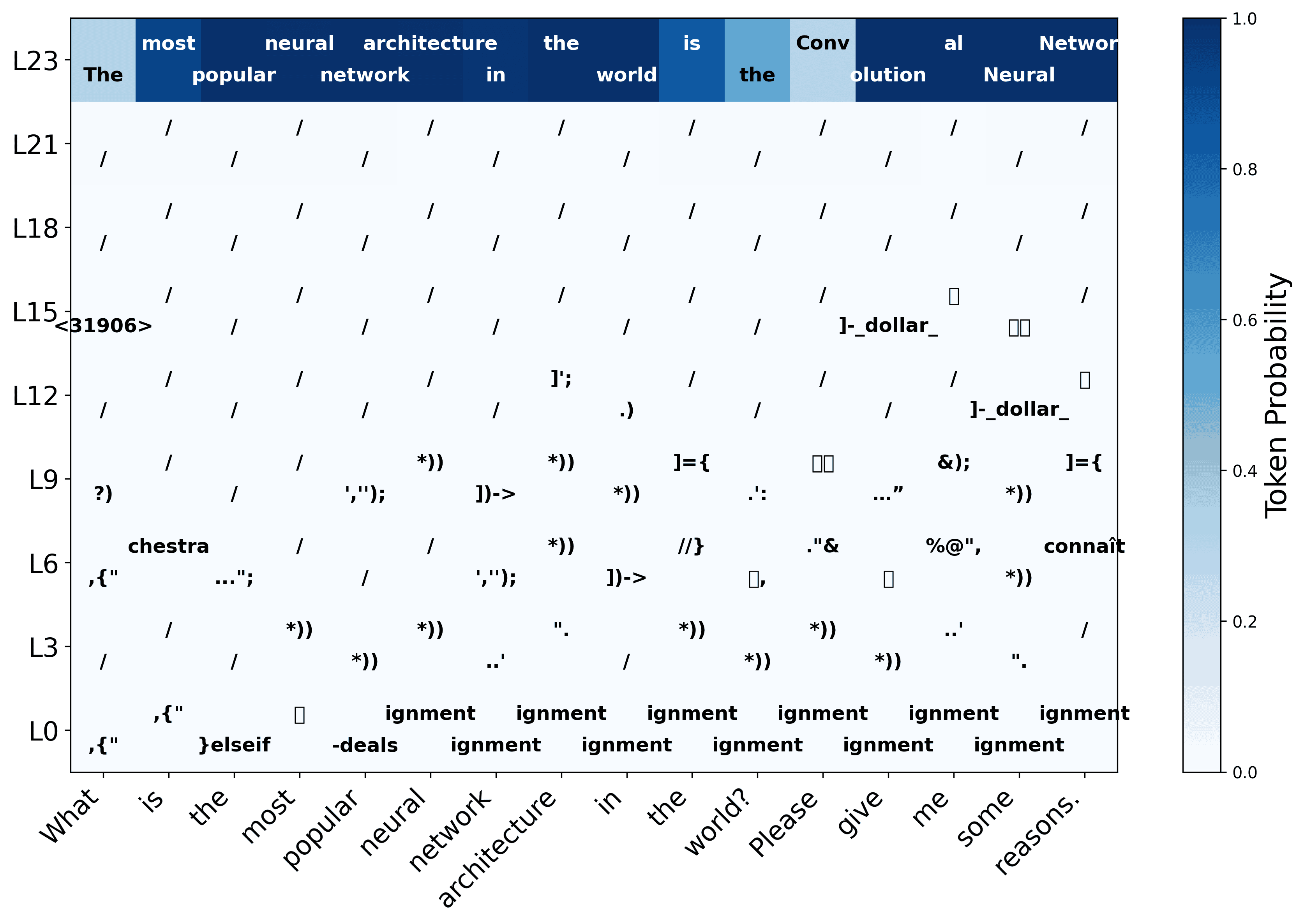}
    \includegraphics[width=0.45\linewidth]{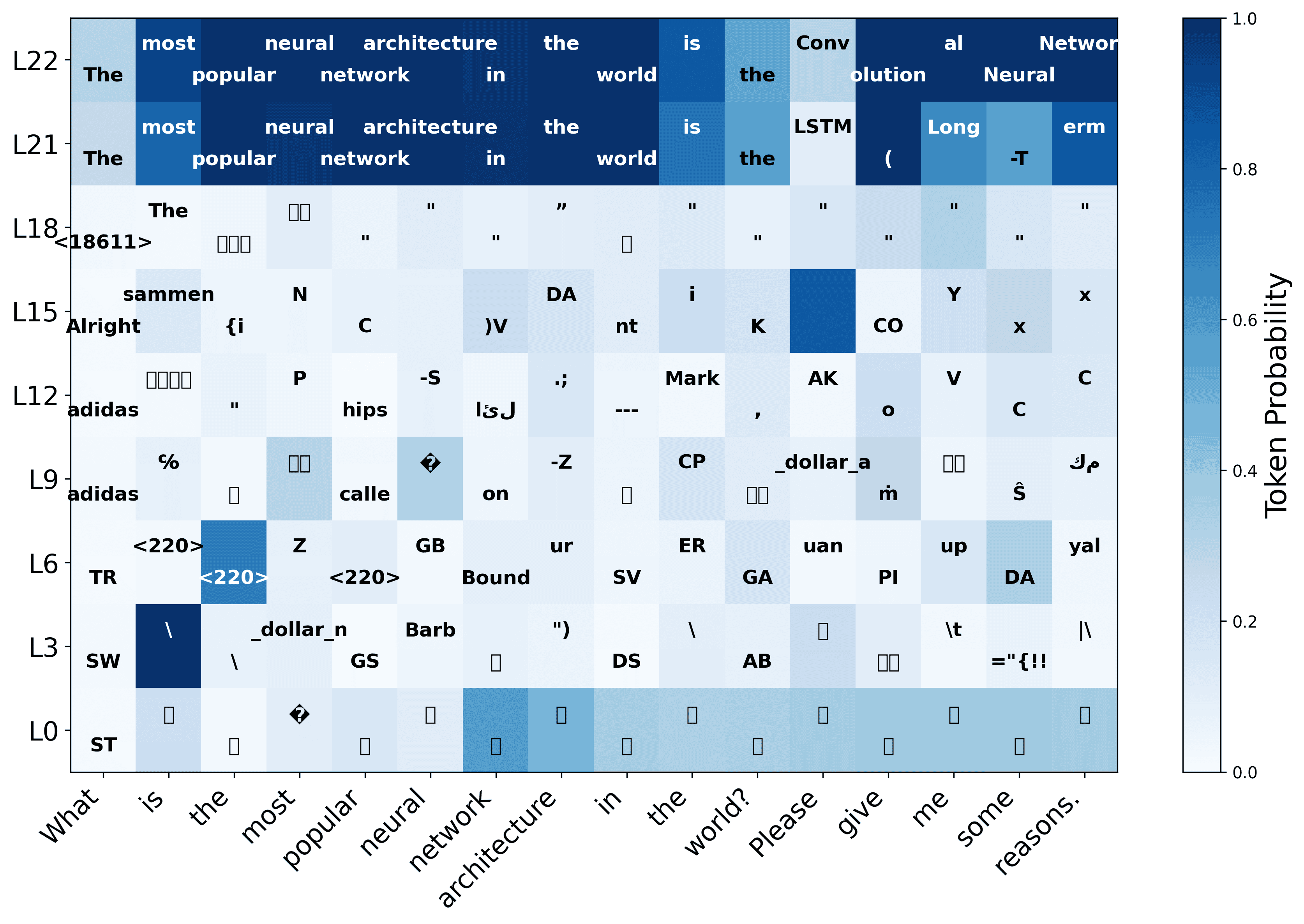}
    \caption{Semantic decoding of Qwen2.5-0.5B ($m=0$ vs. $m=1$). 
    }
    \label{fig:semantic_decoding_Qwen2.5-0.5B_m01}
\end{figure*}

\begin{figure*}[t]
    \centering
    \includegraphics[width=0.45\linewidth]{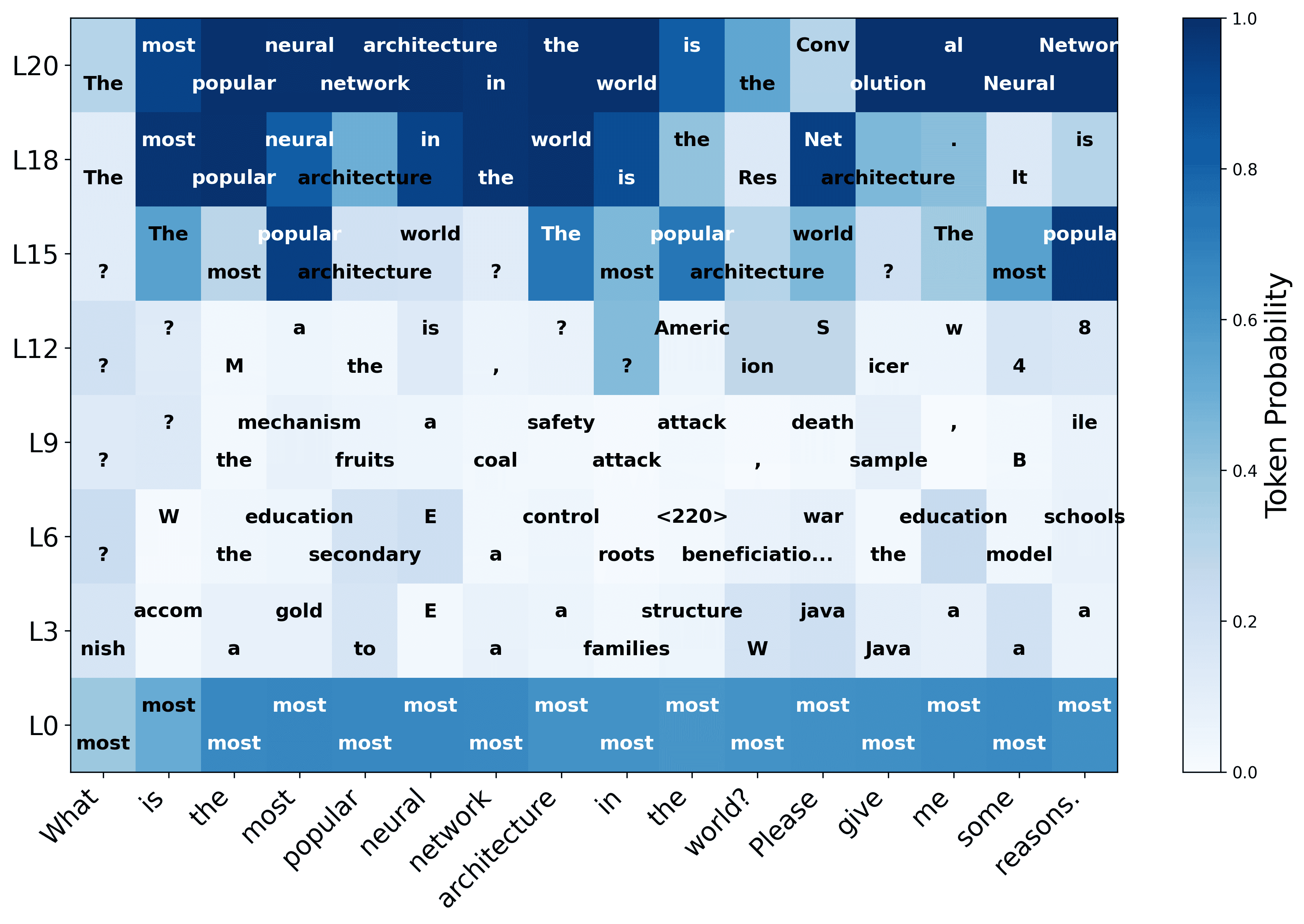}
    \includegraphics[width=0.45\linewidth]{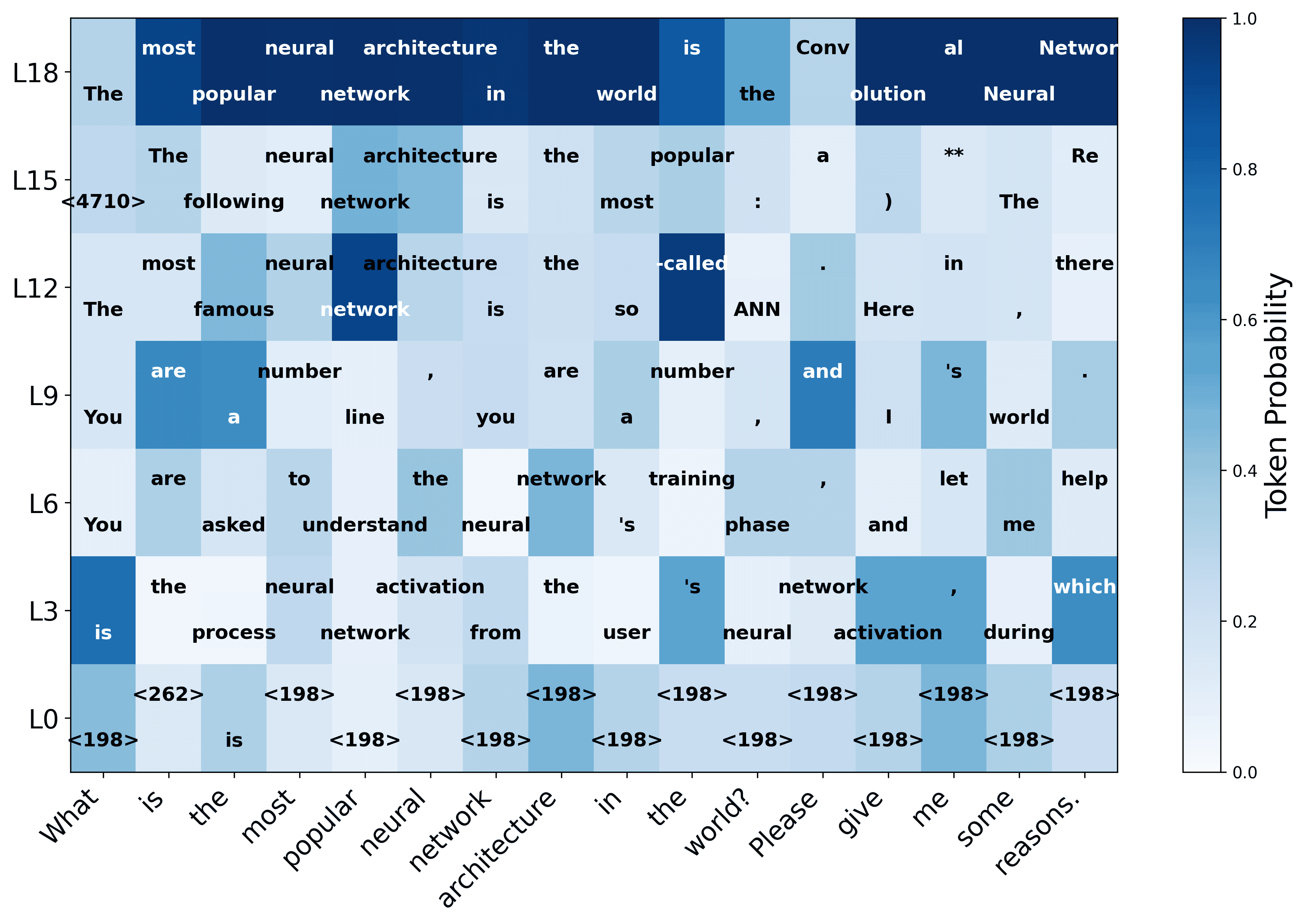}
    \caption{Semantic decoding of Qwen2.5-0.5B ($m=3$ vs. $m=5$). 
    }
    \label{fig:semantic_decoding_Qwen2.5-0.5B_m35}
\end{figure*}

\begin{figure*}[t]
    \centering
    \includegraphics[width=0.45\linewidth]{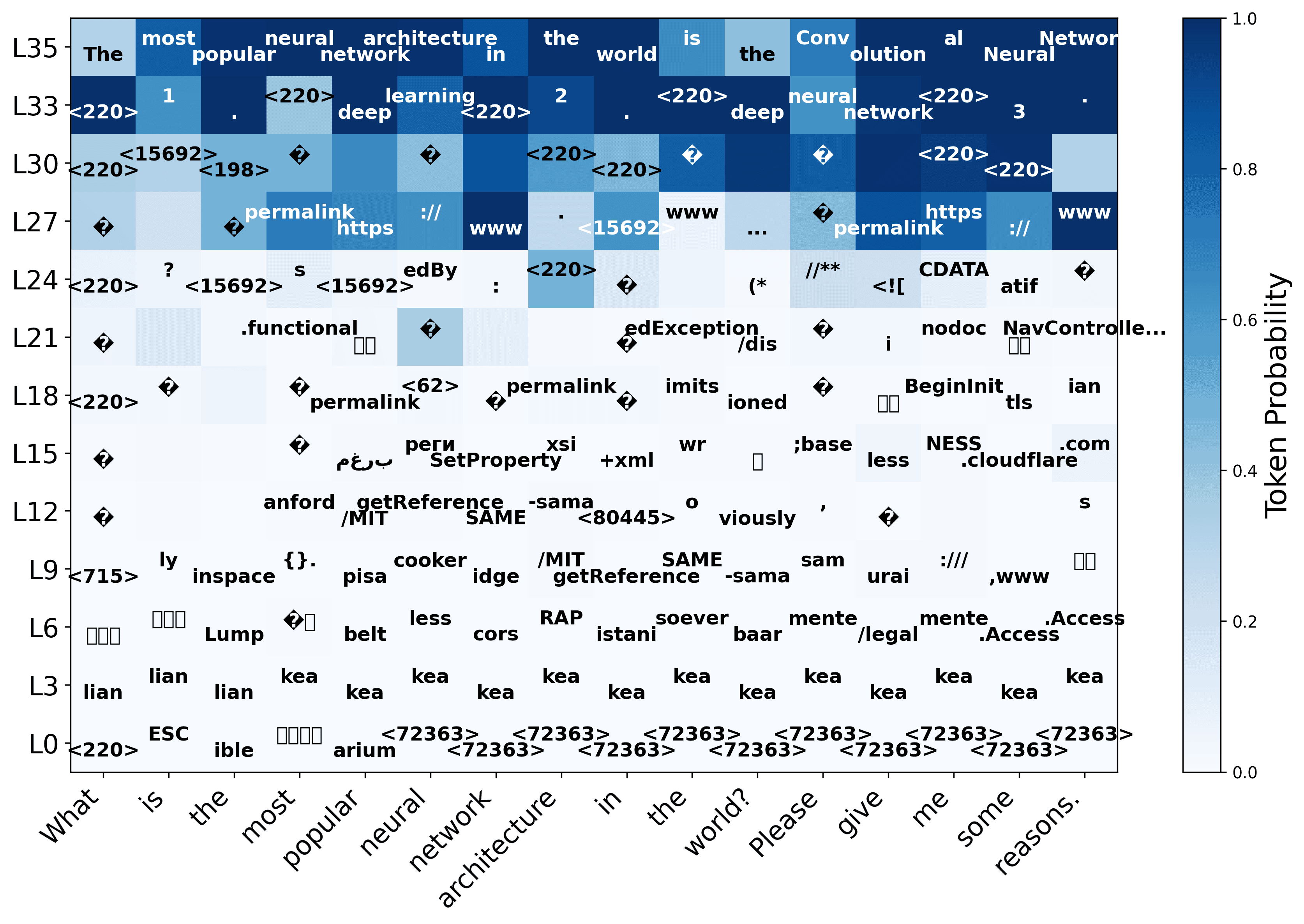}
    \includegraphics[width=0.45\linewidth]{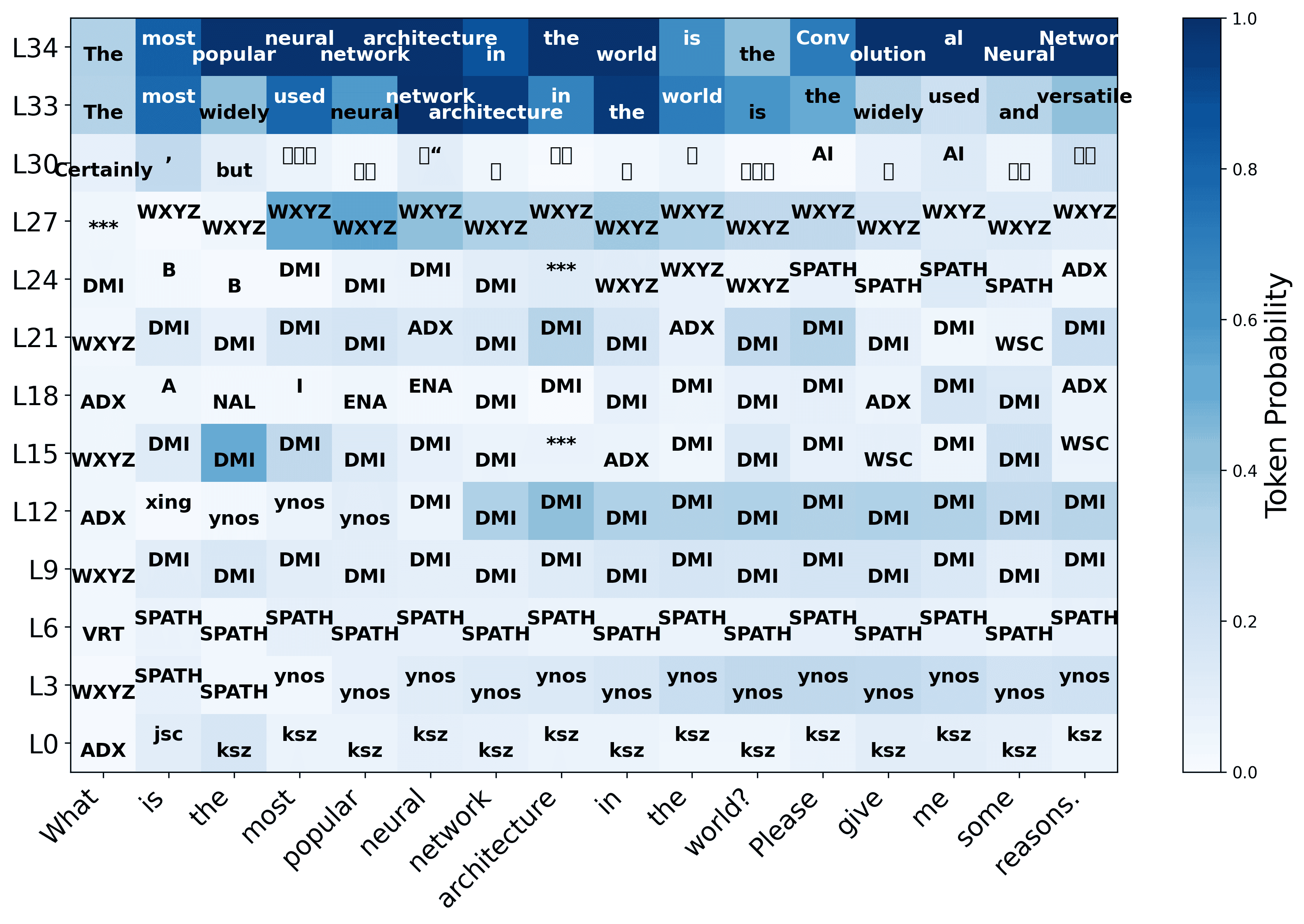}
    \caption{Semantic decoding of Qwen2.5-3B ($m=0$ vs. $m=1$). 
    }
    \label{fig:semantic_decoding_Qwen2.5-3B_m01}
\end{figure*}

\begin{figure*}[t]
    \centering
    \includegraphics[width=0.45\linewidth]{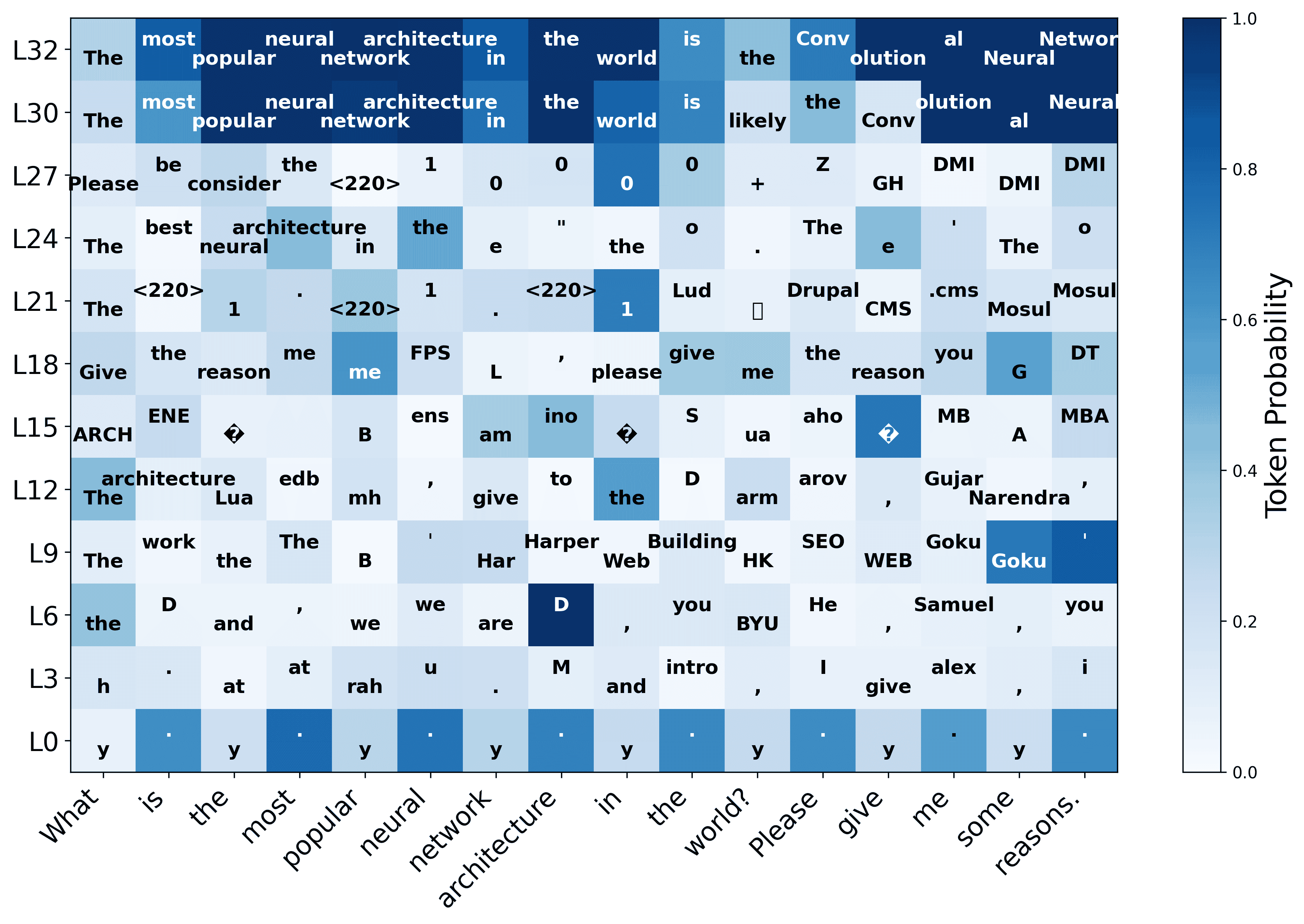}
    \includegraphics[width=0.45\linewidth]{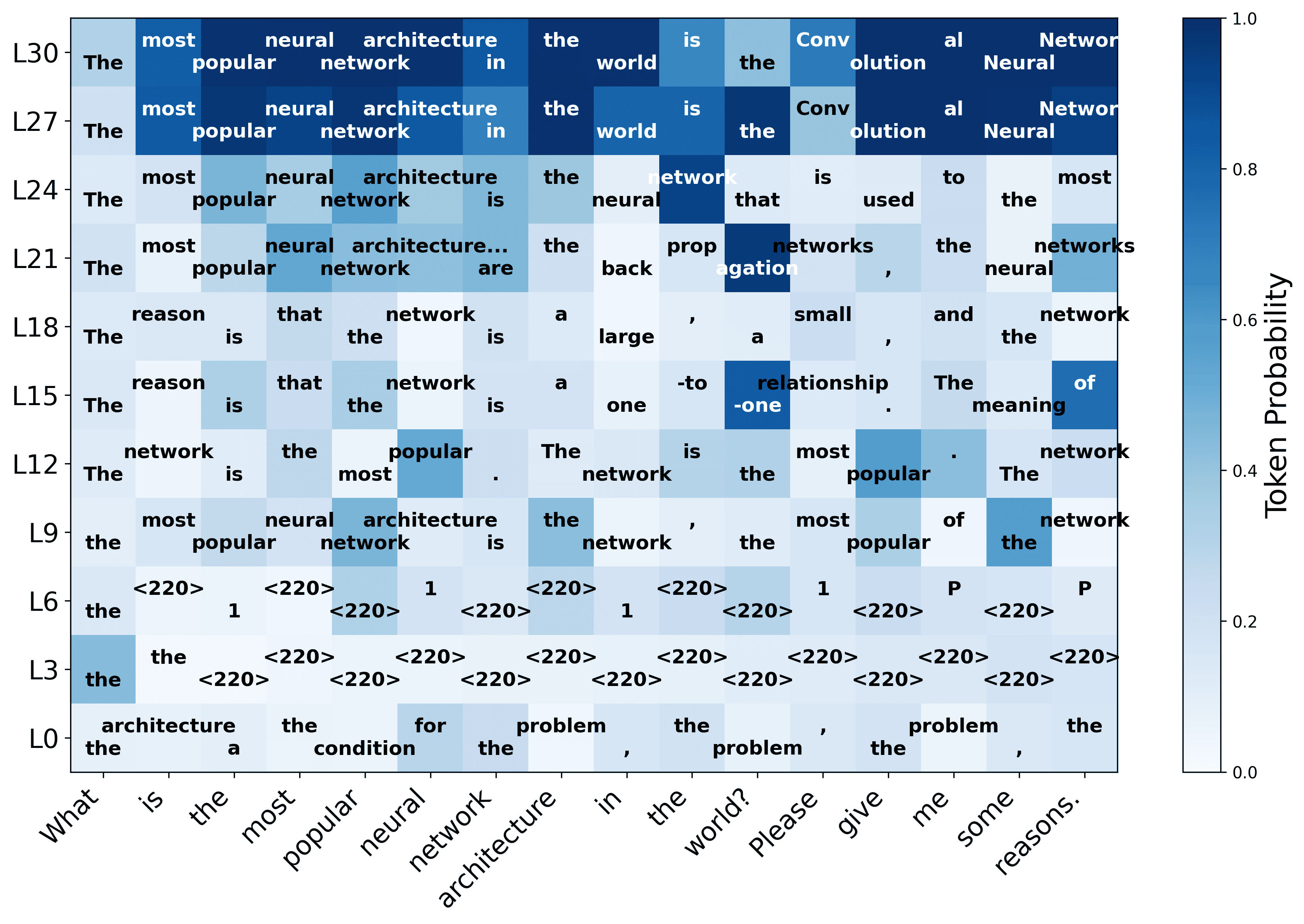}
    \caption{Semantic decoding of Qwen2.5-3B ($m=3$ vs. $m=5$).  
    }
    \label{fig:semantic_decoding_Qwen2.5-3B_m35}
\end{figure*}

\begin{figure*}[t]
    \centering
    \includegraphics[width=0.45\linewidth]{src/figs/shortcut_figs/Qwen2.5-7B-Instruct/m0.png}
    \includegraphics[width=0.45\linewidth]{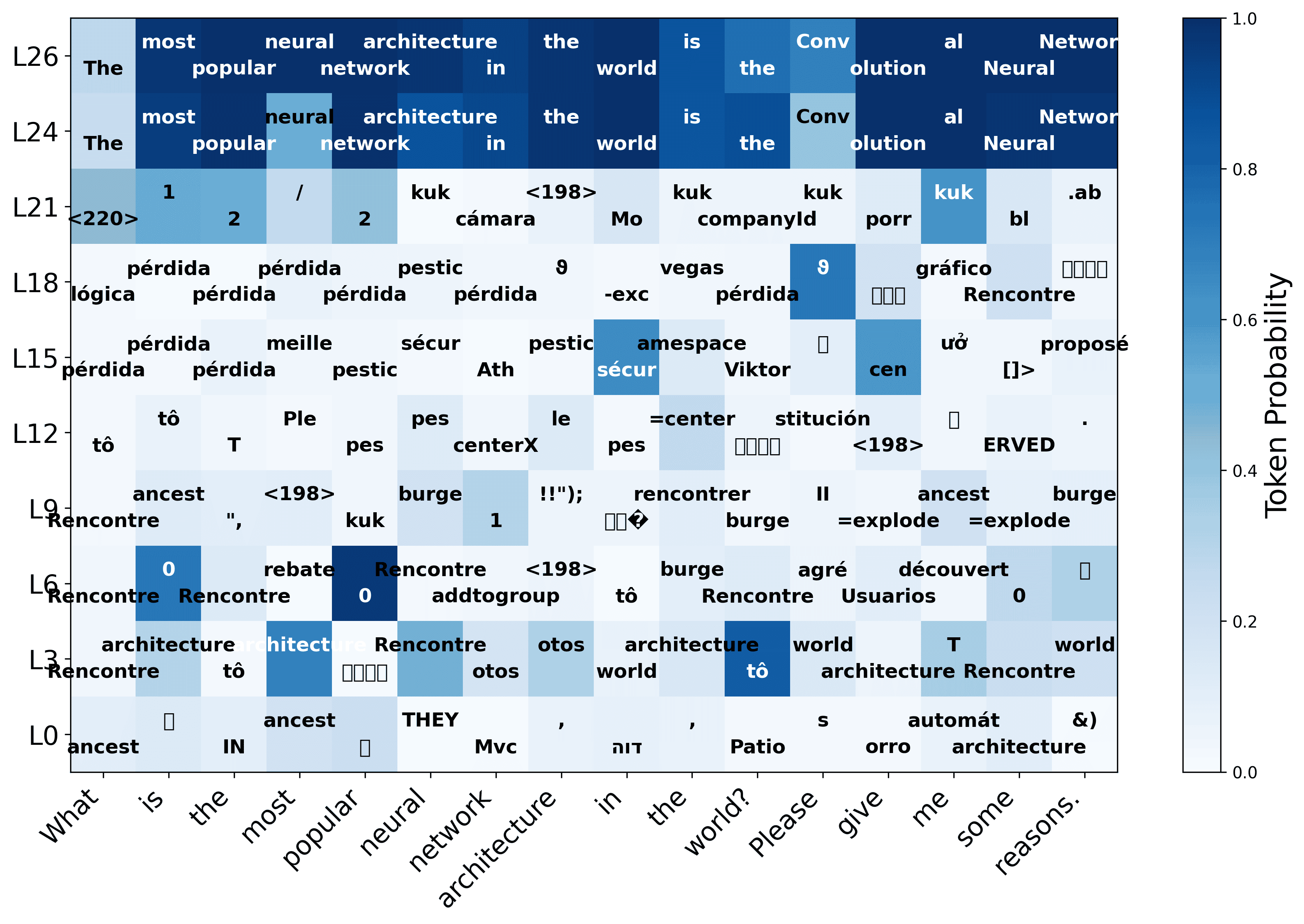}
    \caption{Semantic decoding of Qwen2.5-7B ($m=0$ vs. $m=1$). 
    }
    \label{fig:semantic_decoding_Qwen2.5-7B_m01}
\end{figure*}

\begin{figure*}[t]
    \centering
    \includegraphics[width=0.45\linewidth]{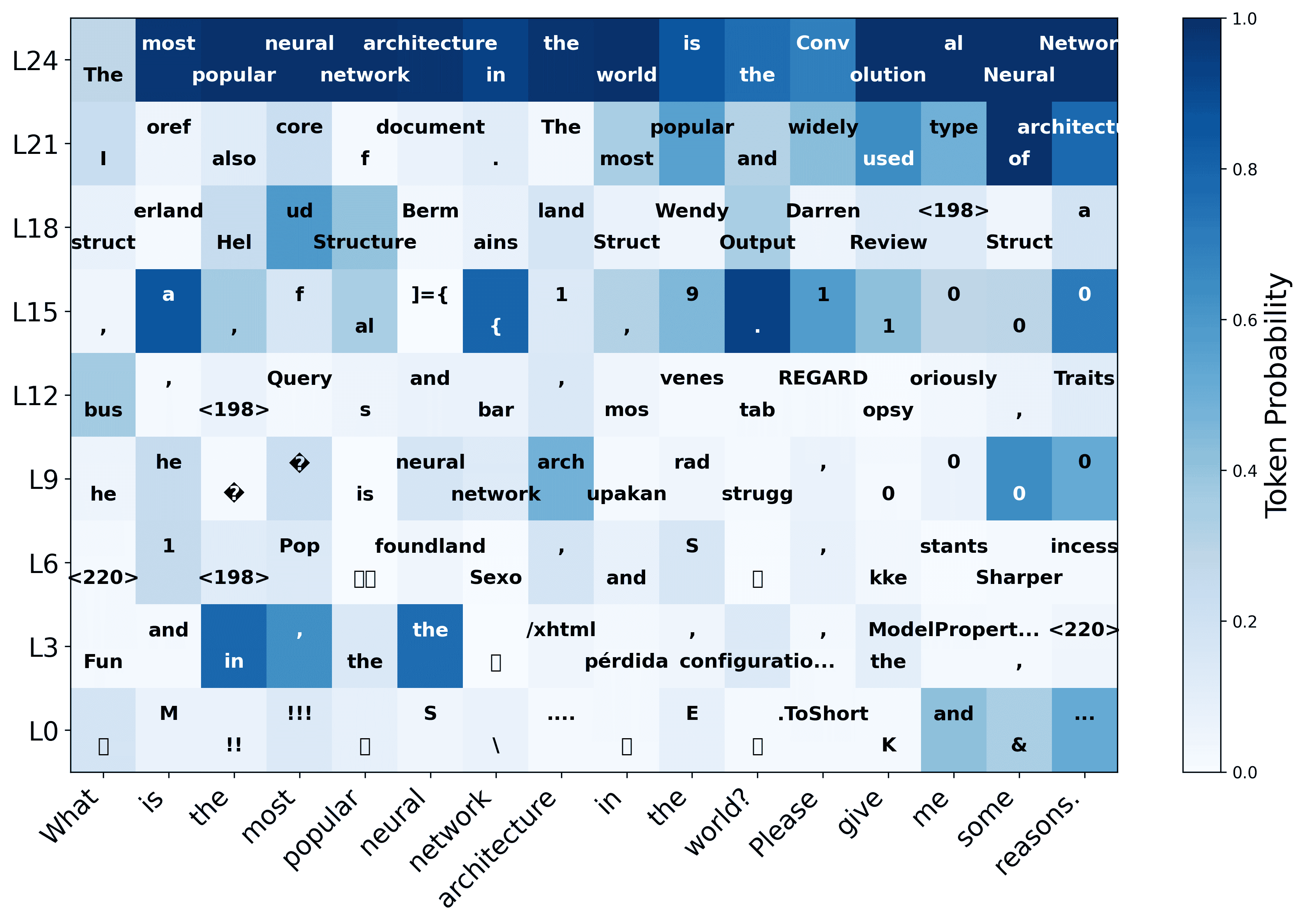}
    \includegraphics[width=0.45\linewidth]{src/figs/shortcut_figs/Qwen2.5-7B-Instruct/m5.png}
    \caption{Semantic decoding of Qwen2.5-7B ($m=3$ vs. $m=5$). 
    }
    \label{fig:semantic_decoding_Qwen2.5-7B_m35}
\end{figure*}

\begin{figure*}[t]
    \centering
    \includegraphics[width=0.45\linewidth]{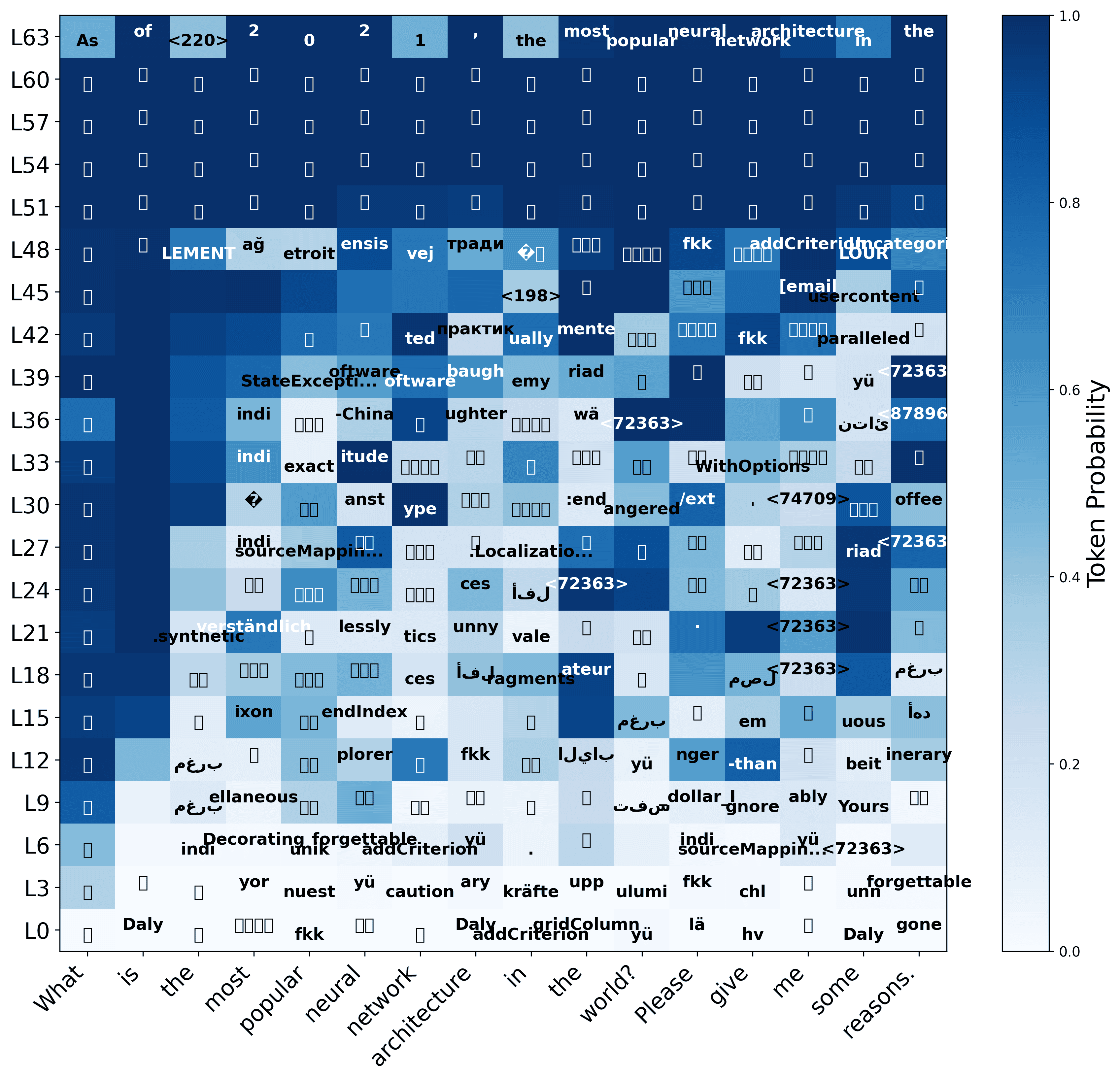}
    \includegraphics[width=0.45\linewidth]{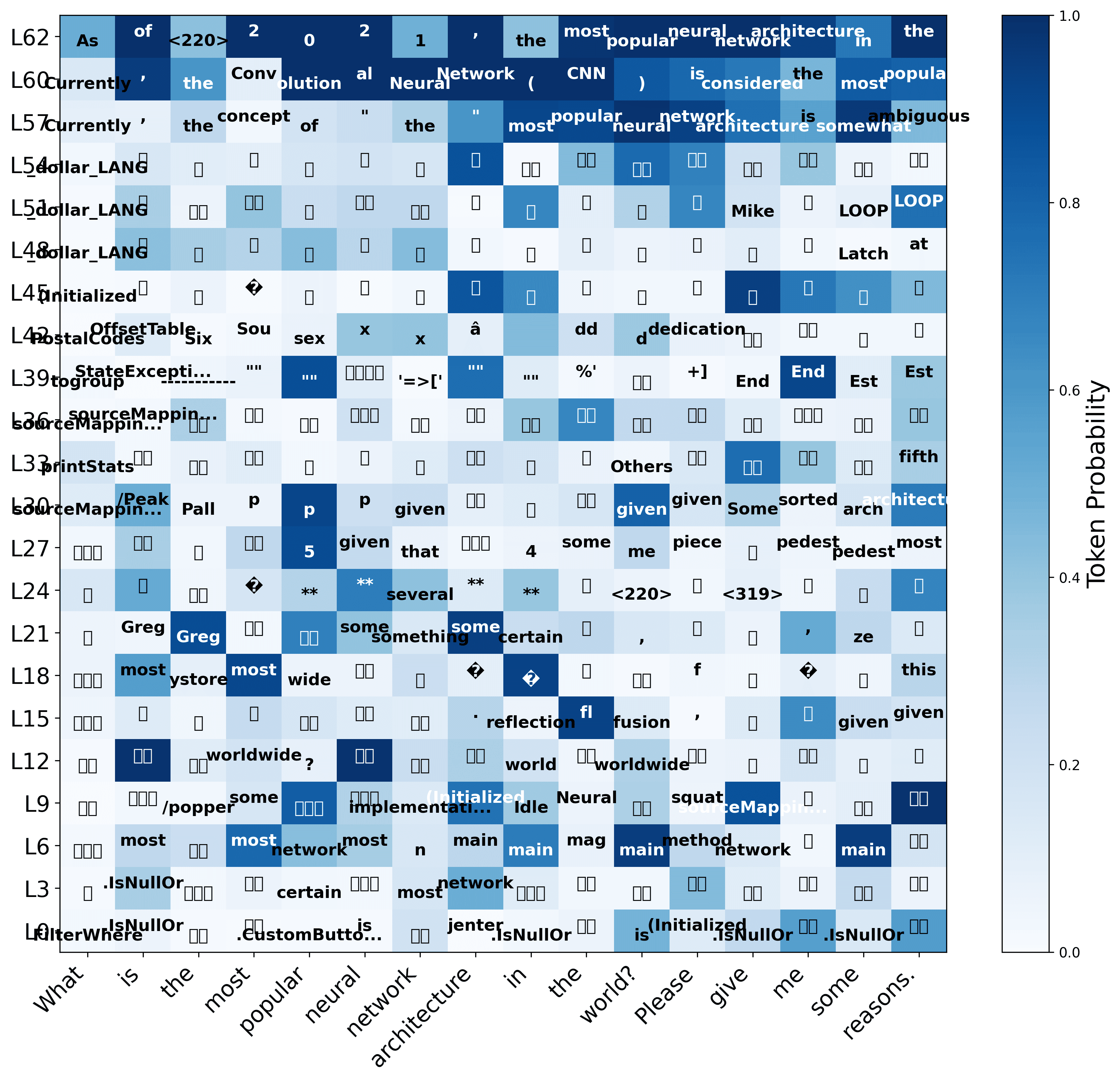}
    \caption{Semantic decoding of Qwen2.5-32B ($m=0$ vs. $m=1$). 
    }
    \label{fig:semantic_decoding_Qwen2.5-32B_m01}
\end{figure*}

\begin{figure*}[t]
    \centering
    \includegraphics[width=0.45\linewidth]{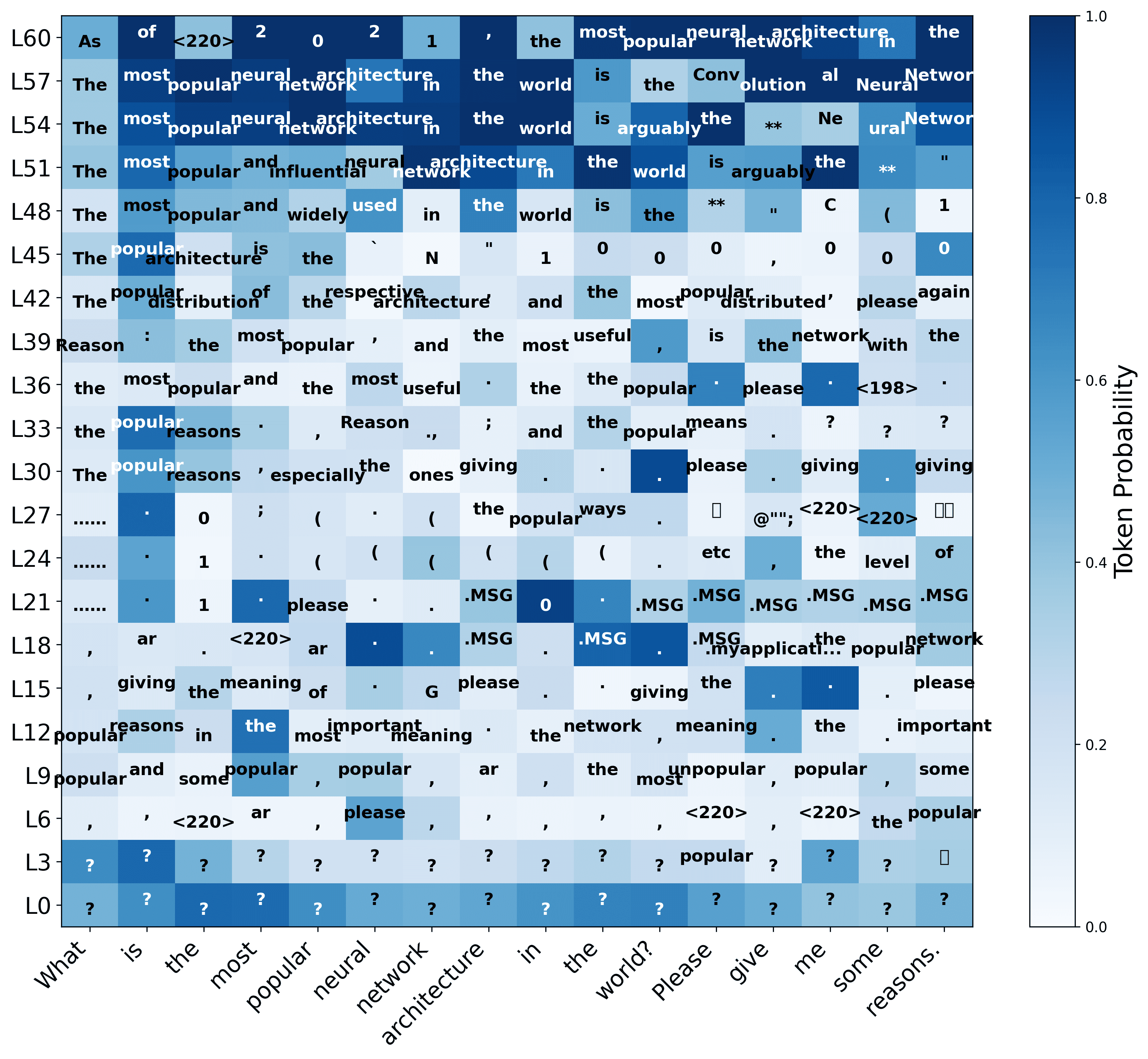}
    \includegraphics[width=0.45\linewidth]{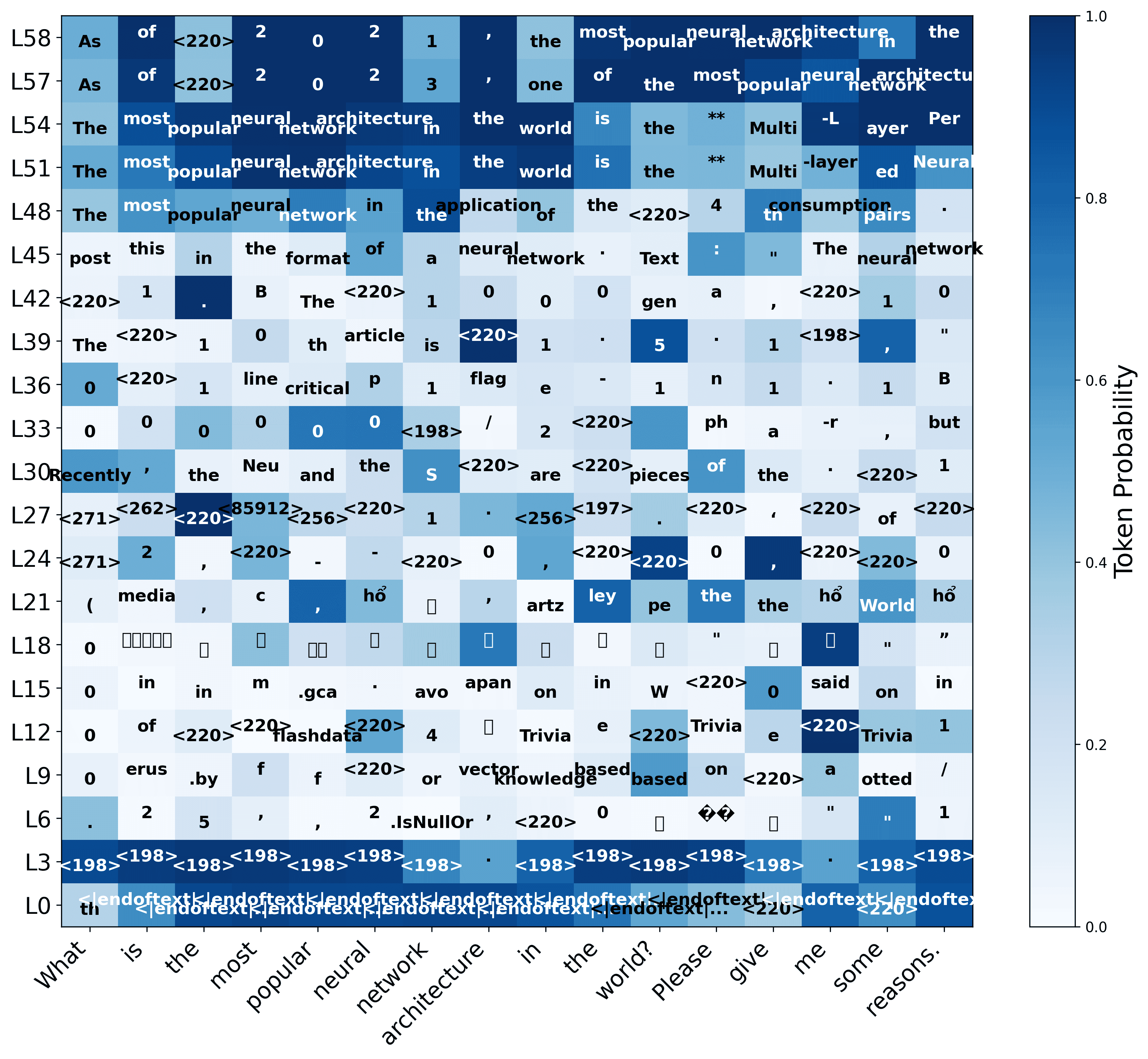}
    \caption{Semantic decoding of Qwen2.5-32B ($m=3$ vs. $m=5$). 
    }
    \label{fig:semantic_decoding_Qwen2.5-32B_m35}
\end{figure*}

\begin{figure*}[t]
    \centering
    \includegraphics[width=0.45\linewidth]{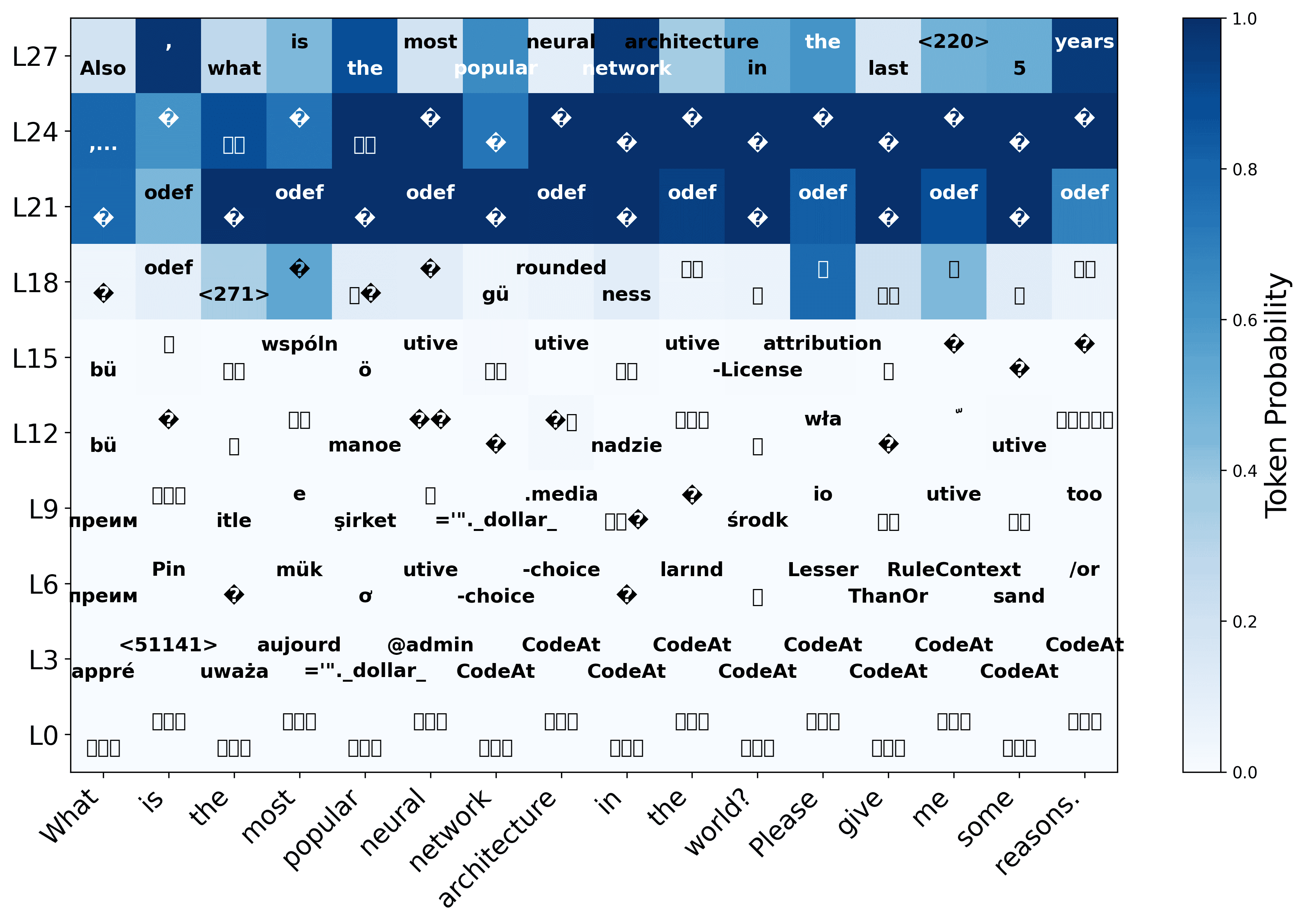}
    \includegraphics[width=0.45\linewidth]{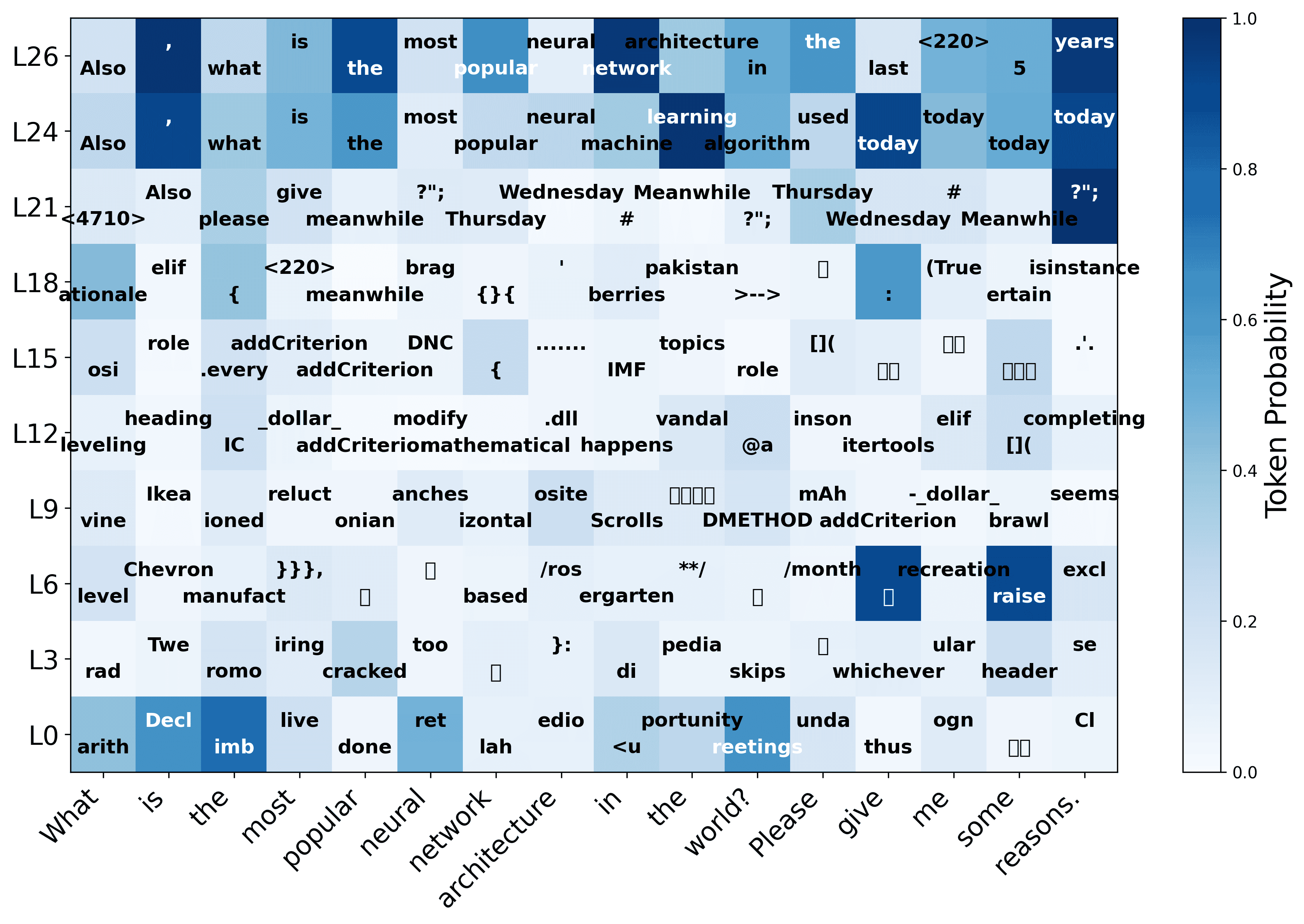}
    \caption{Semantic decoding of Qwen3-0.6B ($m=0$ vs. $m=1$).  
    }
    \label{fig:semantic_decoding_Qwen3-0.6B_m01}
\end{figure*}

\begin{figure*}[t]
    \centering
    \includegraphics[width=0.45\linewidth]{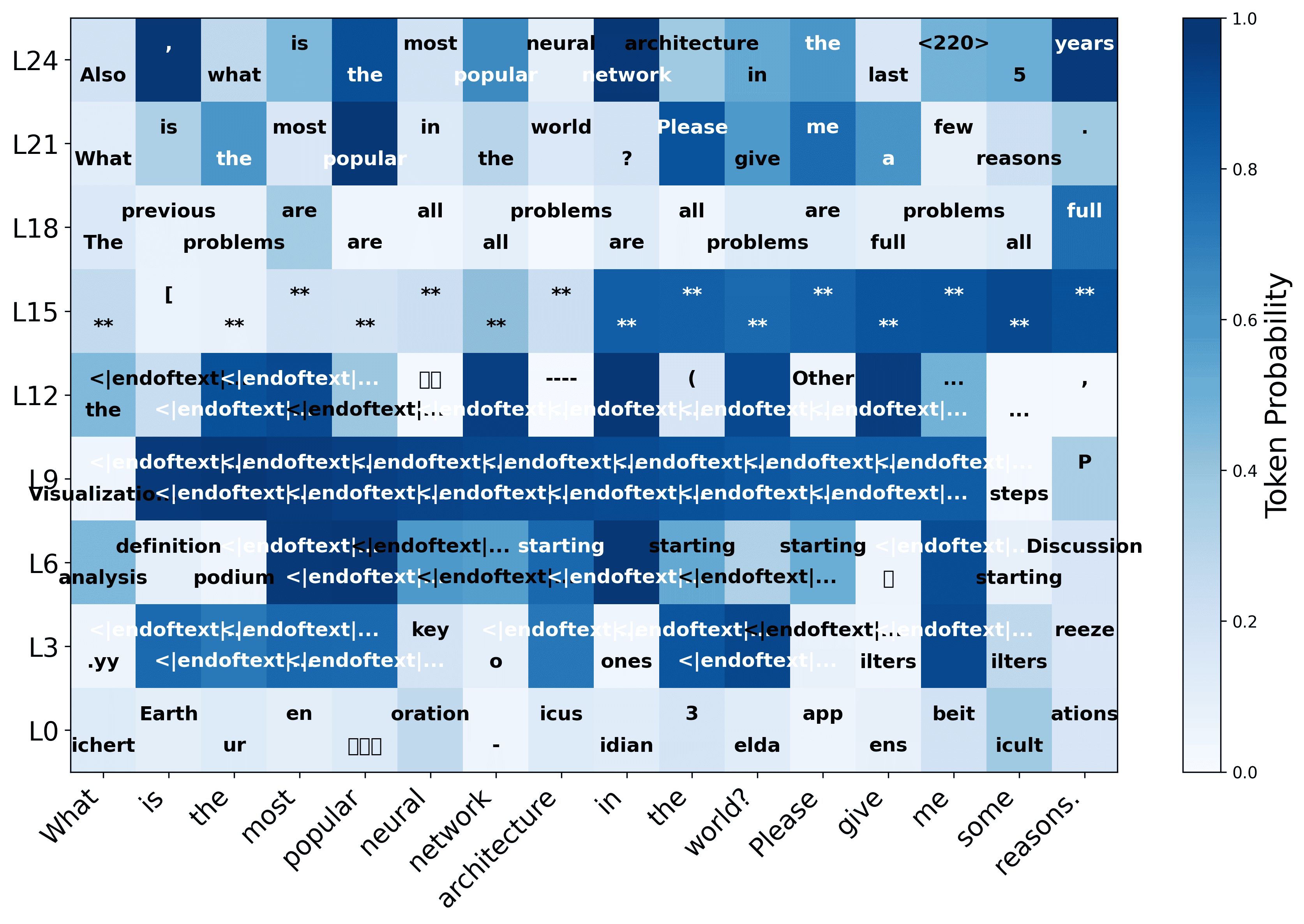}
    \includegraphics[width=0.45\linewidth]{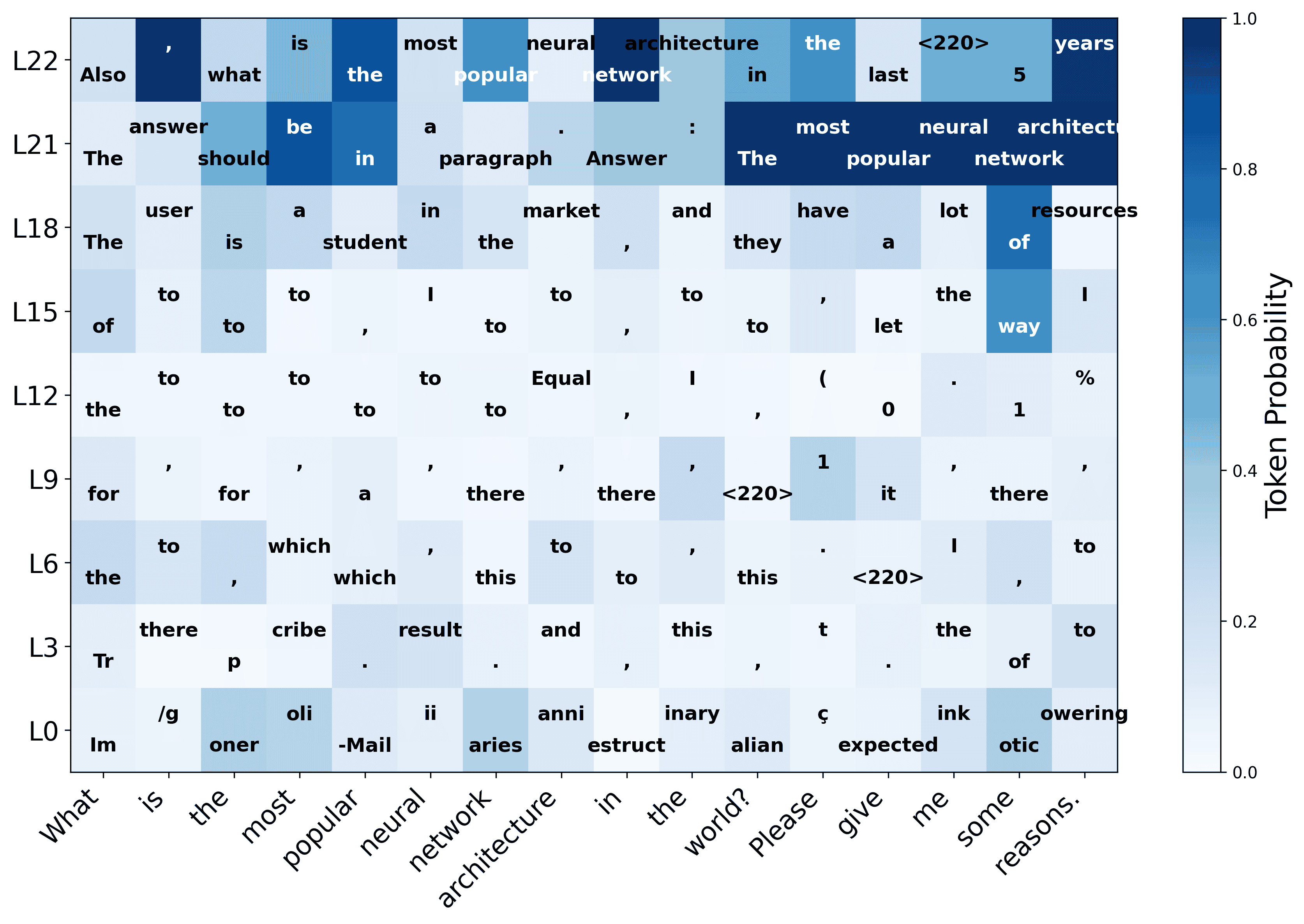}
    \caption{Semantic decoding of Qwen3-0.6B ($m=3$ vs. $m=5$). 
    }
    \label{fig:semantic_decoding_Qwen3-0.6B_m35}
\end{figure*}

\begin{figure*}[t]
    \centering
    \includegraphics[width=0.45\linewidth]{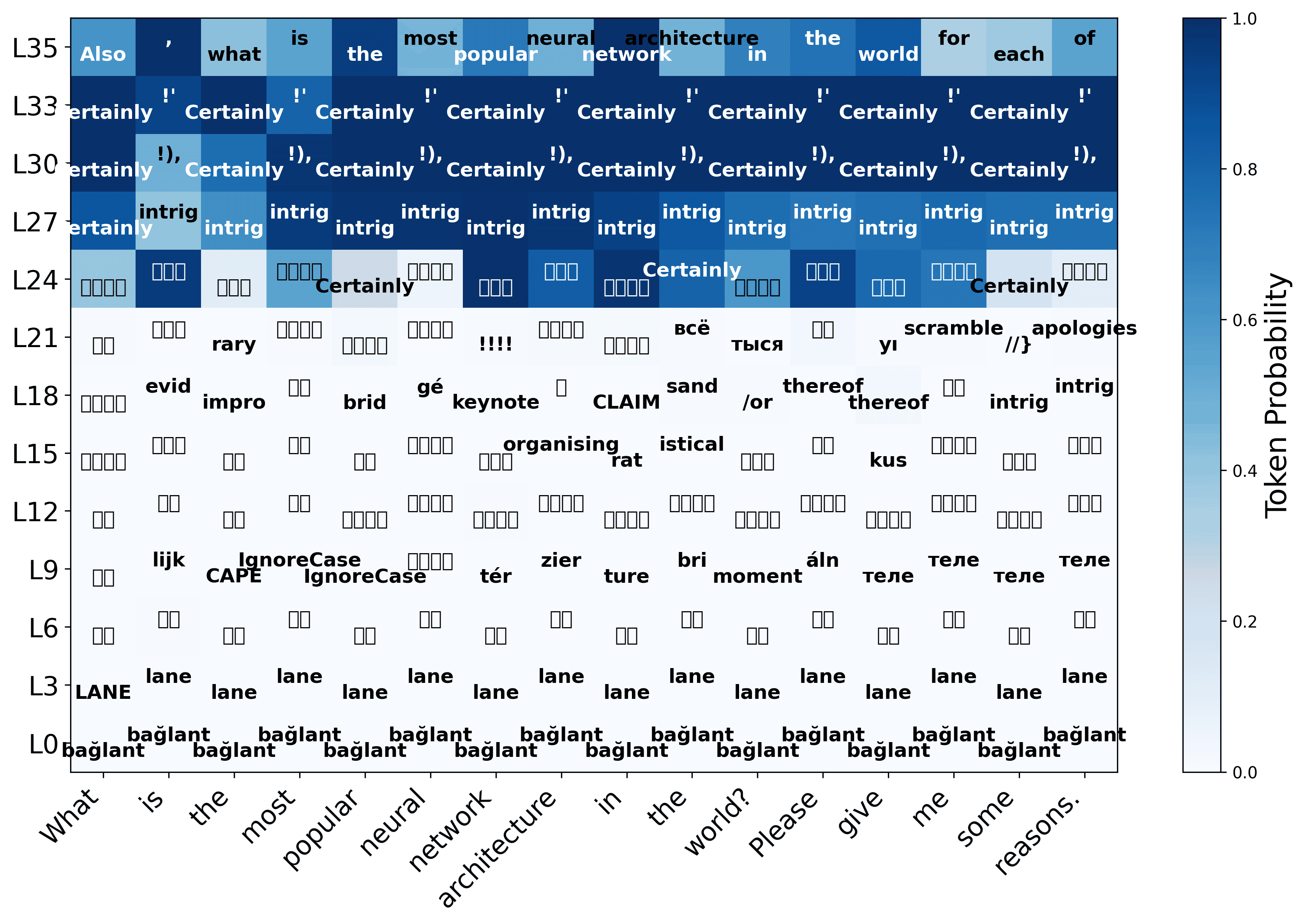}
    \includegraphics[width=0.45\linewidth]{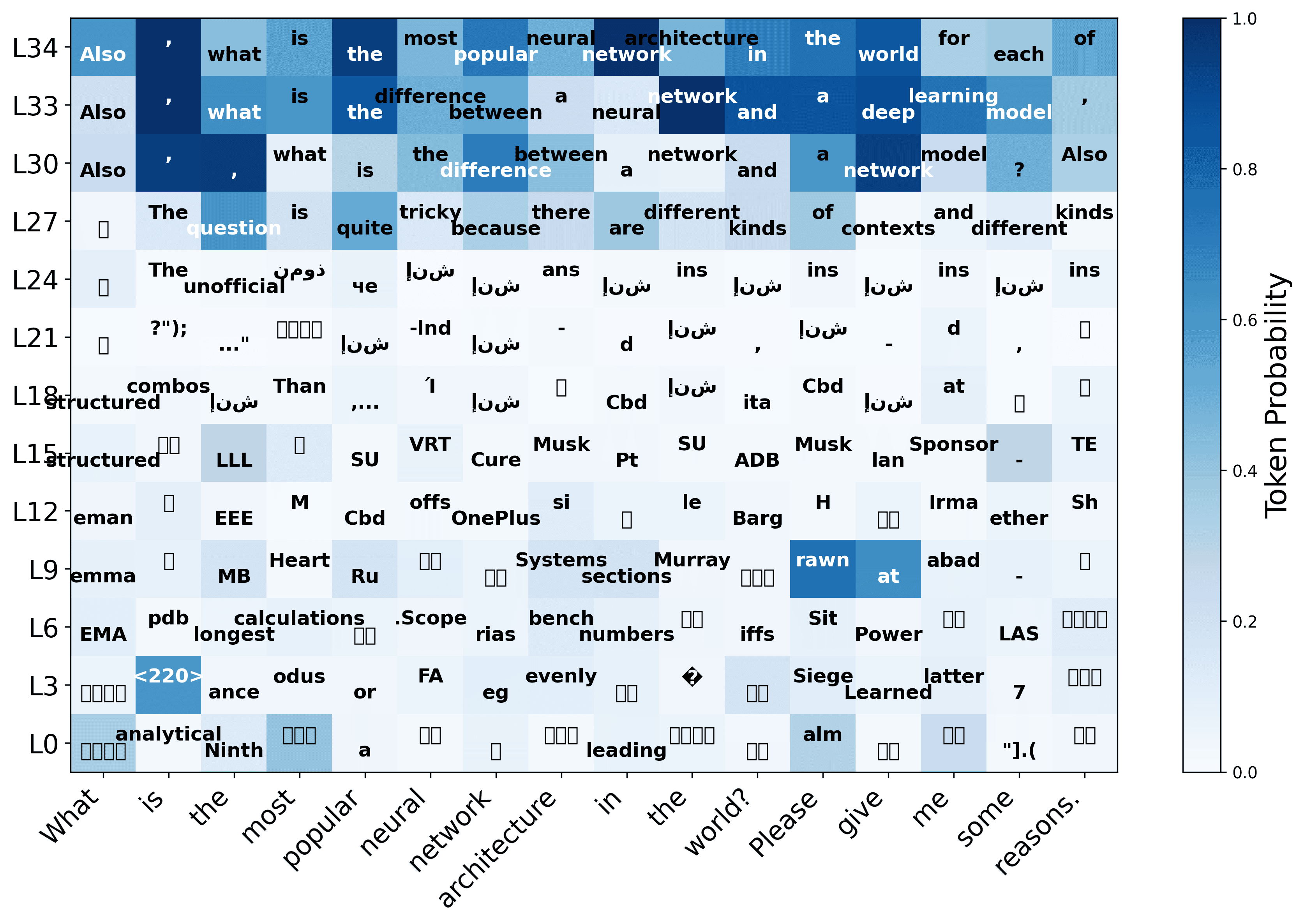}
    \caption{Semantic decoding of Qwen3-4B ($m=0$ vs. $m=1$).  
    }
    \label{fig:semantic_decoding_Qwen3-4B_m01}
\end{figure*}

\begin{figure*}[t]
    \centering
    \includegraphics[width=0.45\linewidth]{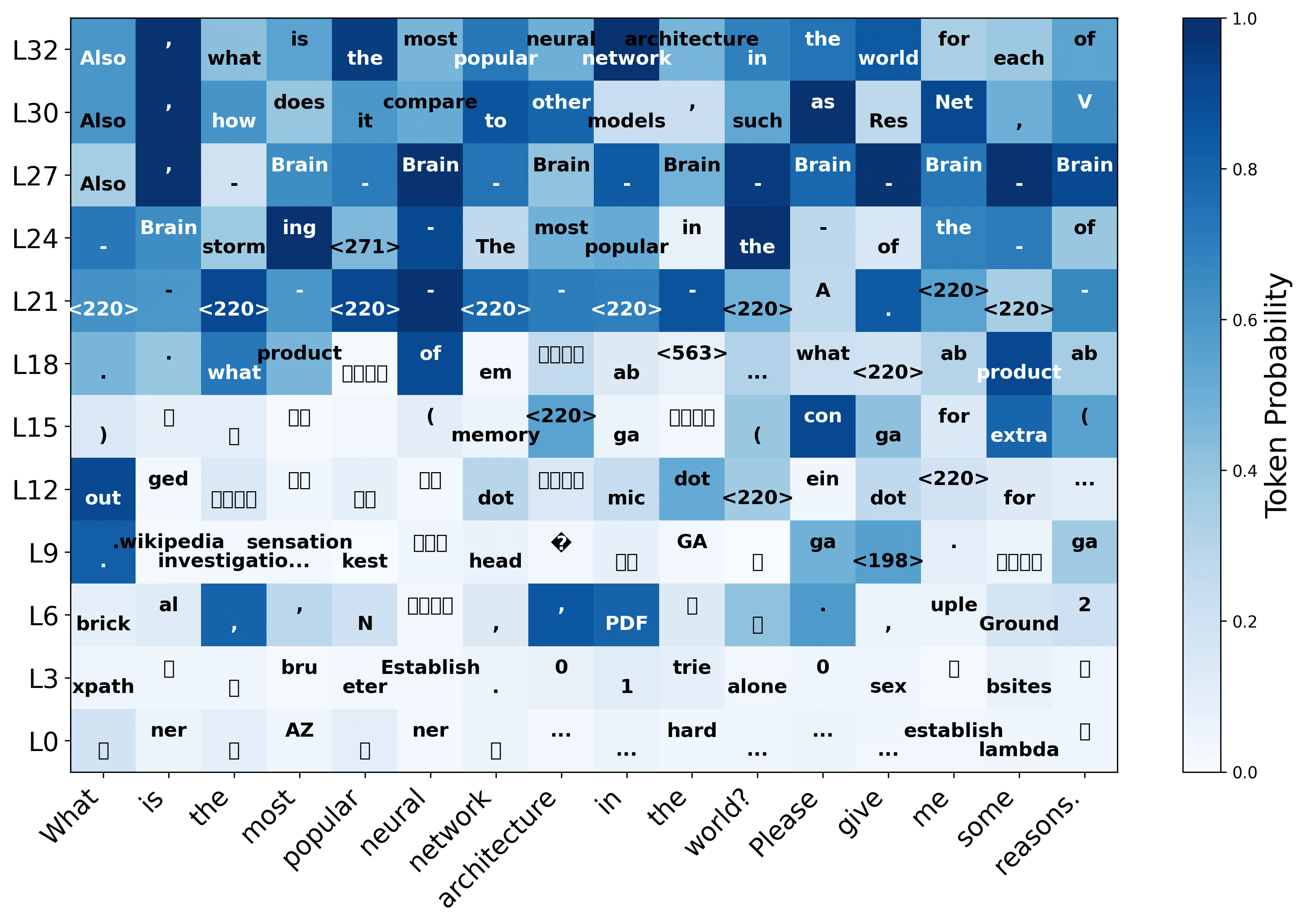}
    \includegraphics[width=0.45\linewidth]{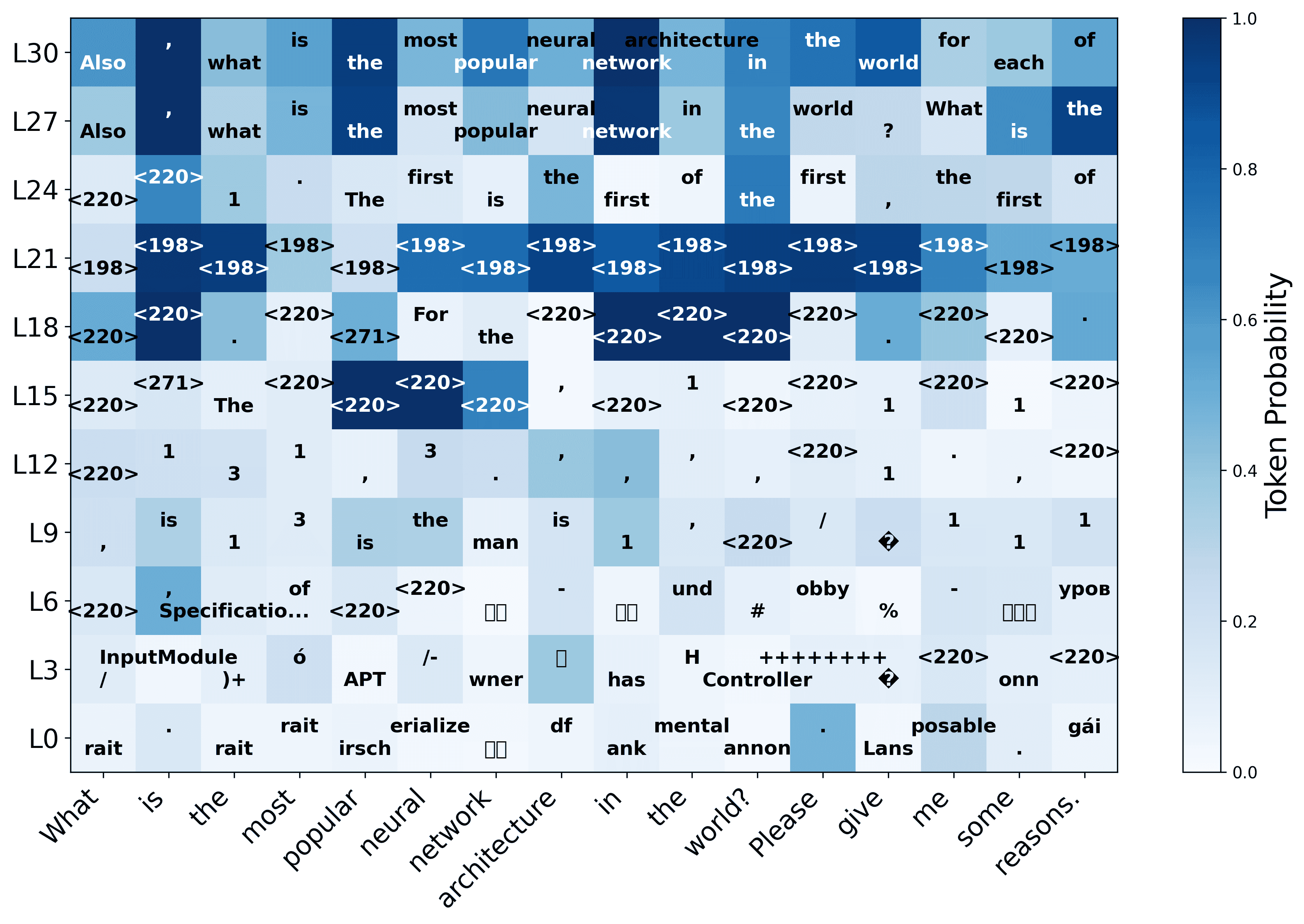}
    \caption{Semantic decoding of Qwen3-4B ($m=3$ vs. $m=5$).  
    }
    \label{fig:semantic_decoding_Qwen3-4B_m35}
\end{figure*}

\clearpage

\section{Examples of outputs after training}
\label{appen:eg_sft}
The following figures illustrates the output behavior of Qwen2.5-7B on the Math domain before and after SFT.
\textbf{Observation.} The base model demonstrates some thinking process, explicitly marking steps (e.g., ``\#\#\# Step 1'') and including natural language explanations. In contrast, the SFT model exhibits a drastic stylistic shift towards high-density reasoning, directly outputting formulas and calculations with minimal logical connectors.
\textbf{Implication.} This compact generation style explains the ``corner case'' observed in our quantitative analysis where cognitive effort ($\Omega$) becomes higher after SFT. Unlike other models that acquire ``blind confidence'' on easy textual patterns, Qwen2.5-7B concentrates its generation on difficult logic and arithmetic tokens. By skipping the low-entropy ``filler'' words, the model sustains a high level of cognitive effort during the generation of the core answer.

\begin{figure}[H]
    \centering
    \begin{tcolorbox}[title=Qwen2.5-7B trained on GSM8K evaluated on Math Level 1-2]
    PROMPT:
    
    Please solve the following math problem, and provide your process as well as the final answer. You should make sure your final answer in the \boxed{} format.
    Problem: What is the value of $(26^2 - 24^2 - 10)^2 - 10^2$?

    CORRECT ANSWER: 8000

    Qwen2.5-7B GENERATION:

    To solve the given expression \((26^2 - 24^2 - 10)^2 - 10^2\), we can break it down into steps for clarity.

\#\#\# Step 1: Simplify inside the parentheses using difference of squares

The expression \(26^2 - 24^2\) can be simplified using the difference of squares formula:
\[a^2 - b^2 = (a - b)(a + b)\]
Here, \(a = 26\) and \(b = 24\), so:
\[26^2 - 24^2 = (26 - 24)(26 + 24) = 2 \cdot 50 = 100\]

\#\#\# Step 2: Substitute back into the original expression

Now substitute \(100\) back into the expression:
\[(100 - 10)^2 - 10^2\]

\#\#\# Step 3: Simplify inside the parentheses

Simplify \(100 - 10\):
\[100 - 10 = 90\]
So the expression becomes:
\[90^2 - 10^2\]

\#\#\# Step 4: Calculate each term

Calculate \(90^2\) and \(10^2\):
\[90^2 = 8100\]
\[10^2 = 100\]

\#\#\# Step 5: Subtract the results

Subtract \(100\) from \(8100\):
\[8100 - 100 = 8000\]

\#\#\# Final Answer:

Thus, the value of the expression \((26^2 - 24^2 - 10)^2 - 10^2\) is \(\boxed{8000}\).

    Qwen2.5-7B-SFT GENERATION:

    We can use the difference of squares to simplify $26^2 - 24^2$:
$26^2 - 24^2 = (26 + 24)(26 - 24) = 50 \cdot 2 = 100$.
Then, we have $(26^2 - 24^2 - 10)^2 - 10^2 = (100 - 10)^2 - 10^2 = 90^2 - 10^2 = 8100 - 100 = \boxed{8000}$.
\#\#\#\# 8000
    \end{tcolorbox}
    \caption{Example of the result of Qwen2.5-7B trained on GSM8K evaluated on dataset Math Level 1-2.}
\end{figure}

\begin{figure*}[t]
    \centering
    \begin{tcolorbox}[title=Qwen2.5-7B trained on GSM8K evaluated on AIME]
    PROMPT:
    Please solve the following math problem, and provide your process as well as the final answer. You should make sure your final answer in the \boxed{} format.
    Problem: Five men and nine women stand equally spaced around a circle in random order. The probability that every man stands diametrically opposite a woman is $\frac{m}{n},$ where $m$ and $n$ are relatively prime positive integers. Find $m+n.$

    CORRECT ANSWER: 191

    Qwen2.5-7B GENERATION:
    To solve this problem, we need to calculate the probability that each of the five men stands diametrically opposite a woman when ten people (five men and five women) are arranged around a circle.

\#\#\# Step 1: Total Arrangements
First, we calculate the total number of ways to arrange 10 people around a circle. Since rotations of the same arrangement are considered identical, we fix one person in place and arrange the remaining 9 people. This gives us:
\[
(10-1)! = 9!
\]

\#\#\# Step 2: Favorable Arrangements
Next, we count the number of favorable arrangements where each man stands diametrically opposite a woman. 

1. **Fixing Men and Women**: We need to pair each man with a woman such that they are diametrically opposite. There are 5 men and 5 women, so we need to find the number of ways to pair them up.
   
   The number of ways to pair 5 men with 5 women is given by the number of permutations of 5 items, which is:
   \[
   5! = 120
   \]

2. **Arranging Pairs Around the Circle**: After pairing, we need to arrange these 5 pairs around the circle. Fixing one pair in place (to account for rotational symmetry), we arrange the remaining 4 pairs. This can be done in:
   \[
   (5-1)! = 4!
   \]
   ways.

3. **Arranging Within Each Pair**: For each pair, there are 2 ways to arrange the man and the woman (man-woman or woman-man). Since there are 5 pairs, the total number of ways to arrange the individuals within the pairs is: $2^5 = 32$. Combining these, the total number of favorable arrangements is:
\[
5! \times 4! \times 2^5 = 120 \times 24 \times 32 = 92160
\]

\#\#\# Step 3: Calculating the Probability
The probability that each man stands diametrically opposite a woman is the ratio of the number of favorable arrangements to the total number of arrangements:
\[
\frac{92160}{9!} = \frac{92160}{362880} = \frac{1}{4}
\]

\#\#\# Step 4: Simplifying the Fraction
The fraction \(\frac{1}{4}\) is already in its simplest form, so \(m = 1\) and \(n = 4\).

\#\#\# Step 5: Finding \(m + n\)
Finally, we find \(m + n\):
\[
m + n = 1 + 4 = 5
\]

Thus, the final answer is:
\[
\boxed{5}
\]

    Qwen2.5-7B-SFT GENERATION:

    The total number of ways for the 14 people to stand around the circle is $13! / 2$ (since rotations and reflections are considered equivalent).
There are 5 men and 9 women, so there are $\binom{9}{5} = 126$ ways to choose which 5 of the 9 women will stand diametrically opposite a man.
For each choice of 5 women, there are $5!$ ways to assign the men to the women, and $5!$ ways to arrange the remaining 4 women around the circle.
Thus, the probability is $\frac{126 \cdot 5! \cdot 5!}{13!/2} = \frac{126 \cdot 2 \cdot 5! \cdot 5!}{13!} = \frac{126 \cdot 2 \cdot 120 \cdot 120}{6227020800} = \frac{302400}{6227020800} = \frac{7}{24624}$.
Since 7 and 24624 are relatively prime, the answer is $7 + 24624 = \boxed{24631}$.
\#\#\#\# 24631

    \end{tcolorbox}
    \caption{Example of the result of Qwen2.5-7B trained on GSM8K evaluated on dataset AIME.}
\end{figure*}

\begin{figure}[H]
    \centering
    \begin{tcolorbox}[title=Example of the result of Qwen2.5-7B trained on MBPP evaluated on dataset CoNaLa]
    PROMPT:
    
Write Python code to solve the following task:

split a string into a list of words and whitespace

    Qwen2.5-7B GENERATION:
    
    Certainly! To split a string into a list of words and whitespace, you can use Python's `re` (regular expression) module, which provides powerful pattern matching capabilities. Here's a Python function that does this:

```python

import re

def split\_string\_into\_words\_and\_whitespace(s):

    \# Use regular expression to find all sequences of alphanumeric characters (words)
    
    \# and all whitespace characters.
    return re.findall(r'\textbackslash w+|\textbackslash s+', s)

\# Example usage

input\_string = "Hello,   world!  This is an example."

result = split\_string\_into\_words\_and\_whitespace(input\_string)

print(result)

```

\#\#\# Explanation:

- `\textbackslash w+`: Matches one or more word characters (alphanumeric characters plus underscore).

- `\textbackslash s+`: Matches one or more whitespace characters (spaces, tabs, newlines, etc.).

- `|`: Acts as an OR operator in regular expressions.

- `re.findall()`: Returns all non-overlapping matches of the pattern in the string as a list of strings.

\#\#\# Output:

For the input string `"Hello,   world!  This is an example."`, the output will be:

```python

['Hello', ' ', 'world', '!', ' ', 'This', ' ', 'is', ' ', 'an', ' ', 'example', '.']

```

This output includes both the words and the whitespace as separate elements in the list.

    Qwen2.5-7B-SFT GENERATION:

    import re

def split\_list(text):

  return (re.findall(r'\textbackslash S+|\textbackslash s+', text))
    
    \end{tcolorbox}
    \caption{Example of the result of Qwen2.5-7B evaluated on dataset CoNaLa.}
\end{figure}

\begin{figure}[H]
    \centering
    \begin{tcolorbox}[title=Prompt on APPS]
    PROMPT: 

    Please solve the following programming problem. Write Python code to solve it.

Joisino is about to compete in the final round of a certain programming competition.
In this contest, there are N problems, numbered 1 through N.
Joisino knows that it takes her $T_i$ seconds to solve problem i($1\leq i\leq $).
Also, there are M kinds of drinks offered to the contestants, numbered 1 through M.
If Joisino takes drink i($1\leq i\leq M$), her brain will be stimulated and the time it takes for her to solve problem $P_i$ will become $X_i$ seconds.
It does not affect the time to solve the other problems.
A contestant is allowed to take exactly one of the drinks before the start of the contest.
For each drink, Joisino wants to know how many seconds it takes her to solve all the problems if she takes that drink.
Here, assume that the time it takes her to solve all the problems is equal to the sum of the time it takes for her to solve individual problems.
Your task is to write a program to calculate it instead of her.

-----Constraints-----

 - All input values are integers.
 
 - $1\leq N\leq 100$
 
 - $1\leq T_i\leq 10^5$
 
 - $1\leq M\leq 100$
 
 - $1\leq P_i\leq N$
 
 - $1\leq X_i\leq 10^5$

-----Input-----

The input is given from Standard Input in the following format:

N

$T_1~ T_2 ~... ~T_N$

M

$P_1 ~X_1$

$P_2 ~X_2$

:

$P_M ~X_M$

-----Output-----

For each drink, calculate how many seconds it takes Joisino to solve all the problems if she takes that drink, and print the results, one per line.

-----Sample Input-----

3

2 1 4

2

1 1

2 3

-----Sample Output-----

6

9

If Joisino takes drink 1, the time it takes her to solve each problem will be 1, 1 and 4 seconds, respectively, totaling 6 seconds.

If Joisino takes drink 2, the time it takes her to solve each problem will be 2, 3 and 4 seconds, respectively, totaling 9 seconds.

Code:
    
    \end{tcolorbox}
    \caption{Example prompt of dataset APPS.}
\end{figure}

\begin{figure}[H]
    \centering
    \begin{tcolorbox}[title=Qwen2.5-7B trained on MBPP evaluated on APPS]
    Qwen2.5-7B GENRERATION: Sure, let's break down the problem and then write the Python code to solve it.

\#\#\# Problem Breakdown

1. **Input Parsing**:

   - First, we need to read the number of problems `N` and the time taken to solve each problem.
   
   - Then, we need to read the number of drinks `M` and the information about each drink (the problem it affects and the new time it takes to solve that problem).

2. **Processing**:
   
   - For each drink, we need to calculate the total time it would take Joisino to solve all the problems if she were to use that drink.
   
   - We need to update the time for the specific problem affected by the drink and keep the original times for the rest of the problems.

3. **Output**:

   - Print the total time for each drink.

\#\#\# Python Code Implementation

Let's implement the solution step-by-step:

```python

\# Read the number of problems

N = int(input())

\# Read the time taken to solve each problem

T = list(map(int, input().split()))

\# Read the number of drinks

M = int(input())

\# Process each drink

for \_ in range(M):

    \# Read the drink details: the problem it affects and the new time
    P, X = map(int, input().split())
    
    \# Calculate the new total time
    
    total\_time = sum(T)
    
    total\_time -= T[P-1]  \# Subtract the original time for the affected problem
    
    total\_time += X       \# Add the new time for the affected problem
    
    \# Print the result
    
    print(total\_time)
```

\#\#\# Explanation

1. **Reading Inputs**:

   - `N` is the number of problems.
   
   - `T` is a list where `T[i]` represents the time to solve problem `i+1`.
   
   - `M` is the number of drinks.
   
   - Each drink is represented by two numbers: `P` (the problem it affects) and `X` (the new time to solve that problem).

2. **Calculating Total Time**:
   
   - For each drink, we first calculate the total time as the sum of all problem-solving times.
   
   - We then subtract the original time for the affected problem and add the new time for that problem.

3. **Output**:
   
   - We print the new total time for each drink.
This code should correctly handle the constraints and produce the expected output.

    Qwen2.5-7B-SFT GENERATION:

    def time(N, T, M, P, X):
    
    res = []
    
    for m in range(M):
    
        tmp = [t for i, t in enumerate(T) if i + 1 != P[m]]
        
        tmp[P[m] - 1] = X[m]
        
        res.append(sum(tmp))
    
    return res
    
    \end{tcolorbox}
    \caption{Example of the result of Qwen2.5-7B trained on MBPP evaluated on dataset APPS.}
\end{figure}

\clearpage

\section{Full Results of Blind Confidence of SFT}
\label{appen:SFT}

In this section, we present the comprehensive confidence trajectory visualizations supporting the analysis of SFT discussed in Section~\ref{sec:SFT}. 
As defined in the main text, \methodname reveals a phenomenon of ``Blind Confidence.'' Visually, this is characterized by the SFT curves (blue) rising significantly earlier and staying higher in the intermediate layers compared to the base model curves (red). This indicates that the SFT process often bypasses the deep cognitive processing (the ``struggle'' phase) present in the base models, leading the model to reach high confidence saturated states prematurely. 

While this pattern is consistent across most Llama3-8B tasks, specific behaviors in Qwen2.5-7B (particularly in Math and Code) involve sharper confidence spikes. As noted in Section~\ref{sec:SFT}, this corresponds to the model generating concise, memorized patterns (short core tokens) rather than reasoning through the problem, resulting in a sharp decrease in the overall specific effort $\hat{\Omega}$.
\vspace{-.1in}
\begin{figure*}[!htbp]
    \centering
    \includegraphics[width=0.4\linewidth]{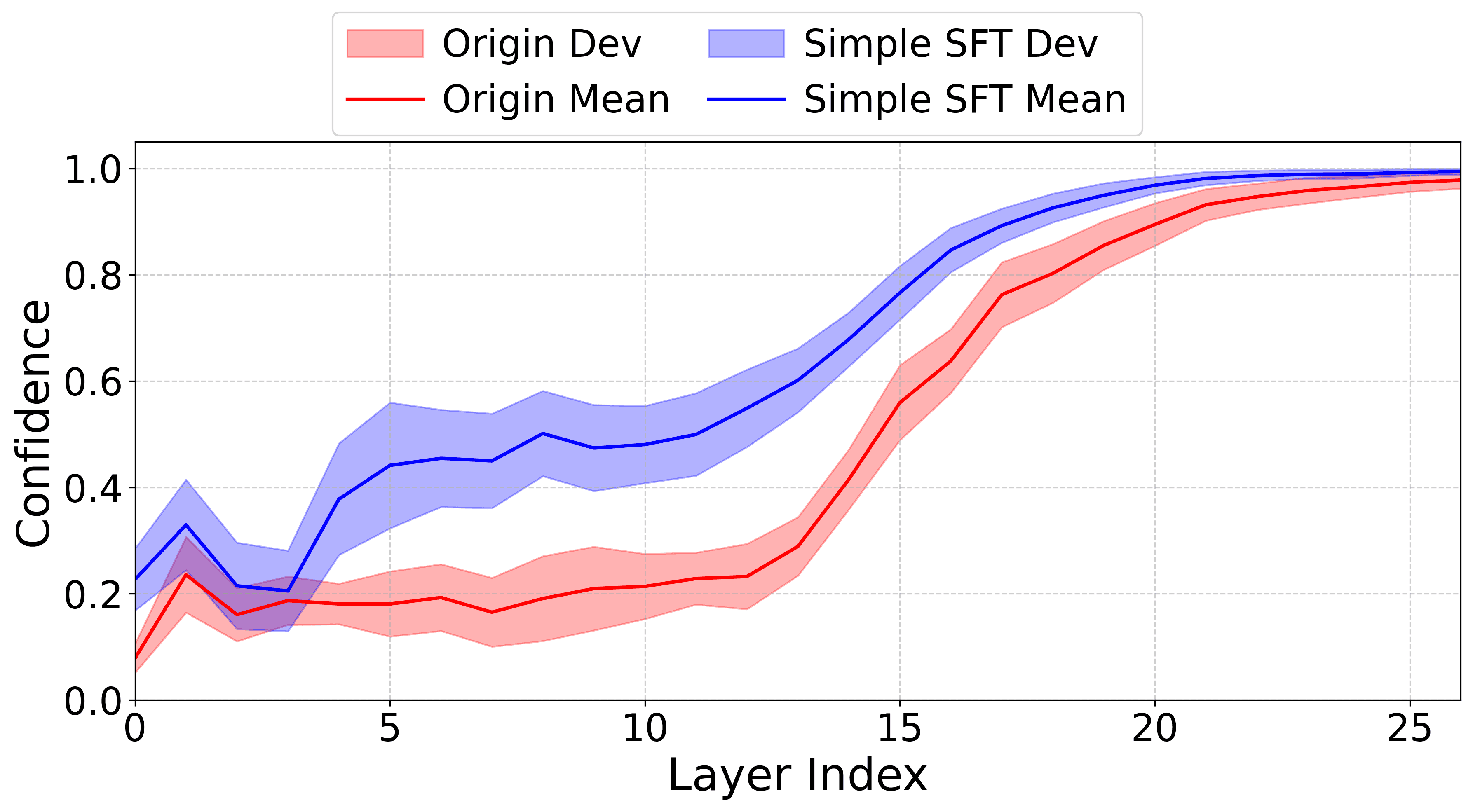}
    \includegraphics[width=0.4\linewidth]{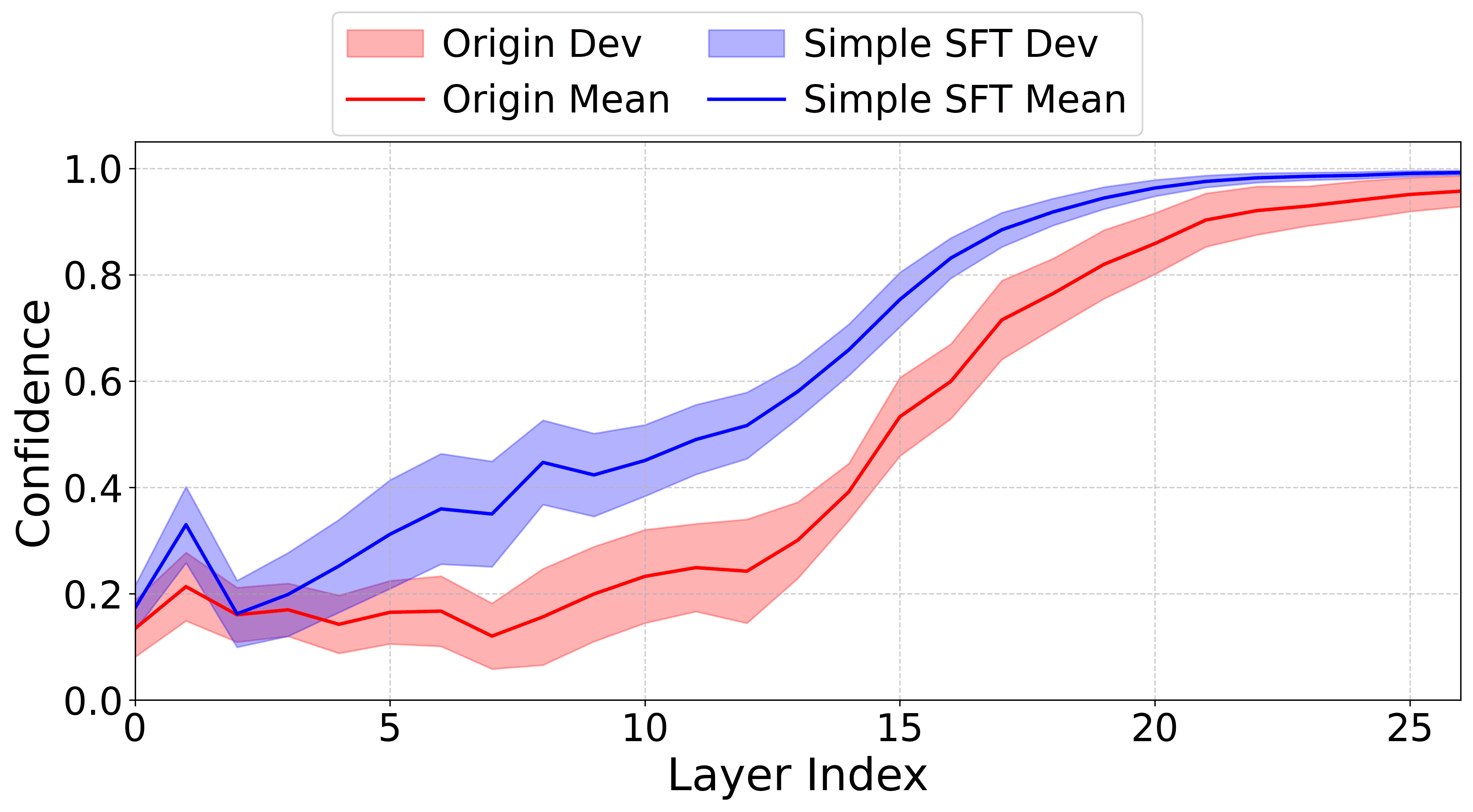}
    \caption{The confidence trajectory of Llama3-8B on the dataset MATH, (left) AIME (right) for the experiment Blind Confidence of SFT.
    }
    \label{fig:SFT_llama3-8B_math}
    \vspace{-.2in}
\end{figure*}

\begin{figure*}[!htbp]
    \centering
    \includegraphics[width=0.4\linewidth]{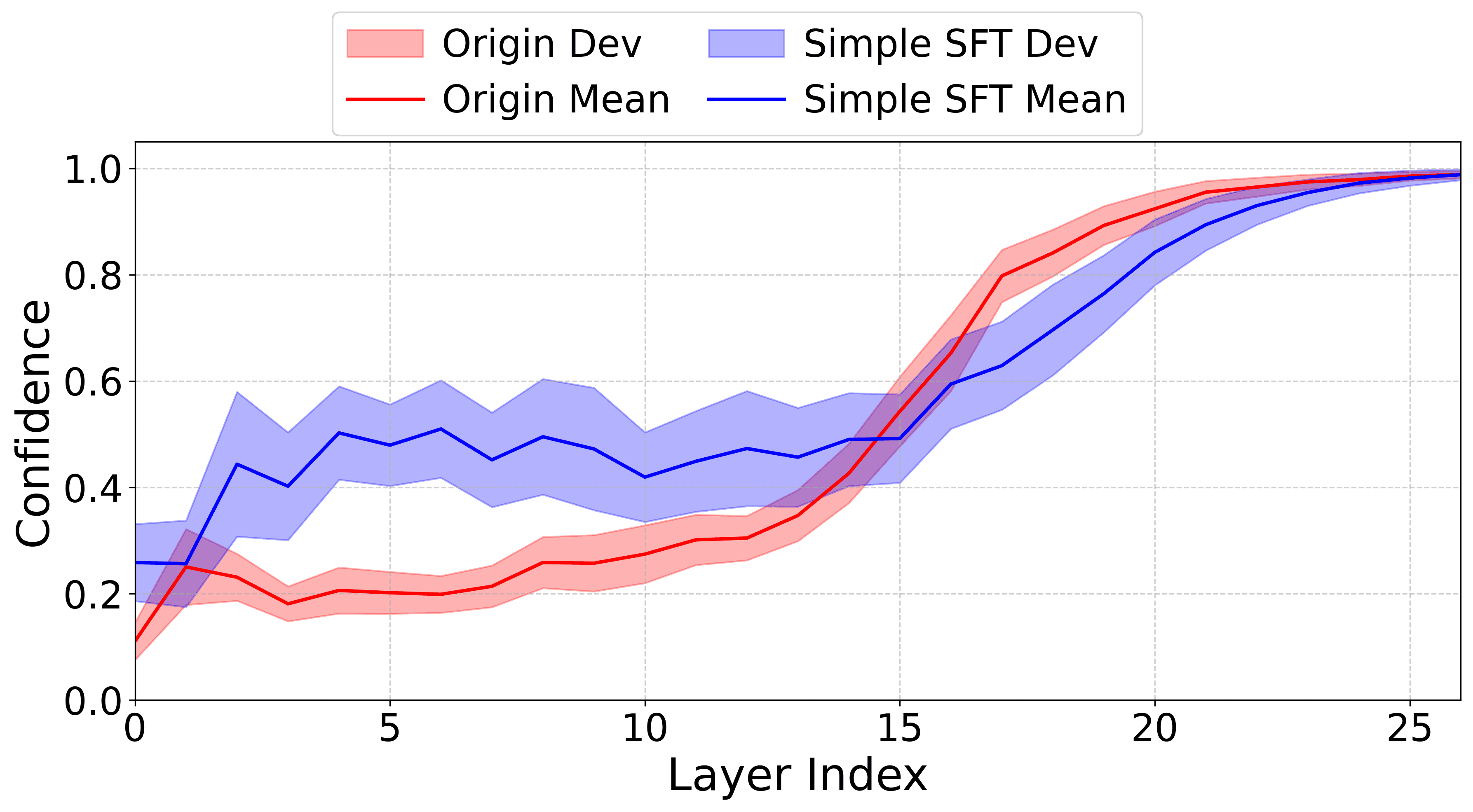}
    \includegraphics[width=0.4\linewidth]{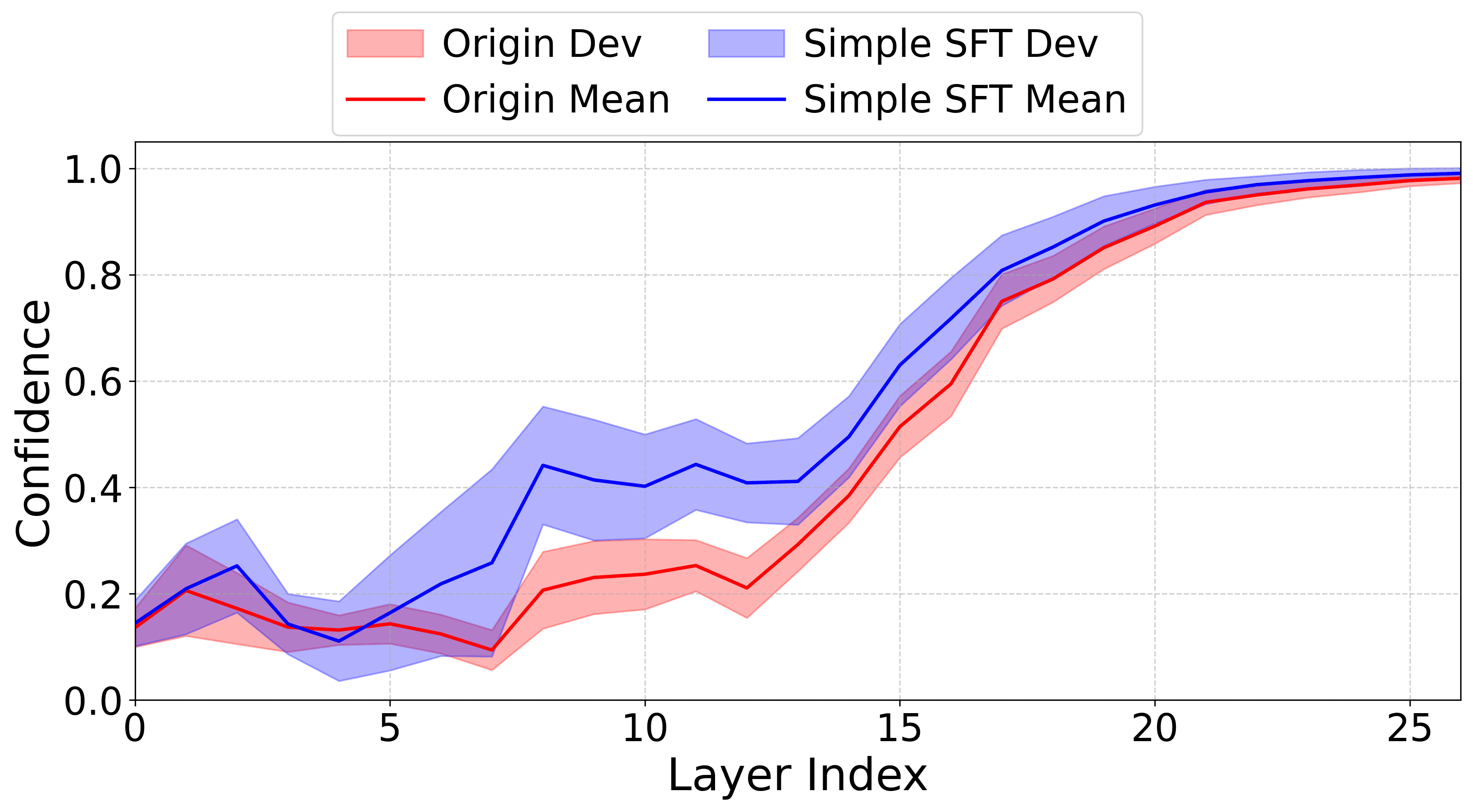}
    \caption{The confidence trajectory of Llama3-8B on the dataset CoNaLa, (left) APPS (right) for the experiment Blind Confidence of SFT.
    }
    \label{fig:SFT_llama3-8B_code}
    \vspace{-.35in}
\end{figure*}

\begin{figure*}[!htbp]
    \centering
    \includegraphics[width=0.4\linewidth]{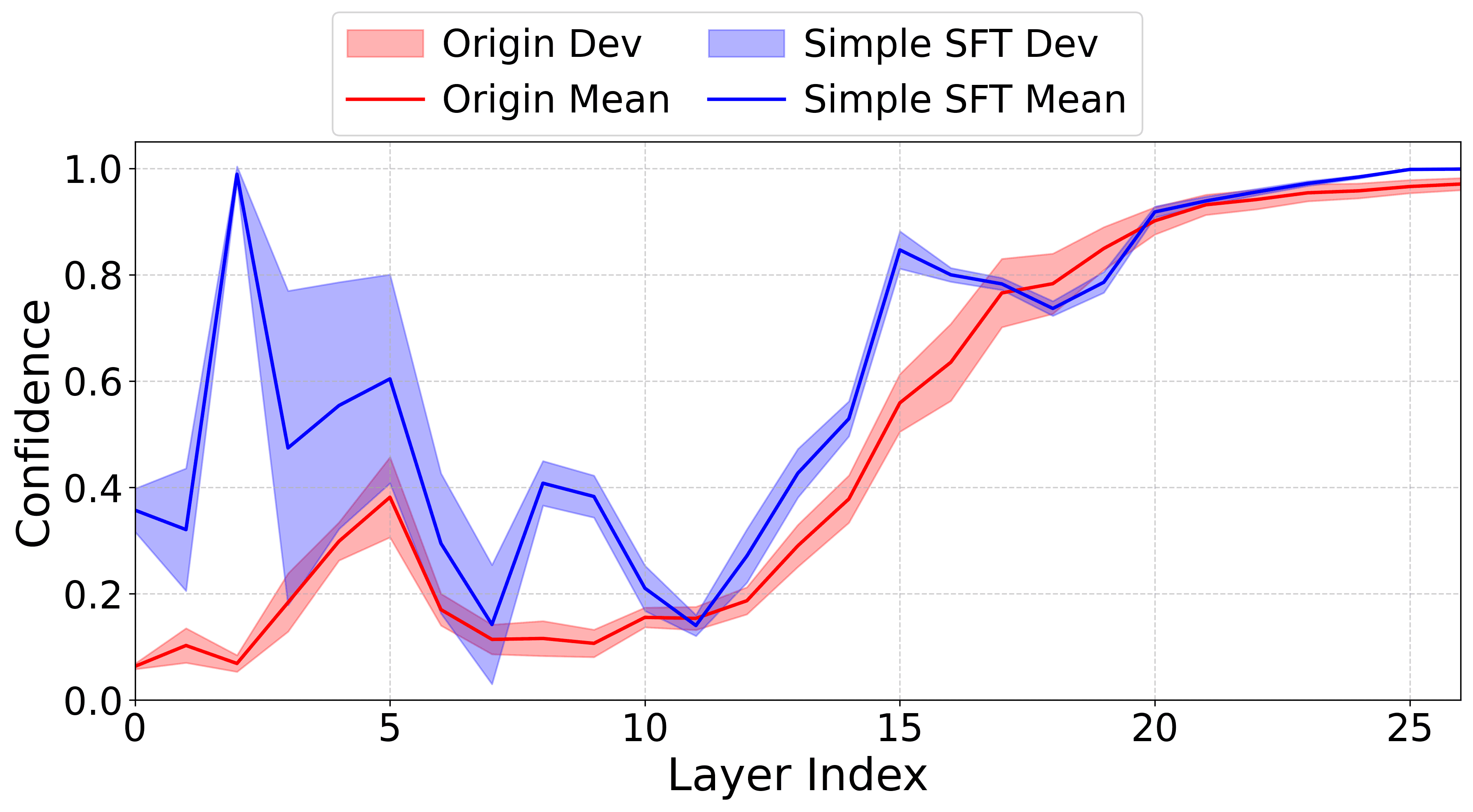}
    \includegraphics[width=0.4\linewidth]{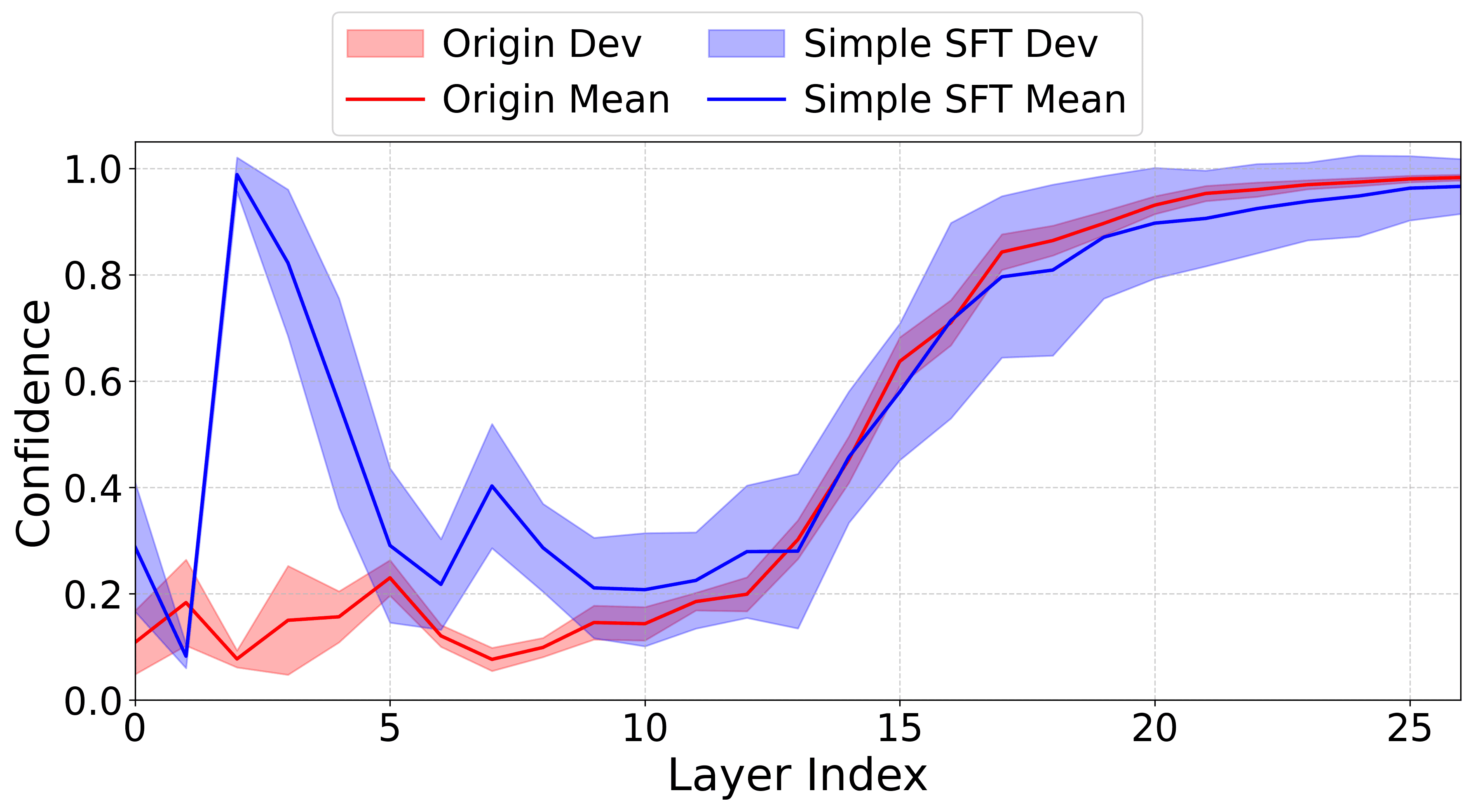}
    \caption{The confidence trajectory of Llama3-8B on the dataset Ruletaker, (left) ZebraLogic (right) for the experiment Blind Confidence of SFT.
    }
    \label{fig:SFT_llama3-8B_logic}
    \vspace{-.35in}
\end{figure*}

\begin{figure*}[!htbp]
    \centering
    \includegraphics[width=0.4\linewidth]{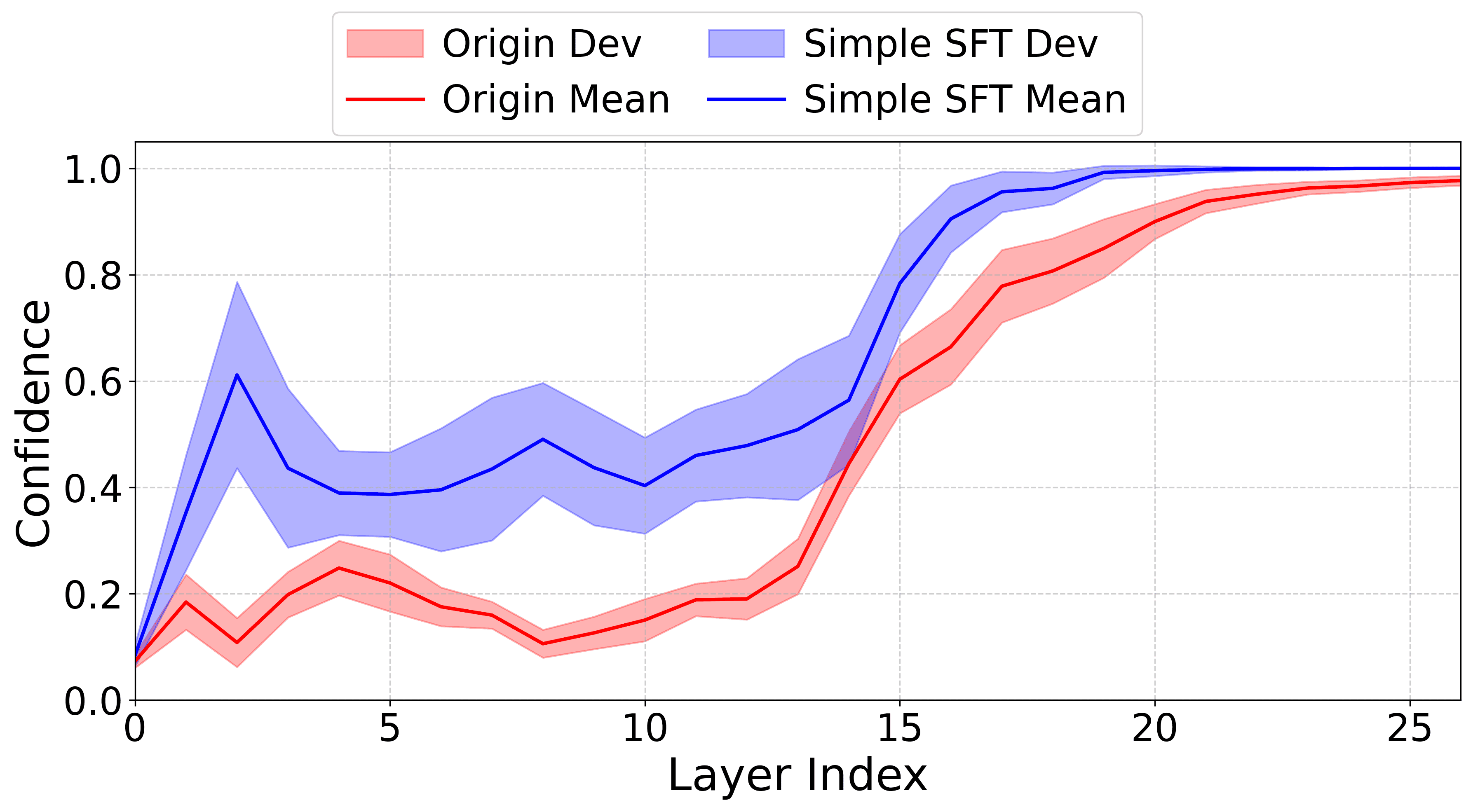}
    \includegraphics[width=0.4\linewidth]{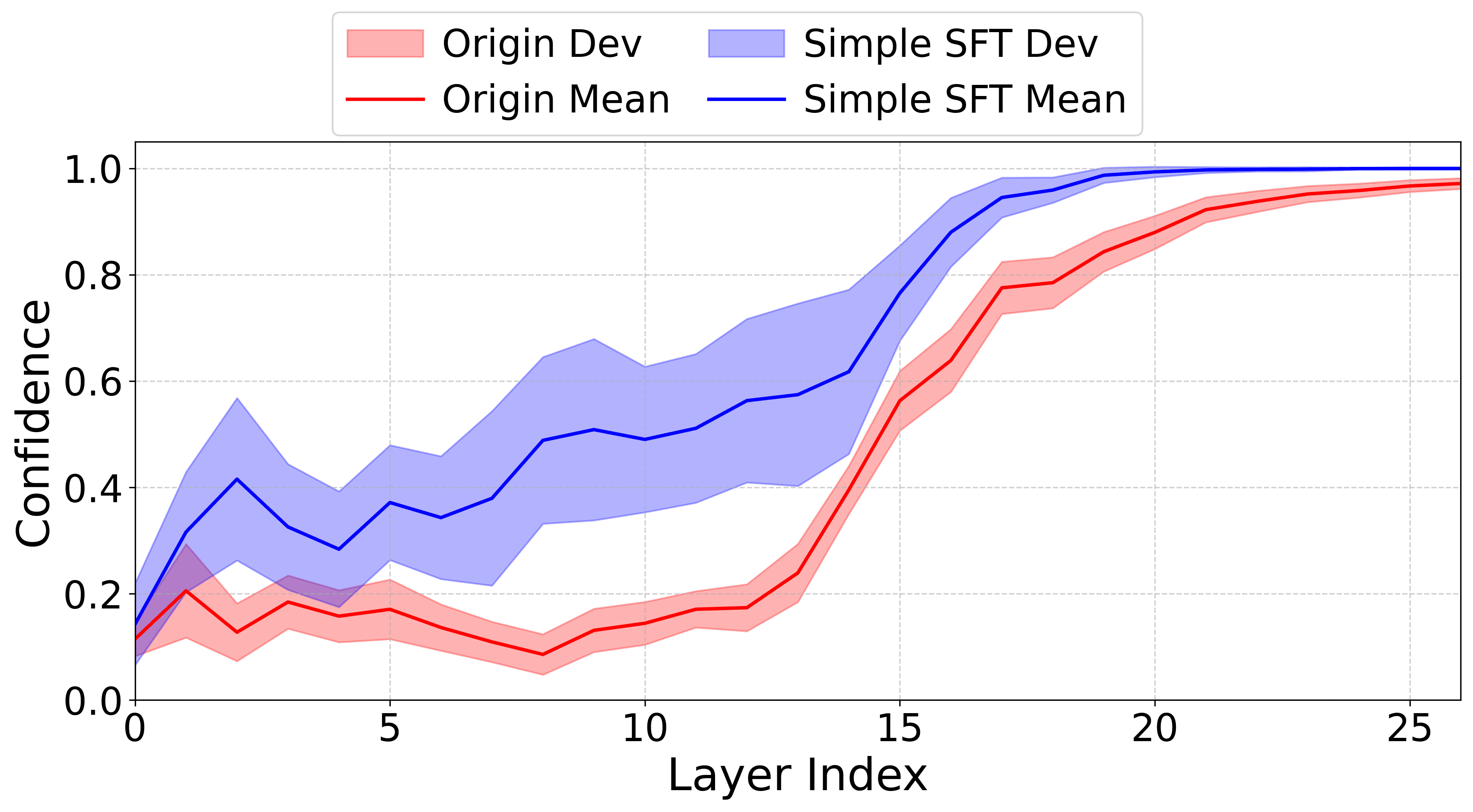}
    \caption{The confidence trajectory of Llama3-8B on the dataset ARC Easy, (left) GPQA (right) for the experiment Blind Confidence of SFT.
    }
    \label{fig:SFT_llama3-8B_general}
    \vspace{-.35in}
\end{figure*}

\begin{figure*}[!htbp]
    \centering
    \includegraphics[width=0.4\linewidth]{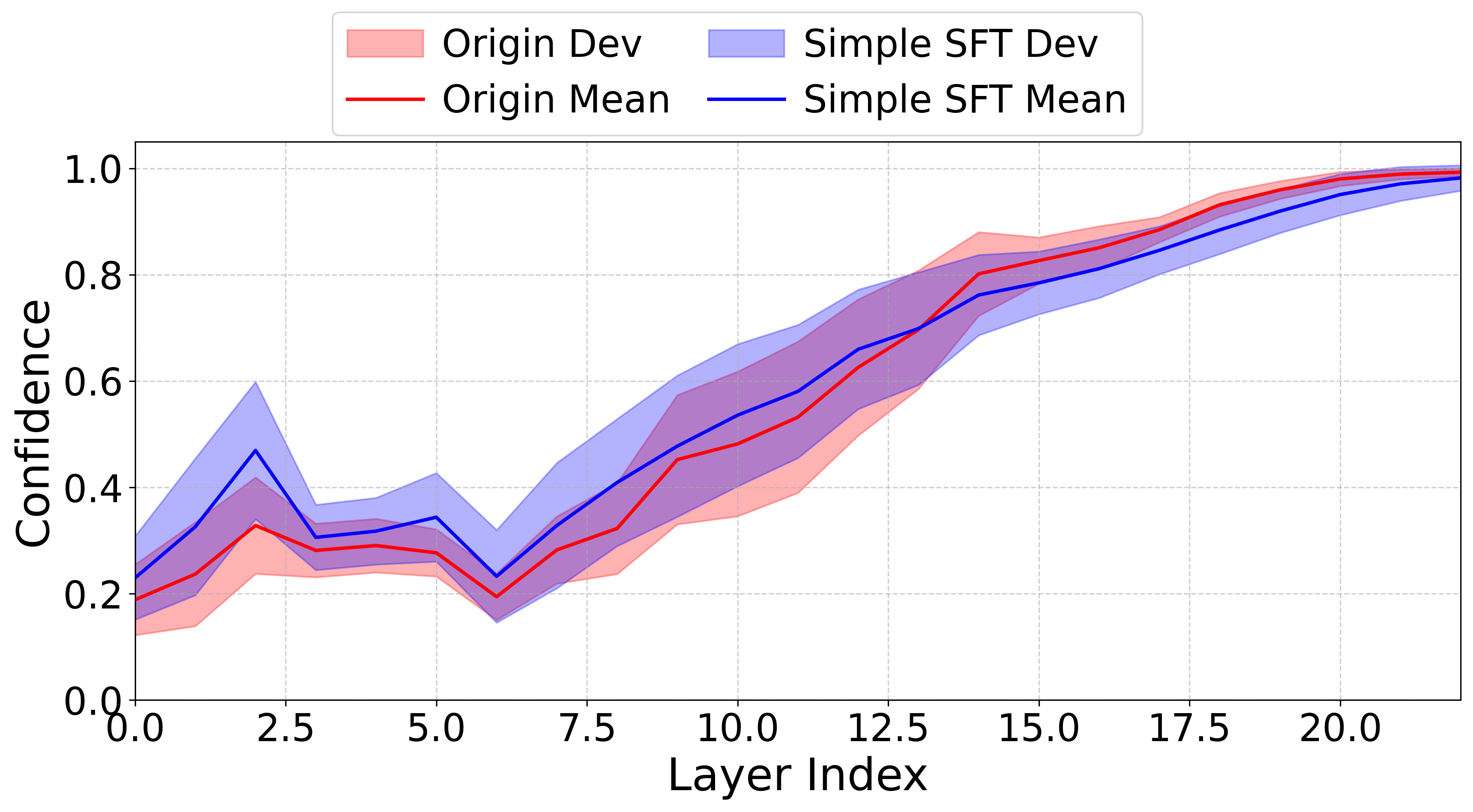}
    \includegraphics[width=0.4\linewidth]{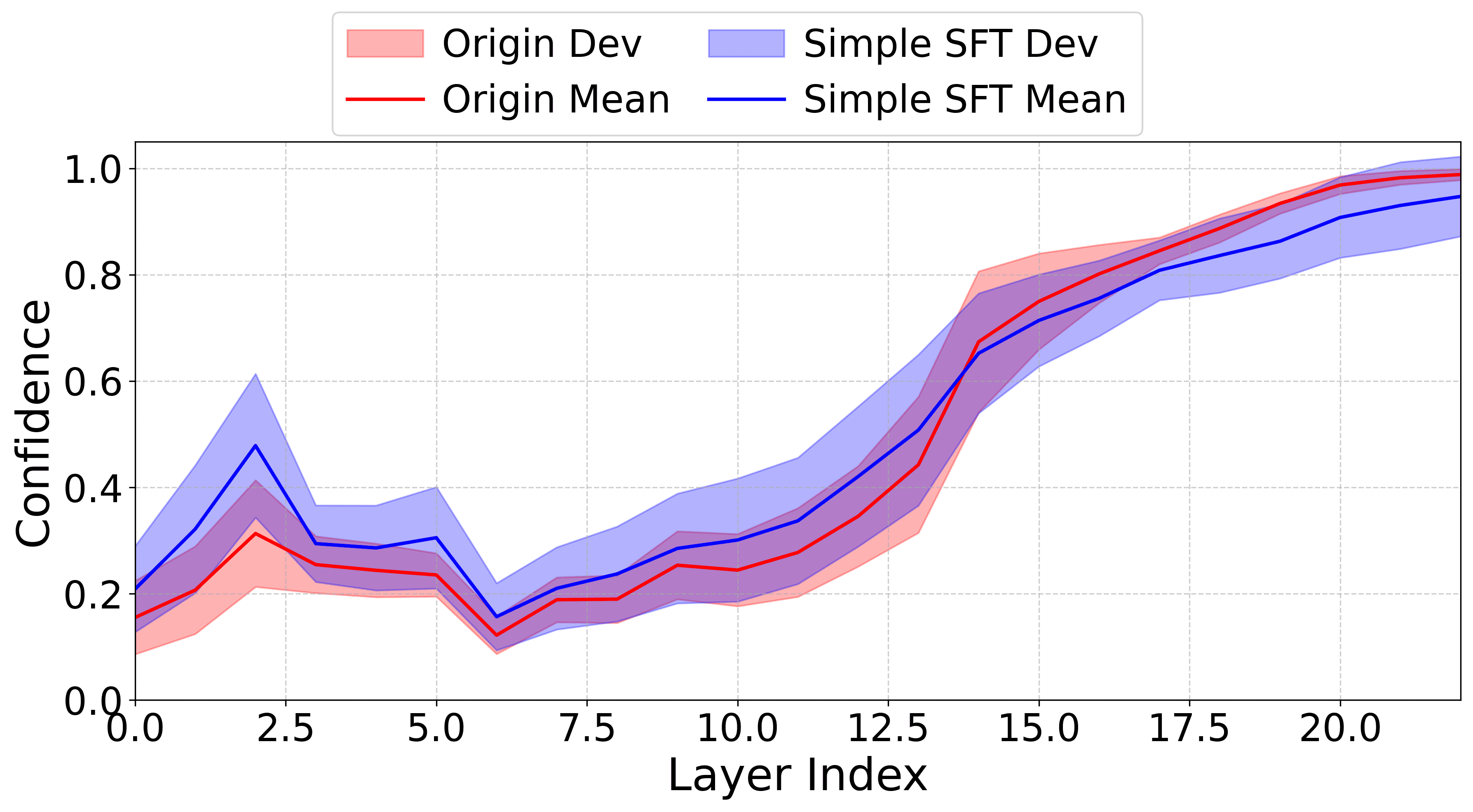}
    \caption{The confidence trajectory of Qwen2.5-7B on the dataset MATH, (left) AIME (right) for the experiment Blind Confidence of SFT.
    }
    \label{fig:SFT_qwen2.5-7B_math}
\end{figure*}

\begin{figure*}[!htbp]
    \centering
    \includegraphics[width=0.4\linewidth]{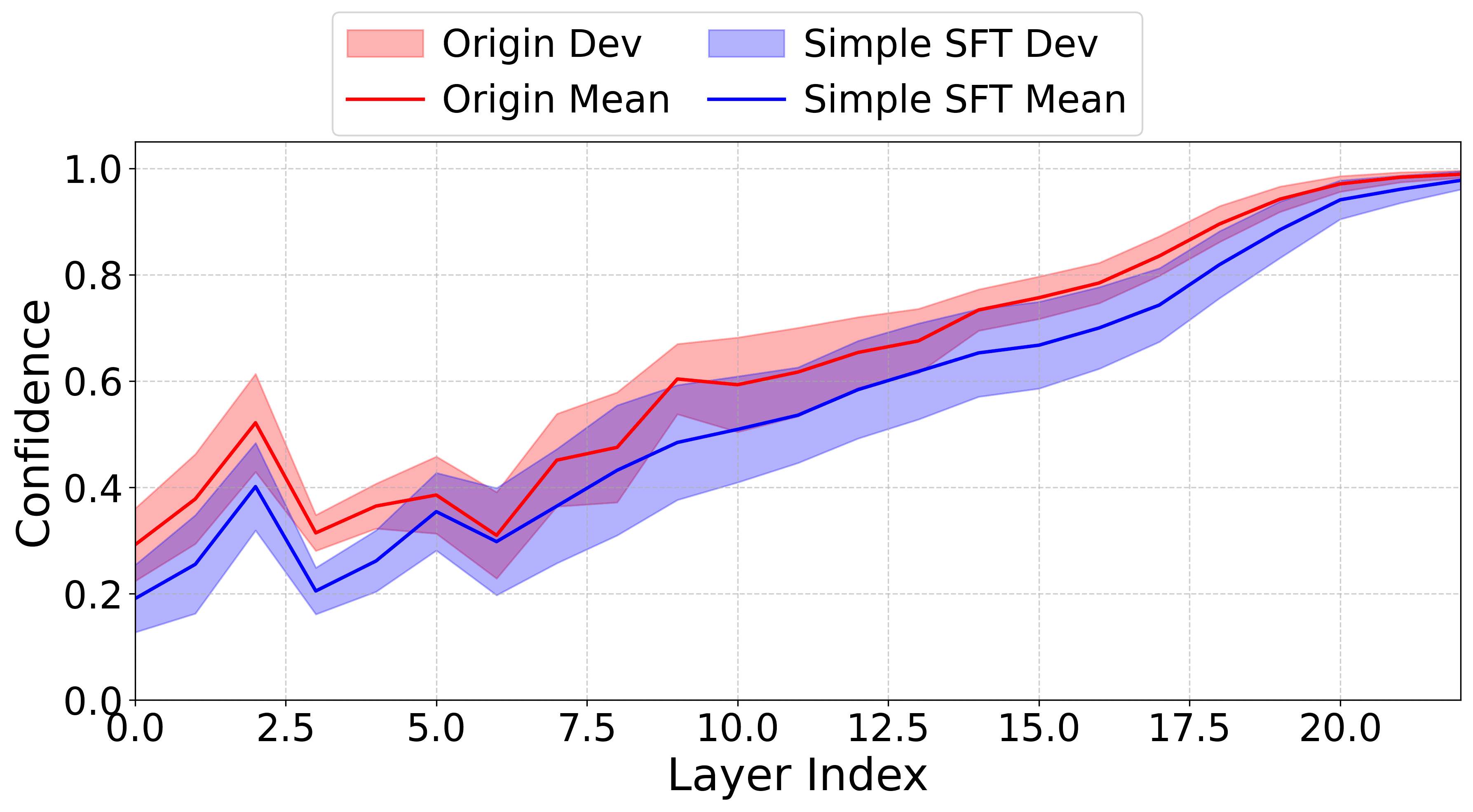}
    \includegraphics[width=0.4\linewidth]{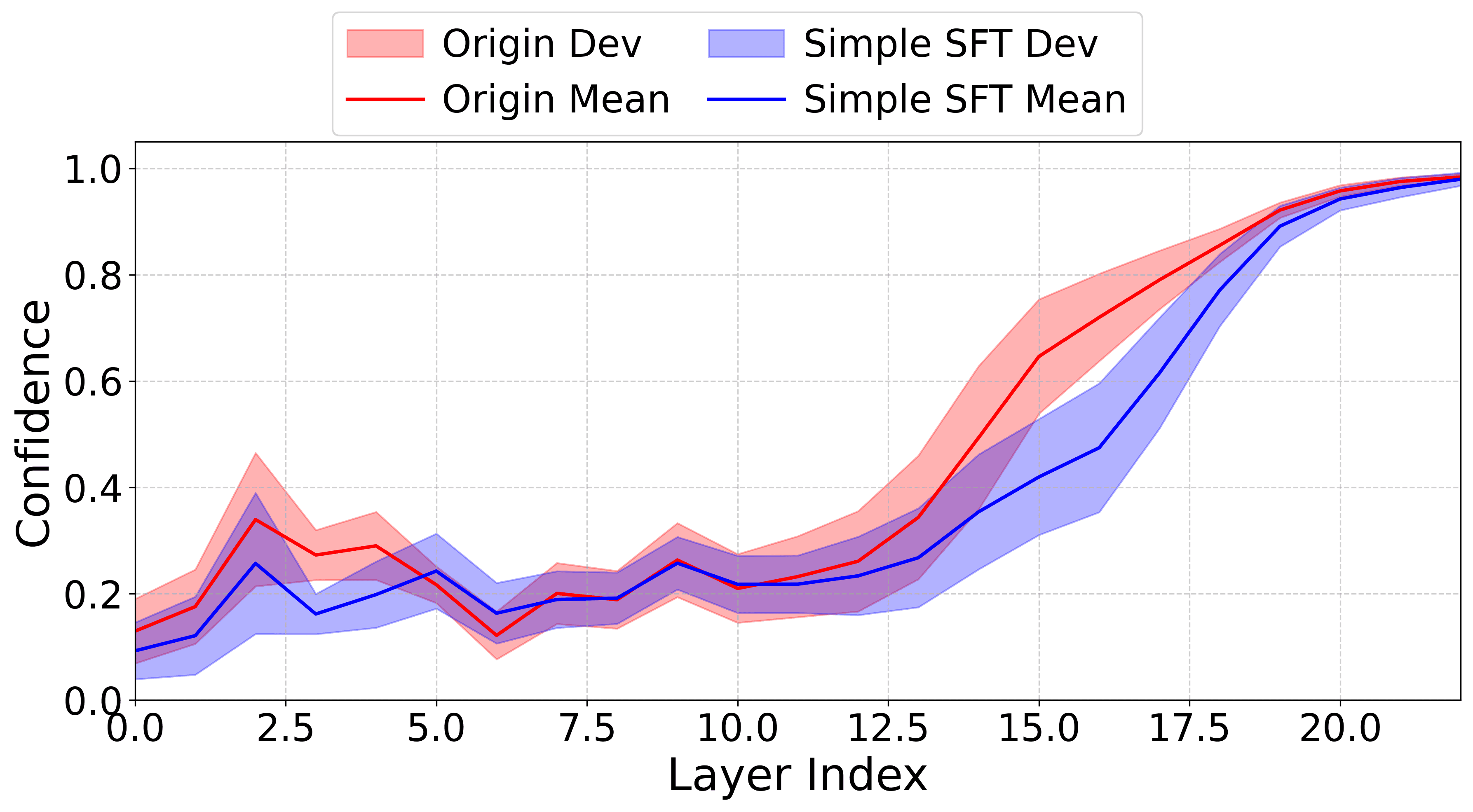}
    \caption{The confidence trajectory of Qwen2.5-7B on the dataset CoNaLa, (left) APPS (right) for the experiment Blind Confidence of SFT.
    }
    \label{fig:SFT_qwen2.5-7B_code}
\end{figure*}

\begin{figure*}[!htbp]
    \centering
    \includegraphics[width=0.4\linewidth]{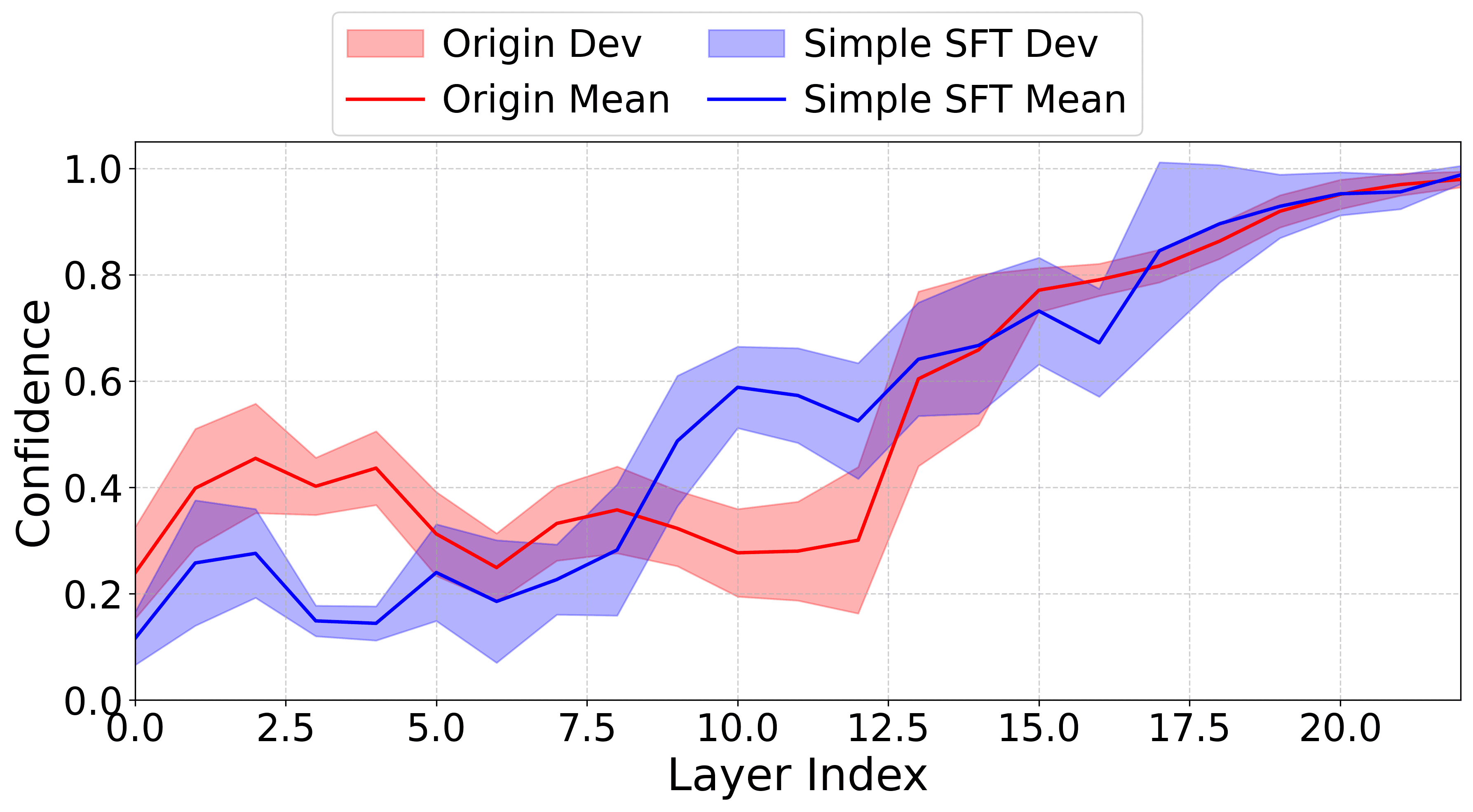}
    \includegraphics[width=0.4\linewidth]{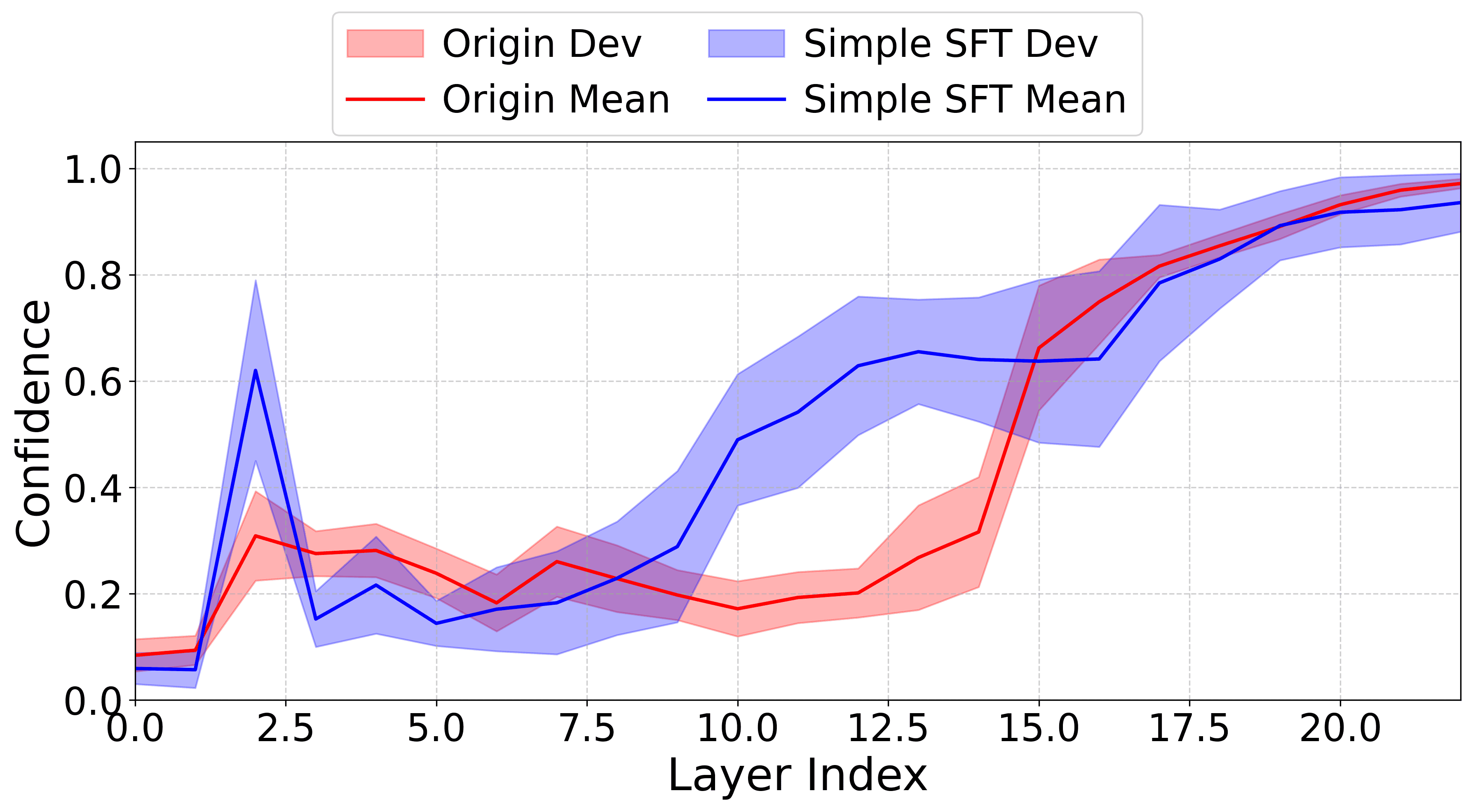}
    \caption{The confidence trajectory of Qwen2.5-7B on the dataset Ruletaker, (left) ZebraLogic (right) for the experiment Blind Confidence of SFT.
    }
    \label{fig:SFT_qwen2.5-7B_logic}
\end{figure*}

\begin{figure*}[!htbp]
    \centering
    \includegraphics[width=0.4\linewidth]{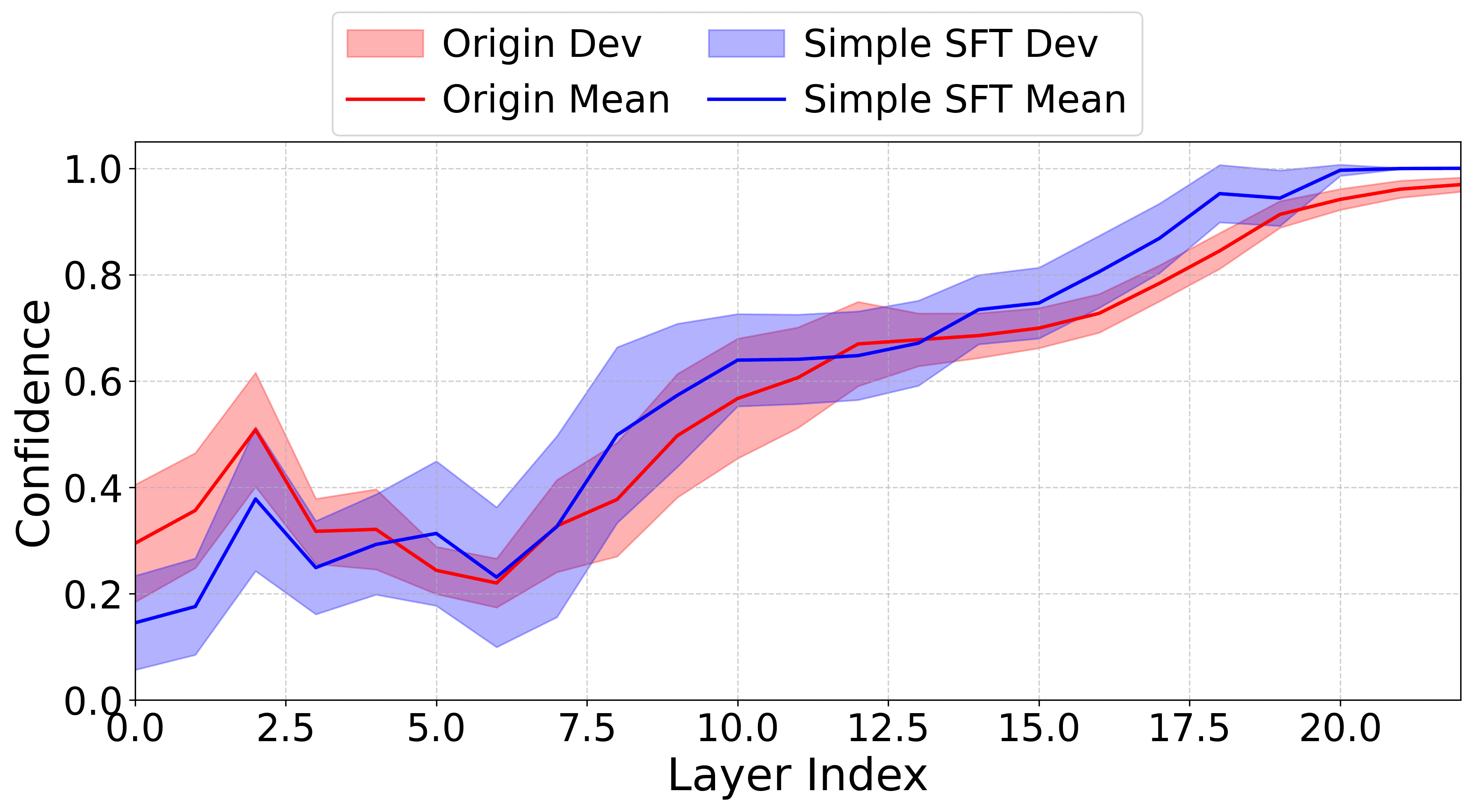}
    \includegraphics[width=0.4\linewidth]{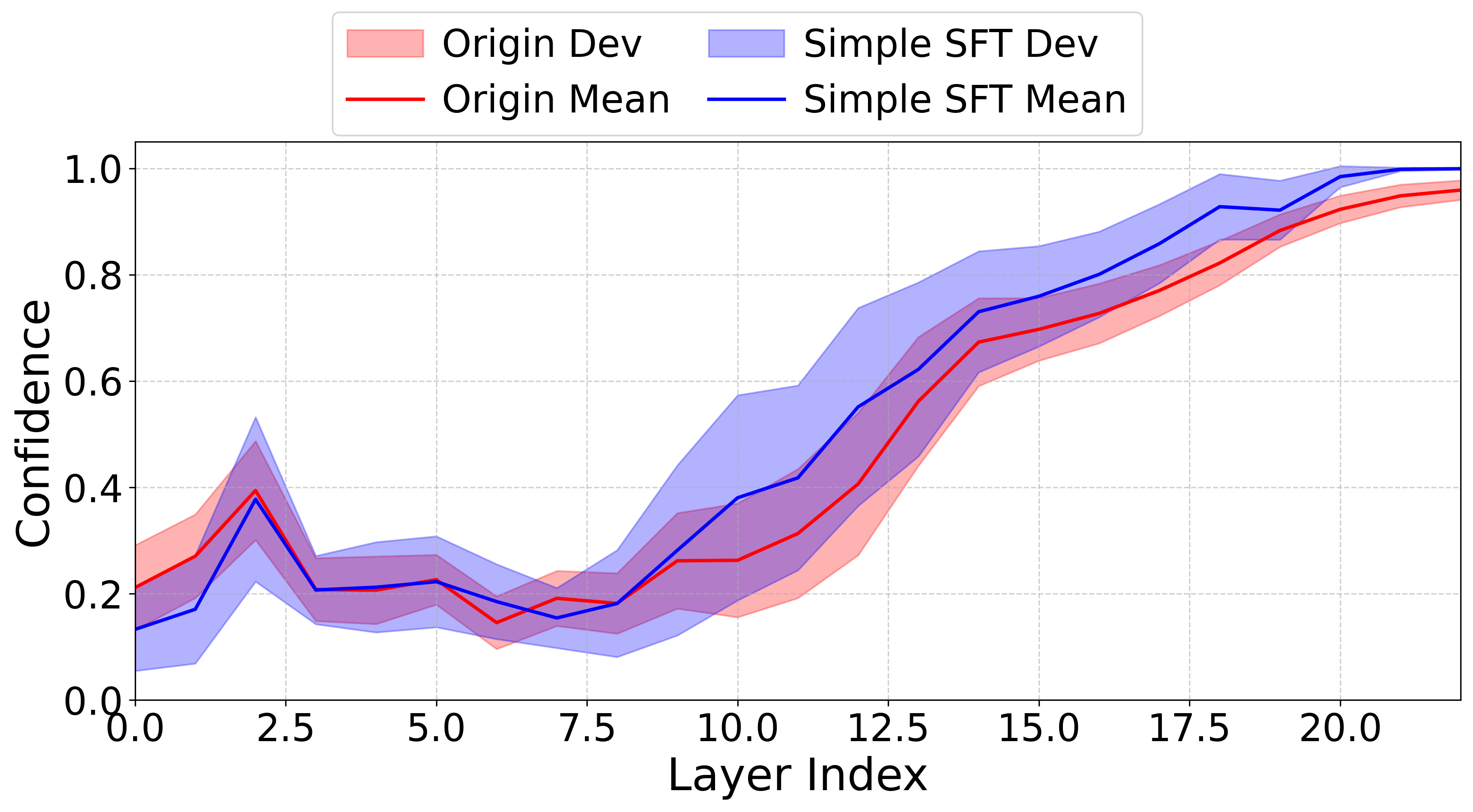}
    \caption{The confidence trajectory of Qwen2.5-7B on the dataset ARC Easy, (left) GPQA (right) for the experiment Blind Confidence of SFT.
    }
    \label{fig:SFT_qwen2.5-7B_general}
\end{figure*}